%% file: DC_RF.tex
\documentclass[10pt]{article}

\usepackage{amsmath}
\usepackage{amsthm}
\usepackage{amssymb}
\usepackage{makecell}
\usepackage{algorithm}
\usepackage{algorithmic}
\usepackage{graphicx} 
\usepackage{grffile}
\usepackage{subfigure}
\usepackage{xcolor}
\usepackage{url}
\usepackage{enumitem}
\usepackage{booktabs}
\input{0-newcommands.tex}

\title{Towards Sharp Analysis for Distributed Learning with Random Features}

\author{
Jian~Li \\ Institute of Information Engineering, Chinese Academy of Sciences \\
\and
Yong~Liu$^{\dagger}$ \thanks{Yong Liu is also the corresponding author.} \\
Gaoling School of Artificial Intelligence, Renmin University of China \\
\and
Weiping~Wang \\
Institute of Information Engineering, Chinese Academy of Sciences
}

\date{\today}
\begin{document}
\maketitle

\begin{abstract}
	In recent studies, the generalization properties for distributed learning and random features assumed the existence of the target concept over the hypothesis space. However, this strict condition is not applicable to the more common non-attainable case. In this paper, using refined proof techniques, we first extend the optimal rates for distributed learning with random features to the non-attainable case. Then, we reduce the number of required random features via data-dependent generating strategy, and improve the allowed number of partitions with additional unlabeled data. Theoretical analysis shows these techniques remarkably reduce computational cost while preserving the optimal generalization accuracy under standard assumptions. Finally, we conduct several experiments on both simulated and real-world datasets, and the empirical results validate our theoretical findings.
\end{abstract}


\input{1-introduction.tex}
\input{2-method.tex}
\input{3-theoretical-assessment.tex}
\input{3-experiment.tex}
\input{4-conclusion.tex}
\input{5-acknowledge.tex}

\bibliographystyle{unsrt}
\bibliography{all}
\input{3-proofs.tex}

\end{document}

%% file: 0-newcommands.tex
\newtheorem{assumption}{Assumption}

\newcommand{\argmin}{\operatornamewithlimits{arg\,min}}
\newcommand{\xx}{\boldsymbol x}
\newcommand{\yy}{\boldsymbol y}
\newcommand{\zz}{\boldsymbol z}

\newcommand{\rhox}{{\rho_{X}}}
\newcommand{\Ltwo}{{L^2_\rhox}}

\newcommand{\Nystrom}[1]{{Nystr\"om}}

\newcommand{\R}{\mathbb R}
\newcommand{\CC}{\R}

\newcommand{\HH}{\mathcal H}

\newcommand{\la}{\lambda}

\newcommand{\X}{{X}}

\renewcommand{\L}{{L}}

\newcommand{\C}{\widehat{\C}}

\newcommand{\yn}{\widehat{y}_n}
\newcommand{\Ll}{L_\la}

\newcommand{\bpr}{\begin{proof}}
		\newcommand{\epr}{\end{proof}}
\newcommand{\be}{\begin{equation}}
		\newcommand{\ee}{\end{equation}}

\providecommand{\nor}[1]{\lVert{#1}\rVert}

\providecommand{\tr}{\operatorname{Tr}}

\renewcommand{\L}{{L}}

\newcommand{\tL}{L_M}
\newcommand{\tC}{C_M}
\newcommand{\tCn}{\widehat{C}_M}

\newcommand{\tCnl}{\widehat{C}_{M,\la}}
\newcommand{\tS}{S_M}
\newcommand{\SH}{S_\mathcal{H}}
\newcommand{\HSH}{\widehat{S}_\mathcal{H}}
\newcommand{\SM}{S_M}
\newcommand{\tSn}{\widehat{S}_M}

\renewcommand{\L}{{L}}

\newcommand{\frho}{f_\rho}

\newtheorem{theorem}{Theorem}
\newtheorem{lemma}{Lemma}
\newtheorem{proposition}{Proposition} 
\newtheorem{remark}{Remark}
\newtheorem{corollary}{Corollary}
\newtheorem{definition}{Definition}

\def\ww{\boldsymbol w}
\def\xx{\boldsymbol x}
\def\yy{\boldsymbol y}
\def\zz{\boldsymbol z}

\providecommand{\nor}[1]{\lVert{#1}\rVert}

\providecommand{\tr}{\operatorname{Tr}}

%% file: 1-introduction.tex
\section{Introduction}
A fundamental problem in machine learning is to achieve tradeoffs between statistical properties and computational costs \cite{bottou2008tradeoffs,li2018multi}, while this challenge is more severe in kernel methods.
Despite the excellent theoretical guarantees, kernel methods do not scale well in large-scale settings because of high time and memory complexities,
typically at least quadratic in the number of examples.
To break the scalability bottlenecks, researchers developed a wide range of practical algorithms, including distributed learning, which produces a global model after training disjoint subset on individual machines with necessary communications \cite{zhang2015divide,lin2017distributed}, Nystr\"om approximation \cite{williams2001using,rudi2015less,li2019approximate} and random Fourier features \cite{rahimi2007random,rudi2017generalization} to alleviate memory bottleneck, as well as stochastic methods \cite{lin2020optimal} to improve the training efficiency.

From the theoretical perspective, many researchers have studied the statistical properties of those large-scale approaches together with kernel ridge regression (KRR) \cite{rudi2015less,lin2016optimal,lin2017distributed}.
Using integral operator techniques \cite{smale2007learning} and the effective dimension to control the capability of RKHS \cite{caponnetto2007optimal}, the generalization bounds have achieved the optimal learning rates.
Recent statistical learning studies on KRR together with large-scale approaches demonstrate that these approaches can not only obtain great computational gains but still remain the optimal theoretical properties, such as KRR together with divide-and-conquer \cite{guo2017learning,mucke2018parallelizing}, with random projections including Nystr\"om approximation \cite{rudi2015less} and random features \cite{rudi2017generalization,carratino2018learning,li2020automated,li2021sharp}.
Since the communication cost is high to combine local kernel estimators in RKHS, it's more practical to combine the linear estimator in the feature space, e.g. federated learning \cite{mcmahan2017communication}.
Therefore, the generalization analysis for the combination of distributed learning and random features is rather important in distributed learning.

The existing works on DKRR \cite{guo2017learning,lin2017distributed,mucke2018parallelizing} and random features \cite{rudi2017generalization,li2019towards,li2021sharp} mainly focus on the attainable case that the true regression belongs to the hypothesis space, ignoring the non-attainable case where the true regression is out of the hypothesis space. 
Since it's hard to select the suitable kernel via kernel selection to guarantee that the target function belongs to the kernel space, the non-attainable case is more common in practice. Therefore, the statistical guarantees for the non-attainable are of practical and theoretical interest in the context of the statistical learning theory. 
The optimal rates for DKRR have been extended to a part of the non-attainable case via sharp analysis for the distributed error \cite{lin2020optimal} and multiple communications \cite{lin2020distributed,liu2021effective}, but these techniques are hard to improve the results for random features.
Meanwhile, some recent studies extended the capacity-independent optimality to the non-attainable, including distributed learning \cite{sun2020optimal}, random features \cite{sun2018but} and Nystr\"om approximation \cite{kriukova2017nystrom}, but the capacity-independent results are suboptimal when the capacity of RKHS is small.
\textit{The capacity-optimality for the combination of distributed learning and random features to the non-attainable case is still an open problem.}

In this paper, we aim at extending the capacity-dependent optimal guarantees to the non-attainable case and improve the computational efficiency with more partitions and fewer random features.
Firstly, using the refined estimation of operators' similarity, we refine the optimal generalization error bound that allows much more partitions and pertains to a part of the non-attainable case.
Then, generating random features in a data-dependent manner, we relax the restriction on the dimension of random features, and thus fewer random features are sufficient to reach the optimal rates.
By using additional unlabeled data to reduce label-independent error terms, we further enlarge the number of partitions and improve the applicable scope in the non-attainable case.
Finally, we validate our theoretical findings with extensive experiments.
Note that, we leave the full proofs in the appendix.

\subsection{Our Contributions}
We highlight our contributions as follows:
\begin{itemize}[leftmargin=*]
    \item \textbf{On the algorithmic front: much higher computational efficiency.}
    This work presents the currently maximum number of partitions and the minimal dimension of random features, extremely improving the computational efficiency.
    \begin{itemize}
        \item \textbf{More partitions.} 
        To achieve the optimal learning rate, the traditional distributed KRR methods \cite{lin2017distributed,guo2017learning} impose a strict constraint on the number of partitions $m \lesssim N^\frac{2r-1}{2r+\gamma}$, which heavily limits the computational efficiency.
        In this paper, using a novel estimation of the key quantity, we first relax the restriction to $m \lesssim N^\frac{2r+\gamma-1}{2r+\gamma}$.
        Then, introducing a few additional unlabeled examples, we improve the number of partitions to $m \lesssim N^\frac{2r+2\gamma-1}{2r+\gamma}$ for the first time.
        \item \textbf{Fewer random features.} 
        By generating random features in a data-dependent manner rather than in a data-independent manner, we reduce the requirement on the number of random features from $M \gtrsim N^\frac{(2r-1)\gamma+1}{2r+\gamma} \quad \forall r \in [1/2, 1]$ to $M \gtrsim N^\frac{2r + \gamma - 1}{2r+\gamma} \vee N^\frac{\gamma}{2r+\gamma} \quad\forall r \in (0, 1]$, where $M$ is the number of random features and $\vee$ indicates the bigger one.
    \end{itemize}
    \item \textbf{On the theoretical front: covering the non-attainable case.}
    The conventional optimal learning properties for KRR \cite{caponnetto2007optimal,rudi2017generalization,guo2017learning} only pertain to the attainable case $r \in [1, 1/2]$, assuming the true regression belongs to the hypothesis space $f_\rho \in \mathcal{H}$ where the problems can not be too difficult.
    However, the condition $f_\rho \in \mathcal{H}$ is too ideal and the non-attainable $r \in (0, 1/2)$ assuming $f_\rho \notin \mathcal{H}$ deserve more attention.
    In this paper, we first restate the classic results in the attainable $r \in [1/2, 1]$. 
    Then, by relaxing the restriction on the number of partitions, we extend the optimal theoretical guarantees to the non-attainable case with the constraints $2r+\gamma \geq 1$ and $2r + 2\gamma \geq 1$.
    Note that we prove KRR with random features applies to all non-attainable cases $r \in (0, 1/2)$.
    \item \textbf{Extensive experimental validation.}
    To validate our theoretical findings, we conduct extensive experiments on simulated data and real-world data. 
    We first construct simulated experiments under different difficulties to validate the learning rate and training time. Then, we perform comparison on a small real-world dataset to verify the effectiveness of data-dependence random features (with a novel approximate leverage score function) and additional unlabeled examples.
    Finally, we compare the proposed \texttt{DKRR-RF} with related work in terms of the performance on three real-world datasets.
    \item \textbf{Technical challenges.}
    \begin{itemize}
        \item \textbf{More partitions with additional unlabeled examples.}
        In the error decomposition, only sample variance is label-dependent. At the same time, other terms are label-independent, and thus we employ additional unlabeled examples to reduce the estimation of label-independent error terms.
        We further improve the applicable scope in the non-attainable case to $m \lesssim N^\frac{2r+2\gamma-1}{2r+\gamma}$. 
        \item \textbf{Random features error in all non-attainable cases.} Using an appropriate decomposition on the operatorial level for random features error, we prove KRR with random features pertains to both attainable and non-attainable case $r \in (0, 1]$.
    \end{itemize}
\end{itemize}

Overall, by overcoming several technical hurdles, we present the optimal theoretical guarantees for the combination of DKRR and RF. With more partitions and fewer random features, the theoretical results not only obtain significant computational gains but also preserve the optimal learning properties to both the attainable and non-attainable case $r \in (0, 1]$. 
Indeed, KRR \cite{caponnetto2007optimal}, DKRR \cite{guo2017learning}, and KRR-RF \cite{rudi2017generalization} are special cases of this paper.
Thus, the techniques presented here pave the way for studying the statistical guarantees of other types kernel approaches (even neural networks) that can apply to the non-attainable case.

%% file: 2-method.tex
\section{Distributed Learning with Random Feature}
In a standard framework of supervised learning,
there is a probability space $\mathcal{X} \times \mathcal{Y}$ with a fixed but unknown distribution $\rho$,
where $\mathcal{X}=\mathbb{R}^d$ is the input space and $\mathcal{Y}=\mathbb{R}$ is the output space.
The training set $D=\{(\xx_i, y_i)\}_{i=1}^N$ is sampled i.i.d. from $\mathcal{X} \times \mathcal{Y}$ with respect to $\rho$.
The primary objective is to fit the target regression $f_\rho$ on $\mathcal{X} \times \mathcal{Y}$.
The Reproducing Kernel Hilbert Space (RKHS) $\mathcal{H}$ induced by a Mercer kernel $K$ is defined as the completion of the linear span of $\{K(\xx, \cdot), \xx \in \mathcal{X}\}$ with respect to the inner product $\langle K(\xx, \cdot), K(\xx', \cdot) \rangle_{\mathcal{H}} = K(\xx, \xx')$.
In the view of feature mappings, an underlying nonlinear feature mapping $\phi:\mathcal{X} \to \mathcal{H}$ associated with the kernel $K$ is $\phi(\xx) := K(\xx, \cdot)$, so it holds $f(\xx)=\langle f, \phi(\xx)\rangle_{\mathcal{H}}$.

\subsection{Kernel Ridge Regression (KRR)}

With an RKHS norm term, kernel ridge regression (KRR) is one of the popular empirical approaches to conducting a nonparametric regression.
KRR can be stated as
\begin{align}
	\label{eq.ERM}
	\widehat{f}_\lambda := \argmin_{f \in \mathcal{H}} \left\{\frac{1}{N}\sum_{i=1}^N (f(\xx_i)-y_i)^2 + \lambda \|f\|_{\mathcal{H}}^2\right\}.
\end{align}
Using the representation theorem,
the nonlinear regression problem (\ref{eq.ERM}) admits a closed form solution $\widehat{f}_\lambda(\xx) = \sum_{i=1}^N \widehat{\alpha_i}K(\xx_i, \xx)$ with
\begin{align}
	\label{eq.krr-solution}
	\widehat{\alpha}=(\mathbf{K}_N+\lambda N \mathbf{I})^{-1}\mathbf{y}_N,
\end{align}
where $\lambda > 0, \mathbf{y}_N=[y_1, \cdots, y_N]^T$ and $\mathbf{K}_N$ is the $N \times N$ kernel matrix with $\mathbf{K}_N(i, j)=K(\xx_i, \xx_j)$.
Although KRR characterizes optimal statistical properties \cite{smale2007learning,caponnetto2007optimal},
it is unfeasible for large-scale settings because of $\mathcal{O}(N^2)$ memory to store kernel matrix and $\mathcal{O}(N^3)$ time to solve the linear system (\ref{eq.krr-solution}).

\subsection{Distributed KRR with Random Features (DKRR-RF)}
Assume that the kernel $K$ have an integral representation
\begin{equation}
	\begin{aligned}
		\label{eq.kernel_random_features}
		K(\xx, \xx') = \int_{\Omega} \psi(\xx, \omega) \psi(\xx',\omega) p (\omega) d \omega, \, \forall \xx, \xx' \in \mathcal{X},
	\end{aligned}
\end{equation}
where $(\Omega, \pi)$ is a probability space and $\psi: \mathcal{X} \times \Omega \to \mathbb{R}$.
We define analogous operators for the constructed kernel $K_M(\xx, \xx') = \phi_M(\xx)^\top \phi_M(\xx')$ to approximate the primal kernel $K(\xx, \xx')$ in \eqref{eq.kernel_random_features} with its corresponding random features via Monte Carlo sampling
\begin{align}
	\label{eq.random-features}
	\phi_M(\xx) = \frac{1}{\sqrt{M}} \big(\psi(\xx, \omega_1), \cdots, \psi(\xx, \omega_M)\big)^\top,
\end{align}
where $\omega_1, \cdots, \omega_M \in \Omega$ are sampled w.r.t $p(\omega)$.

Let the training set $D$ be randomly partitioned into $m$ disjoint subsets $\{D_j\}_{j=1}^m$ with $|D_1|=\cdots=|D_m|=n$.
The local estimator $\widehat{\ww}_j$ on the subset $D_j$ is defined as
\begin{align}
	\label{eq.local-rf}
	\widehat{\ww}_j = \argmin_{\ww \in \CC^M} \left\{\frac{1}{n}\sum_{i=1}^n (f(\xx_i)-y_i)^2 + \lambda \|f\|^2\right\},
\end{align}
where the estimator is $f(\xx) = \langle \ww, \phi_M(\xx) \rangle$.
It admits a closed-form solution 
\begin{align}
	\label{estimator.emp-dc-rf}
	\widehat{\ww}_j=\big[\Phi_{M}^\top \Phi_{M} + \lambda I\big]^{-1}\Phi_{M}^\top\widehat{y}_n,
\end{align}
where $\lambda > 0$.
Note that for $j$-th subset $D_j$, it holds
$\forall (\xx,y) \in D_j, \Phi_{M} = \frac{1}{\sqrt{n}}[\phi_M(\xx_1), \cdots, \phi_M(\xx_{n})]^\top \in \mathbb{R}^{n \times M}$
and $\widehat{y}_n = \frac{1}{\sqrt{n}}(y_1, \cdots, y_{n})^\top.$
The average of local estimators \eqref{estimator.emp-dc-rf} yields a global estimator
\begin{align}
	\label{estimator.emp-global-rf}
	\widehat{f}_{D,\lambda}^M(\xx) = \frac{1}{m}\sum_{j=1}^m  \widehat{f}_{D_j,\lambda}^M(\xx).
\end{align}


%% file: 3-theoretical-assessment.tex
\section{Theoretical Assessment}
In this section,
we present the theoretical analysis on the generalization performance of kernel ridge regression with divide-and-conquer and random features.

The generalization ability of a regression predictor $f:\mathcal{X} \to \mathbb{R}$ is measured in terms of the \textit{expected risk}
\begin{align}
    \label{eq.expected-loss}
    \mathcal{E}(f)=\int_{\mathcal{X} \times \mathcal{Y}} (f(\xx)-y)^2 d \rho(\xx, y).
\end{align}
In this case, the target regression $f_\rho = \argmin_{f} \mathcal{E}(f)$ minimizes the \textit{expected risk} over all measurable functions $f: \mathcal{X} \to \mathbb{R}$.
The generalization ability of a KRR estimator $f \in \Ltwo$ is measured by the \textit{excess risk}, i.e. $\mathcal{E}(f) - \mathcal{E}(f_\rho)$, where $L^2_{\rho_X}=\{f:\mathcal{X}\to\CC ~|~ \|f\|_\rho^2 = \int_X |f(\xx)|^2 d\rho_X < \infty\}$ is the square integral Hilbert space with respect to the marginal distribution $\rho_X$ on the input space $\mathcal{X}$.

\subsection{Assumptions}
We first introduce two standard assumptions, which are also used in statistical learning theory \cite{smale2007learning,caponnetto2007optimal,rudi2017generalization}.

\begin{assumption}[Random features are continuous and bounded]
    \label{asm.rf}
    Assume that $\psi$ is continuous and there is a $\kappa \in [1, \infty)$, such that $|\psi(\xx, \omega)| \leq \kappa, \forall \xx \in \mathcal{X}, \omega \in \Omega$.
\end{assumption}

\begin{assumption}[Moment assumption]
    \label{asm.moment}
    Assume there exists $B > 0$ and $\sigma > 0$, such that for all $p \geq 2$ with $p \in \mathbb{N}$,
    \begin{align}
        \label{eq.moment}
        \int_\mathbb{R} |y|^p d \rho(y|\xx) \leq \frac{1}{2} p ! B^{p-2} \sigma^2.
    \end{align}
\end{assumption}

According to Assumption \ref{asm.rf}, the kernel $K$ is bounded by $K(\xx, \xx) \leq \kappa^2$.
The moment assumption on the output $y$ holds when $y$ is bounded, sub-gaussian or sub-exponential.
Assumptions \ref{asm.rf} and \ref{asm.moment} are standard in the generalization analysis of KRR, always leading to the learning rate $\mathcal{O}(1/\sqrt{N})$ \cite{smale2007learning} in general cases.

\begin{definition}[Integral operators]
    \label{def.integral-operator}
    $\forall~ g \in \Ltwo(X, \rho_X)$, the integral operators $L, L_M$ are defined by the kernel $K$ and the random features $\phi_M$, respectively
    \begin{align*}
        (L g)(\cdot) &= \int_\X K(\cdot,\xx)g(\xx)d\rhox(\xx),\\
        (L_M g)(\cdot) &= \int_\X \langle \phi_M(\cdot), \phi_M(\xx) \rangle g(\xx)d\rhox(\xx).
    \end{align*}
\end{definition}

\begin{definition}[Effective dimension]
    \label{def.effective-dimension}
    The effective dimension of the RKHS $\mathcal{H}$ induce by the kernel $K$ is defined as
    \begin{align*}
        \mathcal{N}(\lambda) &= \operatorname{Tr}\big((L + \lambda I)^{-1} L\big), \quad \lambda > 0, \\
        \mathcal{N}_M(\lambda) &= \operatorname{Tr}\big((L_M + \lambda I)^{-1}L_M\big), \quad \lambda > 0.
    \end{align*}
\end{definition}

The effective dimension $\mathcal{N}(\lambda)$ is used to measure the complexity of RKHS $\mathcal{H}$, and its empirical counterpart is also called degree of freedom \cite{bach2013sharp}.
Similarly, we define the effective dimension $\mathcal{N}_M(\lambda)$ for the random features mapping $\phi_M$ to measure the size of the approximate RKHS  $\mathcal{H}_M$, which is induced by finite dimensional random features $\phi_M: \mathcal{X} \to \mathbb{R}^M$.

\begin{assumption}[Capacity assumption]
    \label{asm.capacity}
    Assume there exists $Q>0$ and $\gamma \in [0, 1]$, such that for any $\lambda > 0$
    \begin{align*}
        \mathcal{N}(\lambda) \leq Q^2\lambda^{-\gamma}.
    \end{align*}
\end{assumption}

\begin{assumption}[Regularity assumption]
    \label{asm.regularity}
    Assume there exists $R > 0$, $r > 0$,  and $g \in \Ltwo$, such that
    \begin{align*}
        f_\rho = L^r g,
    \end{align*}
    where $\frho$ is the target regression, $\|g\|_\rho \leq R$ and the operator $L^r$ denotes the $r$-th power of the integral operator $L:\Ltwo \to \Ltwo$, thus it is also a positive trace class operator.
\end{assumption}

Assumption \ref{asm.capacity} holds when the eigenvalues of the integral operator have a polynomial decay $i^{-1/\gamma}, ~ \forall i > 1$ \cite{rudi2017generalization,li2019towards}.
Thus, faster convergence rates are derived when the eigenvalues decay faster, a.k.a. $\gamma$ approaches $0$,
while $\gamma=1$ corresponds to the capacity-independent case.
Assumption \ref{asm.regularity} (source condition) controls the regularity of the target function $f_\rho$.
The bigger the $r$ is, the stronger regularity of the regression is, and the easier the learning problem is.
Both these two assumptions are widely used in the optimal theory for KRR \cite{caponnetto2007optimal,rudi2017generalization,guo2017learning}.

\subsection{General Results with Fast Rates}
One can prove the optimal generalization guarantees for \texttt{DKRR-RF} by combining the theories in KRR-DC \cite{lin2017distributed} and KRR-RF \cite{rudi2017generalization}.
The attainable case $r \in [1/2, 1]$  requires the existence of $f_\mathcal{H} = \min_{f \in \mathcal{H}} \mathcal{E}(f)$, such that $f_\rho = f_\mathcal{H}$ almost surely \cite{steinwart2008support}, which is widely used in KRR and its variants including distributed KRR and random features based KRR \cite{caponnetto2007optimal,rudi2017generalization,guo2017learning}.

\begin{figure}
    \centering
    \includegraphics[width=\linewidth]{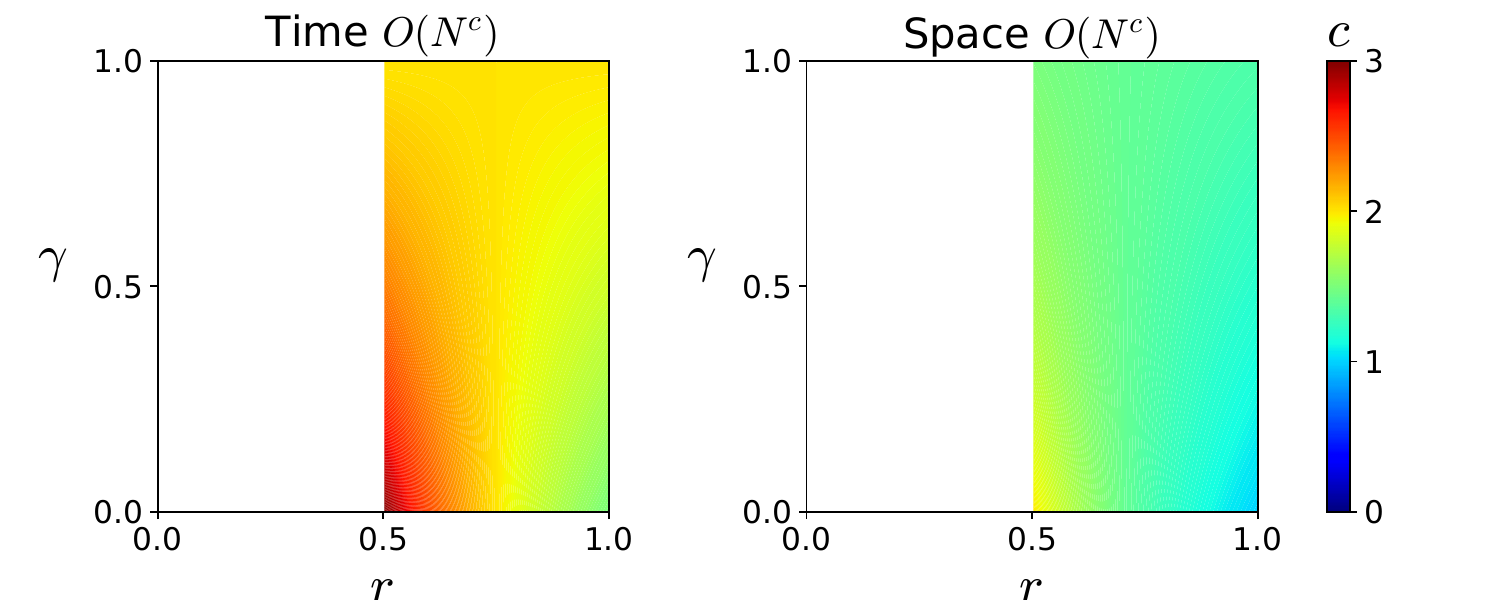}
     \caption{Time complexity and space complexity of Theorem \ref{thm.general-results} in different settings. The color closer to red represents higher complexity. Blank areas represent unfeasible situations.}
    \label{fig.complexities_general_results}
\end{figure}

\begin{theorem}
    \label{thm.general-results}
    Under Assumptions \ref{asm.rf}, \ref{asm.moment}, \ref{asm.capacity} and \ref{asm.regularity}, if $r \in [1/2, 1], \gamma \in [0, 1]$, and $\lambda = N^{-\frac{1}{2r+\gamma}},$
    then
    \begin{align*}
        1 \lesssim m \lesssim N^\frac{2r-1}{2r+\gamma}, \quad
        M \gtrsim N^{\frac{(2r-1)\gamma + 1}{2r + \gamma}},
    \end{align*}
    are enough to guarantee, with a high probability, that
    \begin{align*}
        \mathbb{E} ~ \mathcal{E}(\widehat{f}_{D,\lambda}^M) - \mathcal{E}(f_\mathcal{H})
        = \mathcal{O}\Big(N^{-\frac{2r}{2r+\gamma}}\Big).
    \end{align*}
\end{theorem}

The optimal learning rate $\mathcal{O}\left(N^{-\frac{2r}{2r+\gamma}}\right)$ stated in Theorem \ref{thm.general-results} in the above bound is optimal in a minimax sense for KRR approaches \cite{caponnetto2007optimal}.
Distributed KKR methods have obtained the same optimal error bounds with a stronger condition on the number of partitions, such as KRR-DC \cite{lin2017distributed,mucke2018parallelizing} with $m \lesssim N^\frac{2r-1}{2r+\gamma}$.
In particular, for the general case $r=1/2$, the number of local processors $m=\mathcal{O}(1)$ becomes a constant number that is independent of the sample size $N$.
The time complexity of \texttt{DKRR-RF} is $\mathcal{O}(NM^2/m)$ and the space complexity $\mathcal{O}(NM/m)$, thus we report the computational complexities of Theorem \ref{thm.general-results} in Figure \ref{fig.complexities_general_results}.

\begin{remark}
The general results in Theorem \ref{thm.general-results} have three fatal drawbacks: 1) the above bound is only suitable for the attainable case $r \in [1/2, 1]$ and fail to apply to the non-attainable case $r \in (0, 1/2)$ induced by more complicated problems; 
2) random features generated via Monte Carlo are data-independent, which requires much more features than the data-dependent generating features;
3) the constraint on the number of partitions $m \lesssim N^\frac{2r-1}{2r+\gamma}$ is too strict, leading to a constant number of partitions when $r$ is close to $1/2$.
\end{remark}

\subsection{Refined Results in the Non-attainable Case}
\begin{theorem}
	\label{thm.refined-results}
	Under Assumptions \ref{asm.rf}, \ref{asm.moment}, \ref{asm.capacity} and \ref{asm.regularity}, if $r \in (0, 1]$, $\gamma \in [0, 1]$, $2r + \gamma \geq 1$ and $\lambda = N^{-\frac{1}{2r+\gamma}},$
	then the number of partitions corresponding to
	\begin{align*}
		1 \lesssim m \lesssim N^\frac{2r+\gamma-1}{2r+\gamma}
	\end{align*}
	and the number of random features $M$ satisfying
	\begin{align*}
		M &\gtrsim N^{\frac{1}{2r + \gamma}} \quad \text{when} ~~ 0<r<1/2 \qquad \text{and} \\
		M &\gtrsim N^{\frac{(2r-1)\gamma + 1}{2r + \gamma}} \quad \text{when} ~~ 1/2 \leq r \leq 1,
	\end{align*}
	are enough to guarantee, with a high probability, that
	\begin{align*}
		\mathbb{E} ~ \mathcal{E}(\widehat{f}_{D,\lambda}^M) - \mathcal{E}(f_\rho)
		= \mathcal{O}\Big(N^{-\frac{2r}{2r+\gamma}}\Big).
	\end{align*}
\end{theorem}

Compared to Theorem \ref{thm.general-results}, Theorem \ref{thm.refined-results} allows more partitions and extends the optimal learning guarantees to the non-attainable case $r \in (0, 1/2)$ where the true regression does not lie in RKHS $\mathcal{H}$.
Thus, it achieves significant improvements in both computational efficiency and statistical guarantees.
With the same optimal learning rates, Theorem \ref{thm.refined-results} relaxes the restriction on $m$ from $m \lesssim N^\frac{2r-1}{2r+\gamma}$ to $m \lesssim N^\frac{2r+\gamma-1}{2r+\gamma}$, which allows more partitions and relaxes the constraints from $r \geq 1/2$ to $2r + \gamma \geq 1$.
When $r \in (0, 1/2)$, the number of random features $M \gtrsim N^{\frac{1}{2r + \gamma}}$ increases as the $r$ approaches zero, because $f_\rho$ becomes far away from $\mathcal{H}$ when $r$ is near zero.
When $r \in [1/2, 1]$, we obtain the same level of the number of random features $M \gtrsim N^{\frac{(2r-1)\gamma + 1}{2r + \gamma}}$ as KRR-RF \cite{rudi2017generalization}, which is continuous to $M \gtrsim N^{\frac{1}{2r + \gamma}}$ at the critical points $r=1/2$.
Compared to Figure \ref{fig.complexities_general_results}, Figure \ref{fig.complexities_refined_results} illustrates Theorem \ref{thm.refined-results} not only enlarge the applicable case but also improve the computational efficiency.

\begin{figure}[t]
    \centering
    \includegraphics[width=\linewidth]{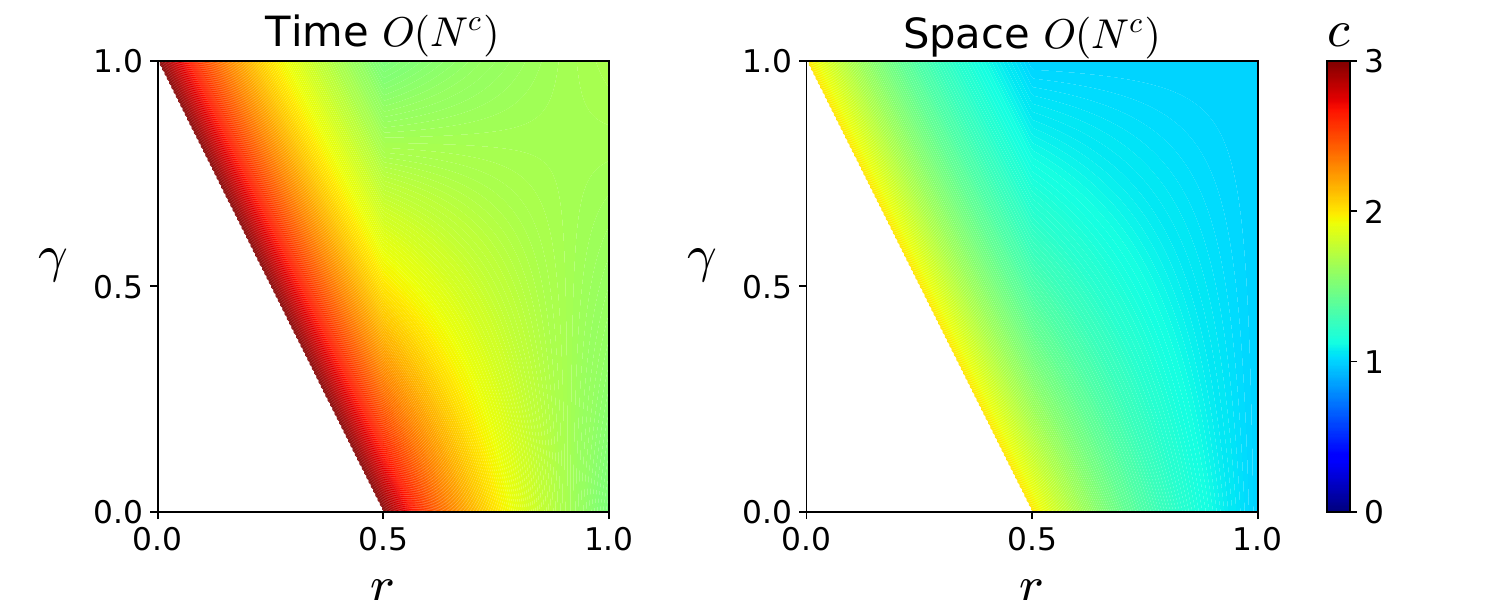}
    \caption{Time complexity and space complexity of Theorem \ref{thm.refined-results} in different settings. The color closer to red represents higher complexity. Blank areas represent unfeasible situations.}
    \label{fig.complexities_refined_results}
\end{figure}

\begin{remark}
	Theorem \ref{thm.refined-results} extends the optimal generalization theories from only attainable case $r \in [1/2, 1]$ to the non-attainable case $2r + \gamma \geq 1$, which include a part of difficult problems $r \in (0, 1/2)$.
	However, there are also many cases satisfying $2r + \gamma < 1$ in the non-attainable case $r \in (0, 1/2)$, where the optimal learning guarantees in Theorem \ref{thm.refined-results} are no longer valid.
	Inspired the literature \cite{chang2017distributed}, we employ additional unlabeled samples to relax the restriction $2r + \gamma \geq 1$ in Section \ref{sec.unlabeled_data}.
\end{remark}

\subsection{Fewer Features with Data-dependent Sampling}
\label{sec.leverage-score}

\begin{assumption}[Compatibility assumption]
    \label{asm.compatibility}
    Define the maximum effective dimension as
    $$
        \mathcal{N}_\infty(\lambda) = \sup_{\omega \in \Omega} \|(L + \lambda I)^{-1/2} \psi(\cdot,\omega)\|_{\rho_X}^2, \lambda > 0.
    $$
    Assume there exists $\alpha \in [0, 1]$ and $F > 0$, such that
    $$
        \mathcal{N}_\infty(\lambda) \leq F\lambda^{-\alpha}.
    $$
\end{assumption}

Using the definition of $\mathcal{N}(\lambda)$, we characterize the lower bounds for $\mathcal{N}_\infty(\lambda)$:
\begin{align*}
    \mathcal{N}(\lambda) &= \mathbb{E}_{\omega} \|(L + \lambda I)^{-1/2} \psi(\cdot,\omega)\|_{\rho_X}^2  \\
    &\leq  \sup_{\omega \in \Omega} \|(L + \lambda I)^{-1/2} \psi(\cdot,\omega)\|_{\rho_X}^2 = \mathcal{N}_\infty(\lambda).
\end{align*}
Compared to the (average) effective dimension used in Assumption \ref{asm.capacity}, the maximum effective dimension offers a finer-grained estimate for the capacity of RKHS \cite{alaoui2015fast,rudi2017generalization,rudi2018fast}, which often leads to shaper estimate for the related quantities.
Using the compatibility assumption, we relax the constraints on the dimension of random features and the number of partitions by generating features in a data-dependent manner, as shown in \cite{rudi2018fast,bach2017EquivalenceKernelQuadrature,li2019towards}.
\begin{theorem}
    \label{thm.data-dependent}
    Under the same assumptions of Theorem \ref{thm.refined-results} and Assumption \ref{asm.compatibility}, if $r \in (0, 1]$, $\gamma \in [0, 1]$, $2r + \gamma \geq 1$ and $\lambda = N^{-\frac{1}{2r+\gamma}},$
    then the number of partitions $m$ satisfying
    \begin{align*}  
        1 \lesssim m &\lesssim N^{\frac{2r + \gamma - 1}{2r + \gamma}}
    \end{align*}
    and the number of random features $M$ satisfying
    \begin{align*}
        M &\gtrsim N^{\frac{\alpha}{2r + \gamma}} \qquad\qquad\qquad \text{when} ~~ 0<r<1/2 \qquad \text{and} \\
        M &\gtrsim N^{\frac{(2r - 1)(1+\gamma-\alpha) + \alpha}{2r + \gamma}} \qquad \text{when} ~~ 1/2 \leq r \leq 1,
    \end{align*}
    is sufficient to guarantee, with a high probability, that
    \begin{align*}
        \mathbb{E} ~ \mathcal{E}(\widehat{f}_{D^*,\lambda}^M) - \mathcal{E}(f_\rho)
        = \mathcal{O}\Big(N^{-\frac{2r}{2r+\gamma}}\Big).
    \end{align*}
\end{theorem}

\begin{figure}[t]
    \centering
    \includegraphics[width=\linewidth]{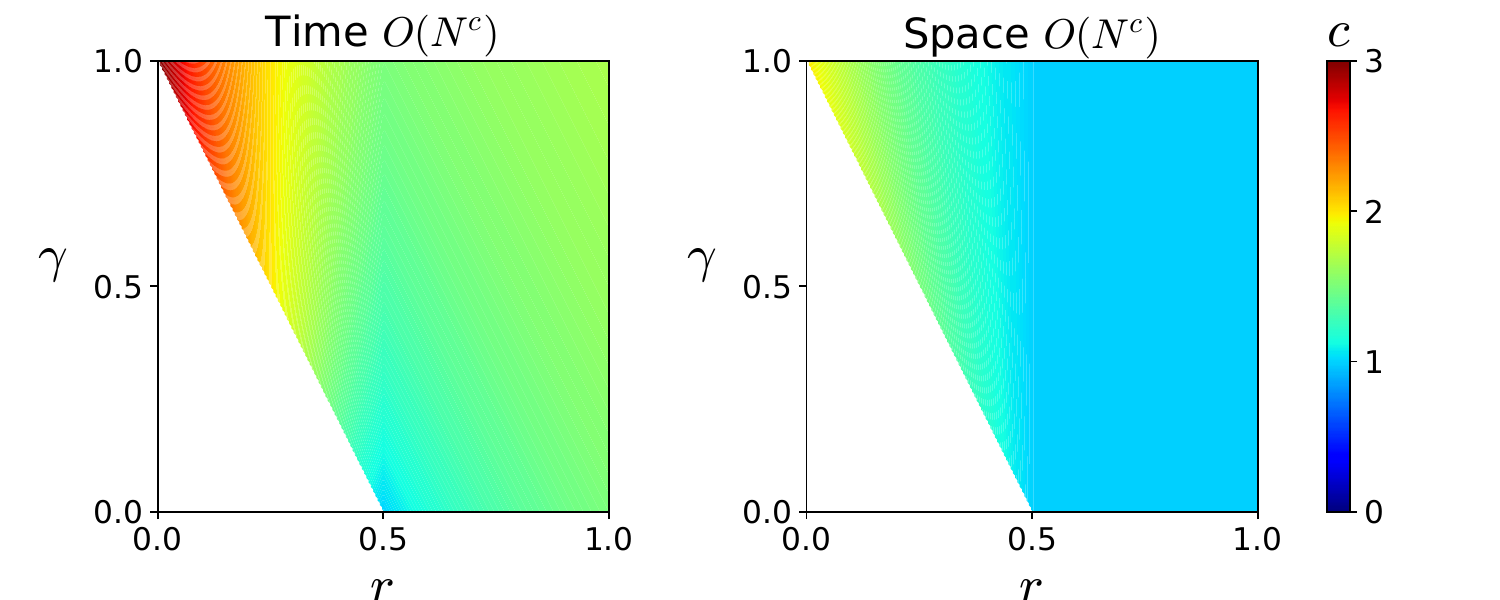}
    \caption{Time complexity and space complexity of Theorem \ref{thm.data-dependent} in different settings. The color closer to red represents higher complexity. Blank areas represent unfeasible situations.}
    \label{fig.complexities_refined_leverage}
\end{figure}

The learning rates of the above theorem are optimal, same as Theorems \ref{thm.refined-results}.
Achieving the same optimal learning rates, Theorem \ref{thm.data-dependent} reduce the computational costs with fewer random features.
The number of required random features is reduced from  $\mathcal{O}\big(N^{\frac{1}{2r + \gamma}}\big)$ to $\mathcal{O}\big(N^{\frac{\alpha}{2r + \gamma}}\big)$ when $r \in (0, 1/2)$ and $\mathcal{O}\big(N^{\frac{(2r-1)\gamma + 1}{2r + \gamma}}\big)$ to $\mathcal{O}\big(N^{\frac{(2r-1)\gamma + 1 + 2(r-1)(1-\alpha)}{2r + \gamma}}\big)$ when $r \in [1/2, 1]$, where the term $2(r-1)(1-\alpha) \leq 0$.
We report the applicable area and computational complexities of Theorem \ref{thm.data-dependent} in Figure \ref{fig.complexities_refined_leverage}.
It shows the use of data-dependent sampling significantly reduce both the time and space complexities.
The situations near the boarder line $2r+\gamma=1$ are away from the same computational complexities as the exact KRR.

\begin{remark}  
    From Theorem 1 in \cite{li2019towards}, we find that the requirement on the data-dependent random features is bounded as $M \gtrsim d_{\tilde{l}}:= \sup_{\ww \in \Omega} l_{\lambda}(\ww) / q(\ww)$, where $d_{\tilde{l}} \propto \mathcal{N}_\infty(\lambda) \leq F N^\frac{\alpha}{2r+\gamma}$.
    The condition is the same as Theorem \ref{thm.data-dependent} in the non-attainable $r \in (0, 1/2)$ and milder than Theorem \ref{thm.data-dependent} in the attainable case $r \in [1/2, 1]$.
    However, the theoretical analysis provided in \cite{li2019towards} only pertains to the general case $(r=1/2, \gamma=1)$ and obtains error bounds with the convergence rate $\mathcal{O}(1/\sqrt{N})$.
\end{remark}

\begin{remark}
    \label{remark.data_dependent}
    According to the definition of $\mathcal{N}_\infty(\lambda)$, the sampling probability of random features $\pi(\omega)$ is independent of data, which leads to a pessimistic estimate of $\alpha$.
    However, generating random features in a data-dependent manner relaxes the estimate of $\alpha$ closer to $\gamma$.
    A theoretical example of data-dependent random features was given in Example 2 \cite{rudi2017generalization}, which guarantees $\mathcal{N}_\infty(\lambda) = \mathcal{N}(\lambda)$ (such that $\alpha = \gamma$) by constructing random features generated in a data-dependent way.
    In practice, leverage sampling algorithms were proposed to obtain data-dependent random features \cite{li2019towards}, where $\alpha$ is close to $\gamma$.
    To intuitively illustrate the improvement of data-dependent random features, we boldly assume $\alpha=\gamma$ by generating data-dependent random features.
\end{remark}

\subsection{More Partitions with Unlabeled Data}
\label{sec.unlabeled_data}
In this part, we introduce the additional unlabeled samples $\widetilde{D}_j$ to relax this restriction further.
We consider the merged dataset $D^*$ on the $j$-th processor,
$
    D_j^* = D_j \cup \widetilde{D}_j
$
with
\begin{equation*}
    y_i^*=\left\{
    \begin{array}{lr}
        \frac{|D_j^*|}{|D_j|}y_i, & \quad \text{if} (\xx_i, y_i) \in D_j, \\
        0,                        & \quad \text{otherwise}.
    \end{array}
    \right.
\end{equation*}
Let $D^*=\bigcup_{j=1}^m D_j^*, |D^*|=N^*$ and $|D_1^*| = \cdots = |D_m^*|=n^*$. We define semi-supervised kernel ridge regression with divide-and-conquer and random features by
\begin{align}
    \label{f.global-sDKRR-RF}
    \widehat{f}_{D^*,\lambda}^M = \frac{1}{m}\sum_{j=1}^m  \widehat{f}_{D_j^*, \lambda}^M.
\end{align}

\begin{figure}[t]
    \centering
    \includegraphics[width=\linewidth]{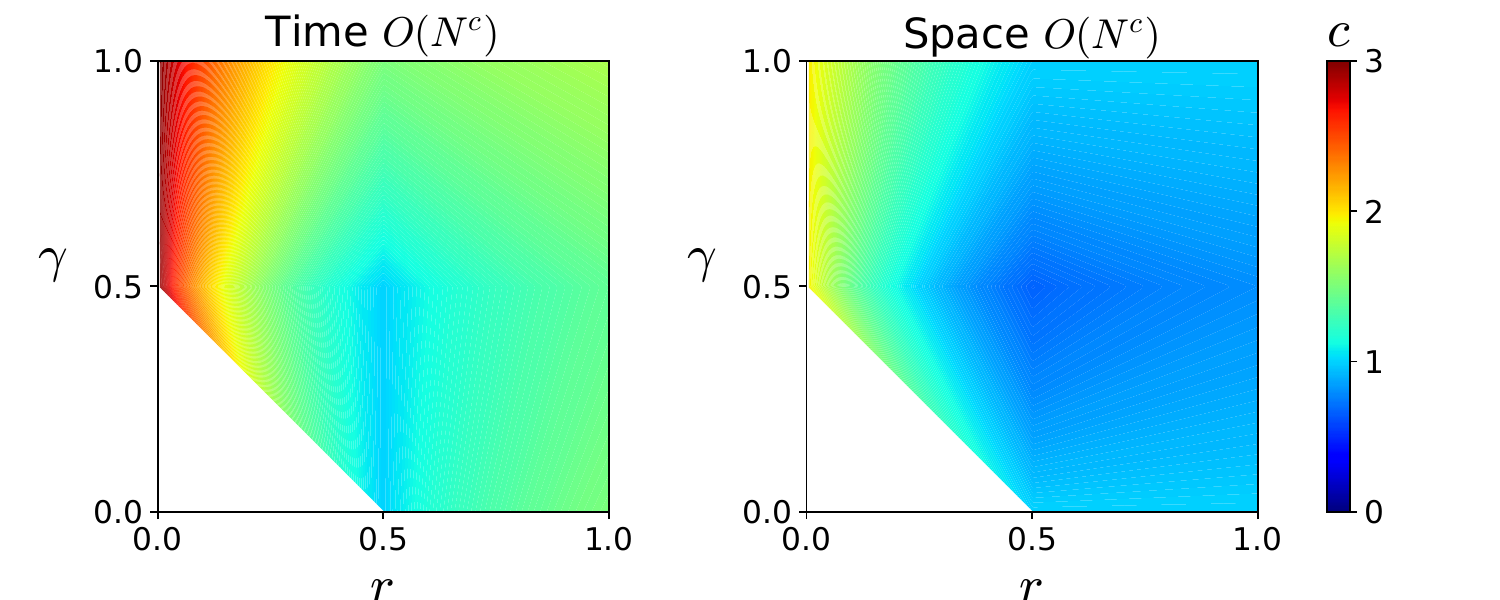}
    \caption{Time complexity and space complexity of Theorem \ref{thm.unlabeled-explicit} versus different values of $r$ and $\gamma$. The color which is closer to red represents higher complexity.}
    \label{fig.complexities_unlabeled}
\end{figure}
\begin{theorem}
    \label{thm.unlabeled-explicit}
    Under the same assumptions of Theorem \ref{thm.data-dependent}, if $r \in (0, 1], \gamma \in [0, 1], 2r + 2\gamma \geq 1$ and $\lambda = N^{-\frac{1}{2r+\gamma}},$
    then the total number of samples corresponding to
    \begin{align*}
        N^* \gtrsim NN^\frac{\gamma+\alpha-1}{2r+\gamma} \vee N,
    \end{align*}
    the number of local processors satisfying
    \begin{align*}      
        1 \lesssim m \lesssim N^\frac{2r+2\gamma-1}{2r+\gamma}
    \end{align*}
    and the number of random features $M$ satisfying
    \begin{align*}
        M &\gtrsim N^{\frac{\alpha}{2r + \gamma}} \quad \text{when} ~~ 0<r<1/2 \qquad \text{and} \\
        M &\gtrsim N^{\frac{(2r - 1)(1+\gamma-\alpha) + \alpha}{2r + \gamma}} \quad \text{when} ~~ 1/2 \leq r \leq 1,
    \end{align*}
    are sufficient to guarantee, with a high probability, that
    \begin{align*}
        \mathbb{E} ~ \mathcal{E}(\widehat{f}_{D^*,\lambda}^M) - \mathcal{E}(f_\rho)
        = \mathcal{O}\Big(N^{-\frac{2r}{2r+\gamma}}\Big).
    \end{align*}
\end{theorem}
To our best knowledge, for the first time, we prove that the number of partitions can achieve $m \lesssim N^\frac{2r+2\gamma-1}{2r+\gamma}$, while the existing constraints on $m$ of the existing work \cite{lin2020optimal,liu2021effective} are $m \lesssim N^\frac{2r+\gamma-1}{2r+\gamma}$.
Such that, much more partitions are allowed in distributed KRR methods.
The relaxation of condition on the partition number $m$ can not only lead to better computational efficiency but also covers more difficult problems, where the suitable problems are enlarged from the situation $2r + \gamma \geq 1$ to the situation $2r + 2\gamma \geq 1$.
Figure \ref{fig.complexities_unlabeled} reveals the advantages of \texttt{DKRR-RF} with unlabeled data. 
Theorem \ref{thm.unlabeled-explicit} provides the largest applicable area $2r + 2\gamma \geq 1$ but also the highest computational efficiency owing to more partitions.

\begin{remark}
    From the error decomposition, there are two error terms related to the number of partitions $m$: sample variance and empirical error.
    Sample variance depends on the number of labeled samples $n$, while empirical error is input-dependent but output-independent; thus, it is related to the number of total samples $n^*$.
    Meanwhile, the similarity between empirical and expected covariance operators $\|\tCnl^{-1/2}C_{M, \la}^{1/2}\|$ is also label-free, and thus it is related to the total sample size $n^*$ rather than $n$.
    To achieve the optimal learning rates, we consider the constraints on both the required labeled samples $n$ and the total samples $n^*$.
    Considering both conditions for supervised learning $m=N/n$ and semi-supervised learning $m=N^*/n^*$, we then obtain two constraints on the number of partitions $m$ and consolidate them together.
\end{remark}

\section{Compared with Related Work}
The existing optimal learning guarantees of KRR \cite{caponnetto2007optimal}, KRR-DC \cite{guo2017learning,mucke2018parallelizing} and KRR-RF \cite{rudi2017generalization,liu2021effective} only apply to the attainable case $r \in [1/2, 1]$.
In this paper, we apply the optimal generalization error bounds to the non-attainable case $r \in (0, 1/2)$ with some restrictions, including $2r + \gamma \geq 1$ in Theorem \ref{thm.refined-results} and $2r + 2\gamma \geq 1$ in Theorem \ref{thm.unlabeled-explicit}.
Using refined estimation, we extend the random features error to the non-attainable case.

\subsection{Applicable Area from $r \in [1/2, 1]$ to $2r + \gamma \geq 1$}

The key to obtaining the optimal learning rates with integral-operator approach is to bound the identity $\|(\tCn + \la I)^{-1/2}(C_{M} + \la I)^{1/2}\|$ as a constant, where $C_{M}$ and $\tCn$ are the expected and empirical covariance operators defined in Definition \ref{def.ops-rf}.
In conventional distributed KRR \cite{lin2017distributed,chang2017distributed}, they estimated the operator difference after first order (or second order) decomposition
	\begin{align*}
		&\|(C_M + \la I)^{-1/2}(\tCn + \la I)^{1/2}\|^2 \\
		\leq ~ &\|(C_M + \la I)^{-1/2}\|\|(C_M + \la I)^{-1/2}(C_M - \tCn)\| + 1
		\\= ~ &\mathcal{O}\left(\frac{m}{{\lambda} N} + \sqrt{\frac{{\mathcal{N}(\lambda)}m}{{\lambda} N}} \right). \quad \text{Section 4 \cite{guo2017learning}.}
	\end{align*}
    To bound the identity as a constant, the local sample size should larger enough $n \geq \frac{\mathcal{N}(\lambda)}{\lambda}$.
	 it holds $m \lesssim N^\frac{2r-1}{2r+\gamma}$ for KRR-DC and only applies to $r \geq 1/2$. 
	However, this paper directly estimates the identity in total (rather than in parts after decomposition) based on concentration inequalities for self-adjoint operators and obtain 
	\begin{align*}
		&\|(C + \la I)^{-1/2}(\tCn + \la I)^{1/2}\| \\
		\leq ~ &\left(1 - \Big\|(C + \la I)^{-1/2}(C - \tCn)(\tCn + \la I)^{-1/2}\Big\|\right)^{-1/2}
		\\= ~ &\mathcal{O}\left(\frac{m}{{\lambda} N} + \sqrt{\frac{m}{{\lambda} N}} \right). \qquad \text{Theorem \ref{thm.refined-results}}
	\end{align*}
    To bound the identity as a constant, the local sample size only needs $n \geq \frac{1}{\lambda}$, which is smaller than \cite{guo2017learning} with $\mathcal{N}(\lambda)$.
	Therefore, our estimation of $\|(C_M + \la I)^{-1/2}(\tCn + \la I)^{1/2}\|$ in Theorem \eqref{thm.refined-results} is $\sqrt{\mathcal{N}(\lambda)}$ tighter than that in \cite{guo2017learning}.
	To bound identity as a constant, we then have $m \lesssim N^\frac{2r+\gamma-1}{2r+\gamma}$, which is the key to obtain more partitions and extends the optimal learning guarantees to the non-attainable case $2r + \gamma \geq 1$.

    \subsection{Applicable Area from $2r + \gamma \geq 1$ to $2r + 2\gamma \geq 1$}

	Only sample variance is dependent on the labeled samples, while other error terms involving the estimate of $\|(C + \la I)^{-1/2}(\tCn + \la I)^{1/2}\|$ are label-free.
	Thus, there are two restrictions on the number of partitions $m$: sample variance (label-dependent) and the estimate of $\|(C + \la I)^{-1/2}(\tCn + \la I)^{1/2}\|$ (label-free).
	
	As shown in the proof of Theorem \ref{thm.data-dependent}, the global sample variance (label-dependent) can be estimated
	\begin{align*}
		\frac{1}{m}\mathbb{E} \|\widehat{f}_{D_j,\lambda}^M - \widetilde{f}_{D_j,\lambda}^M\|_\rho^2 
		~\leq~ \mathcal{O} \left(m N^\frac{1-4r-2\gamma}{2r+\gamma} + N^\frac{-2r}{2r+\gamma}\right)
	\end{align*}
	To achieve the optimal learning rates $\mathcal{O}(N^\frac{-2r}{2r+\gamma})$, the number of partitions should satisfy $m \lesssim N^\frac{2r+2\gamma-1}{2r+\gamma}$.
	Then, we utilize additional unlabeled samples to relax the condition on the estimate of $\|(C + \la I)^{-1/2}(\tCn + \la I)^{1/2}\|$.
	Using Assumption \ref{asm.compatibility}, one can further relax the condition of $m$ due to 
	\begin{align*}
		&\|(C + \la I)^{-1/2}(\tCn + \la I)^{1/2}\| 
		\\\leq ~ &\mathcal{O}\left(\frac{m\mathcal{N}_\infty(\lambda)}{N^*} + \sqrt{\frac{m\mathcal{N}_\infty(\lambda)}{N^*}} \right)
		\\= ~ &\mathcal{O}\left(\frac{m}{{\lambda^\alpha} N^*} + \sqrt{\frac{m}{{\lambda^\alpha} N^*}} \right). 
		\qquad \text{Theorem \ref{thm.unlabeled-explicit}}
	\end{align*}
	To guarantee the key quantity $\|(C + \la I)^{-1/2}(\tCn + \la I)^{1/2}\|$ be a constant, we have $m \lesssim \lambda^\alpha N^* = \mathcal{O}(N^* N^\frac{-\alpha}{2r+\gamma})$.
	We then consider the dominant constraints:
	\begin{itemize}
		\item The case $\alpha < 1 -\gamma$. It holds $2r+2\gamma-1 < 2r+\gamma-\alpha$, thus the number of partition is $m \lesssim N^\frac{2r+2\gamma-1}{2r+\gamma}$.
		\item 
		The case $\alpha \geq 1 -\gamma$. It holds $\gamma+\alpha-1 \geq 0$ and we  make use of additional unlabeled examples $N^* \gtrsim N N^\frac{\gamma+\alpha-1}{2r+\gamma}$ to guarantee $m \lesssim N^\frac{2r+\gamma-\alpha}{2r+\gamma} \leq N^\frac{2r+2\gamma-1}{2r+\gamma}$.
	\end{itemize}

	\subsection{Random Features Error in the Non-attainable Case}
	Using appropriate decomposition on operatorial level, we derive the random features error for both attainable and non-attainable case, where the dimension of random features should satisfy $M \gtrsim N^\frac{\gamma}{2r+\gamma}$ for the non-attainable case $r \in (0, 1/2)$.
	The extension from the attainable case to the non-attainable case is non-trivial, where the non-attainable case requires refined estimations for operators similarity.

	The operatorial definitions of intermediate estimators $\widetilde{f}_{D_j,\lambda}^M$, $f_{\lambda}^M$  and $f_{\lambda} $ in Lemma \ref{lem.estimator-definitions} involve the true regression $f_\rho$, where $\frho = L^r g$ (under Assumption \ref{asm.regularity}) is related the range of $r$.
	Such that, we estimate the last there error terms (empirical error $\|\widetilde{f}_{D_j,\lambda}^M - {f}_{\lambda}^M\|$, random features error $\| {f}_{\lambda}^M - {f}_{\lambda}\|$ and approximation error $\| {f}_{\lambda} - \frho\|$) that involves $\widetilde{f}_{D_j,\lambda}^M$, $f_{\lambda}^M$ and $f_{\lambda} $ for the non-attainable case.
	Meanwhile, because the empirical error satisfies $\|\widetilde{f}_{D_j,\lambda}^M - {f}_{\lambda}^M\| \leq (\sqrt{2} + 2)\left(\|f_\lambda^M - f_\lambda\| + \|f_\lambda - f_\rho\|\right)$ and the approximation error $\|f_\lambda - f_\rho\|$ naturally applies to the non-attainable case, only random features error $\| {f}_{\lambda}^M - {f}_{\lambda}\|$ is needed to specifically estimated for the non-attainable case.



%% file: 3-experiment.tex
\section{Experiments}
\begin{figure*}[t]
    \centering
    \subfigure[Target functions with difficult difficulties.]{
        \centering
        \includegraphics[width=0.45\columnwidth]{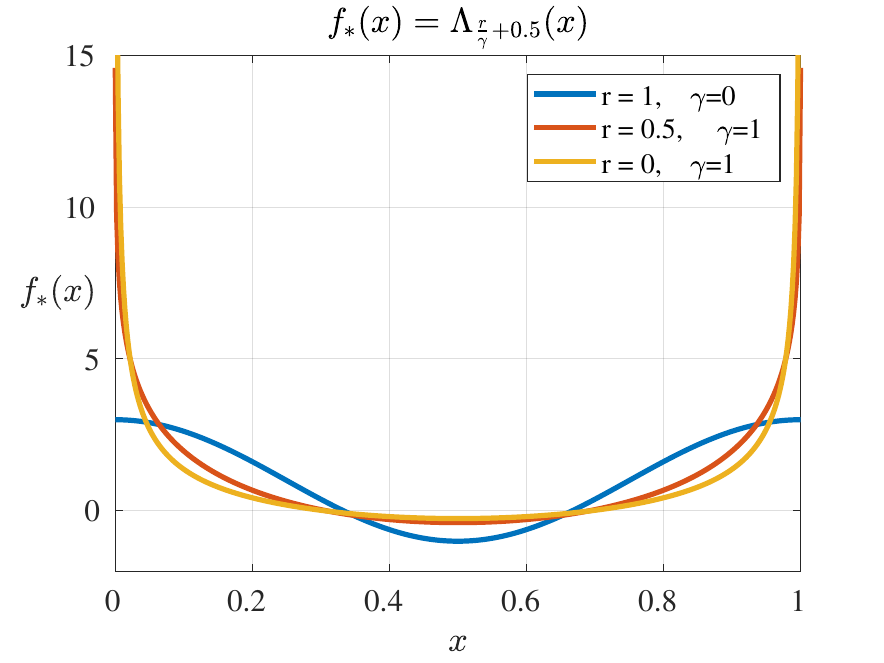}
        }
    \subfigure[Fitting with $r=1, \gamma=0$.]{
        \centering
        \includegraphics[width=0.45\columnwidth]{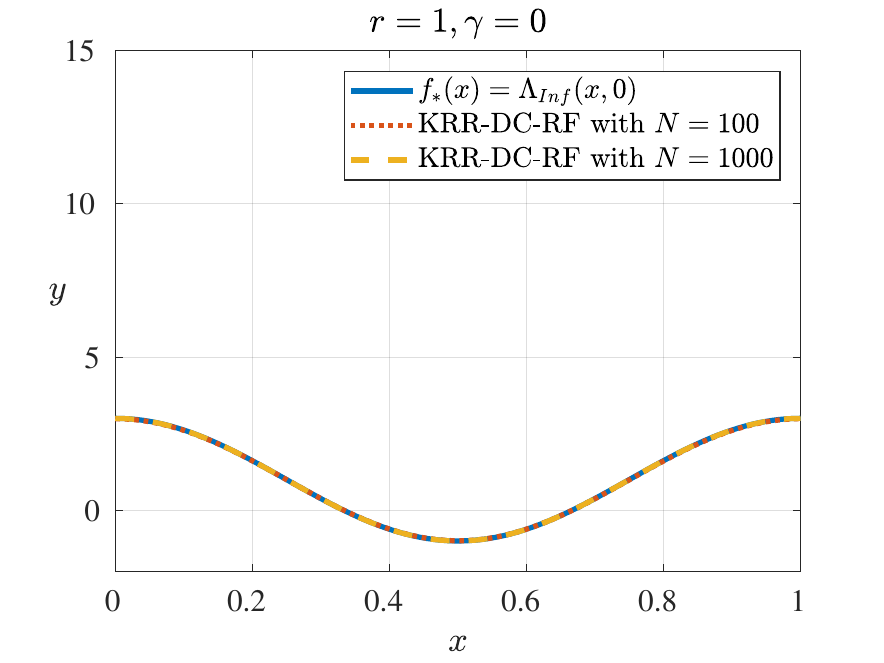}}
    \subfigure[Fitting with $r=0.5, \gamma=1$.]{
        \centering
        \includegraphics[width=0.45\columnwidth]{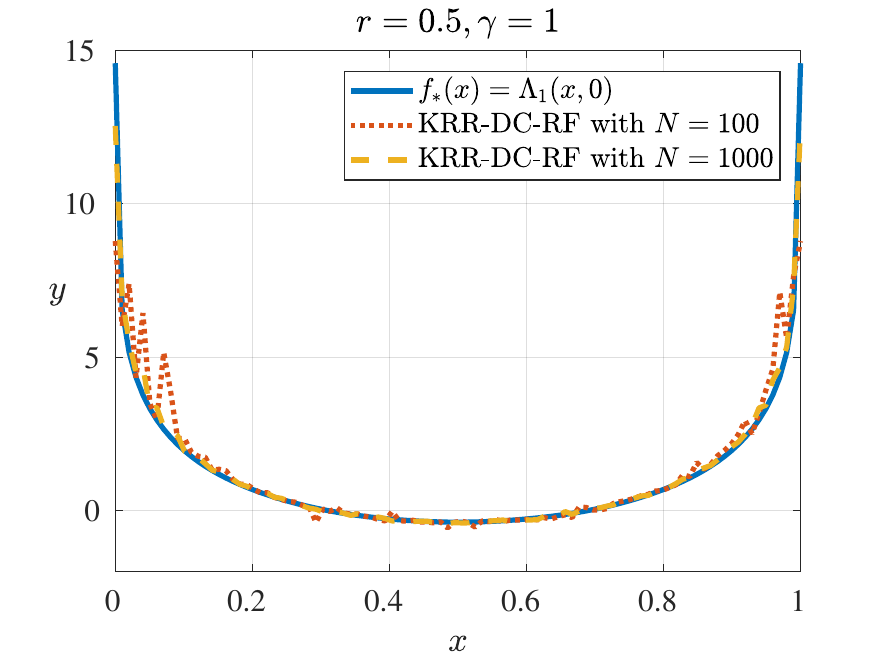}}
    \subfigure[Fitting with $r=0, \gamma=1$]{
        \centering
        \includegraphics[width=0.45\columnwidth]{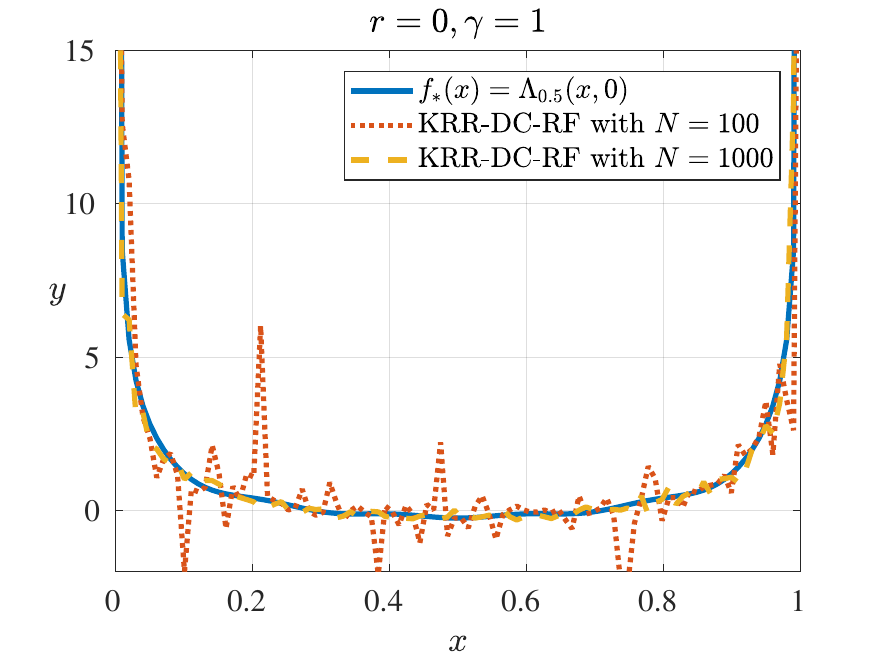}}
    \caption{Problems with different difficulties and the predicted results of DKRR-RF.}
    \label{fig.different_difficulties}
\end{figure*}

To validate the theoretical findings, we conduct experiments on both simulated data and real-world data.
In the numerical experiments, we study the computational and statistical tradeoffs of \texttt{DKRR-RF}, KRR-DC, KRR-RF, and KRR.
In real-world experiments, we first explore the effectiveness of data-dependent random features and additional unlabeled samples on a small world dataset.
Then, we compare the statistical performance of \texttt{DKRR-RF}, KRR-DC, and KRR-RF on three large-scale real-world datasets w.r.t. the number of random features $M$ and the number of partitions $m$.

\subsection{Numerical Experiments (for Theorem \ref{thm.refined-results})}
In this section, to validate our theoretical findings, we perform experiments on simulated data.
From Theorem \ref{thm.refined-results}, we find that the learning rates become slower as the ratio $\frac{\gamma}{r}$ increases, which is
\begin{align*}
    \mathcal{O}\Big(N^{-\frac{2r}{2r+\gamma}}\Big) = \mathcal{O}\Big(N^{-\frac{2}{2 + \gamma/r}}\Big).
\end{align*}
As the ratio $\gamma/r$ increases, the hardness of the problem increases.
Such that, given a fixed $\gamma$, a smaller $r$ leads to a slower converge rate of generalization error bounds.
As $r$ decreases from $1$ to near zero, the learning rates are in the range $N^{(0, \frac{-2}{2+\gamma}]}$.
Inspired by numerical experiments in \cite{rudi2017generalization,jun2019kernel}, we introduce the spline kernel of order $q \geq 2$, where more details are referred in \cite{Wahba1990smod} (Eq. 2.1.7)
\begin{align}
    \label{eq.spline-kernel}
    \Lambda_q (\xx, \xx') &= \sum_{k \in \mathbb{Z}} \frac{e^{2\pi i k(\xx - \xx')}}{|k|^q} \\
    &= 1 + 2 \sum_{k=1}^\infty \frac{\cos(2\pi k (\xx - \xx'))}{k^q}.
\end{align}
More importantly, the spline kernels naturally construct random features for any $q, q' \in \mathbb{R}$
\begin{align}
    \label{eq.spline-rf}
    \int_0^1 \Lambda_q(\xx, \zz) \Lambda_{q'}(\xx', \zz) d z = \Lambda_{q+q'}(\xx, \xx').
\end{align}
Using the following settings, we perform experiments on both easy and difficult problems
\begin{enumerate}[leftmargin=*]
    \item[-] \textbf{Input distribution}: $\mathcal{X} = [0, 1]$ and $\rho_X$ is the uniform distribution.
    \item[-] \textbf{Output distribution}: the target function $f_*(\xx) = \Lambda_{\frac{r}{\gamma} + \frac{1}{2}}(\xx, 0)$ with a variance $\epsilon^2$.
    \item[-] \textbf{Kernel and Random features}: $K(\xx, \xx') = \Lambda_\frac{1}{\gamma}(\xx, \xx')$. According to \eqref{eq.kernel_random_features} and \eqref{eq.spline-rf}, $\psi(\xx, \omega_i) = \Lambda_\frac{1}{2\gamma}(\xx, \omega_i)$ with $\omega_i$ sampled i.i.d from uniform distribution $U[0, 1].$ The random features of the spline kernel are
          \begin{align*}
              \phi_M(\xx)= M^{-1/2} \big(\psi(\xx, \omega_1), \cdots, \psi(\xx, \omega_M)\big).
          \end{align*}
\end{enumerate}

Then, conditions used in Theorem \ref{thm.refined-results} are satisfied \cite{rudi2017generalization}, including Assumption \ref{asm.capacity}, \ref{asm.regularity} with $\alpha = 1$ and no unlabeled data.
As shown in Figure \ref{fig.different_difficulties} (a), the smaller ratio $\gamma/r$ leads to a smoother curve, which corresponds to a easier problem.
We explore regression problems with different 
difficulties in terms of different settings for $r$ and $\gamma$.

According to the target regression $f_*(\xx) = \Lambda_{\frac{r}{\gamma} + \frac{1}{2}}(\xx, 0)$ and a variance $\epsilon^2=0.01$, the training data is generated  with various sample size $N \in \{1000, 2000, \cdots, 10000\}$ and $10000$ samples for testing.
To study the difference between the simulated excess risk and the theoretical excess risk, we repeat the data generating and the training $10$ times and estimate the averaged excess risk on the testing data.
On each training, we perform \texttt{DKRR-RF} $\widehat{f}_{D,\lambda}^M$ \eqref{estimator.emp-global-rf}, KRR-DC \cite{guo2017learning}, KRR-RF \cite{rudi2017generalization} and KRR \cite{caponnetto2007optimal} by evaluating both statistical performance (mean square error, MSE) and computational costs (training time).
Meanwhile, according to Theorem \ref{thm.refined-results}, we set $\lambda = N^{-\frac{1}{2r+\gamma}}$, $M = \widetilde{C} N^\frac{(2r-1)\gamma+1}{2r+\gamma} \vee N^\frac{1}{2r+\gamma}$ and $m = N^\frac{2r+\gamma-1}{2r+\gamma} / \widetilde{C}$, where $\widetilde{C}$ is an estimation of the constant $32 \kappa^2 \log(2/\delta)$. 

\begin{figure*}[t]
    \centering
    \subfigure{
        \centering
        \includegraphics[width=0.3\columnwidth]{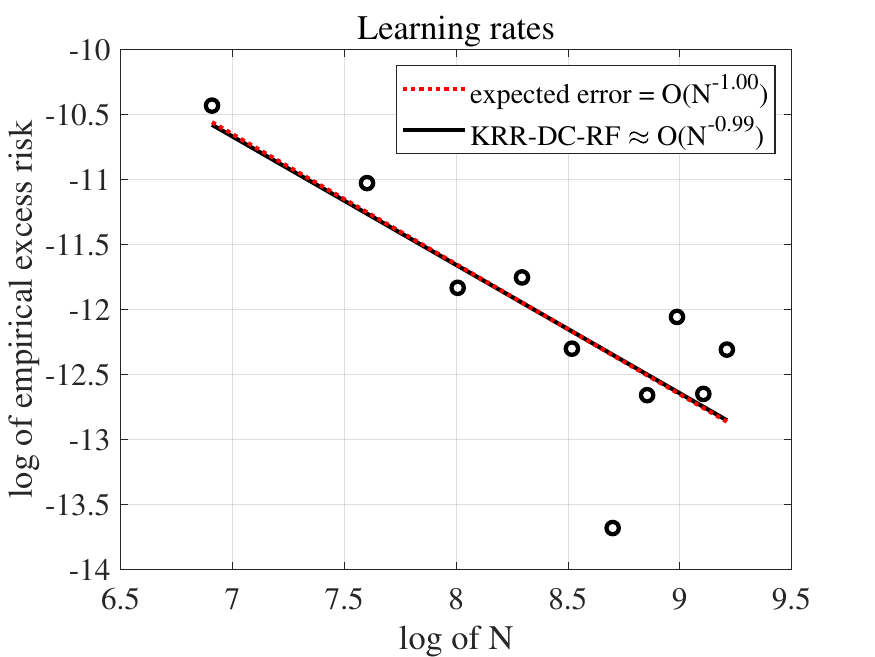}
    }
    \subfigure{
        \centering
        \includegraphics[width=0.3\columnwidth]{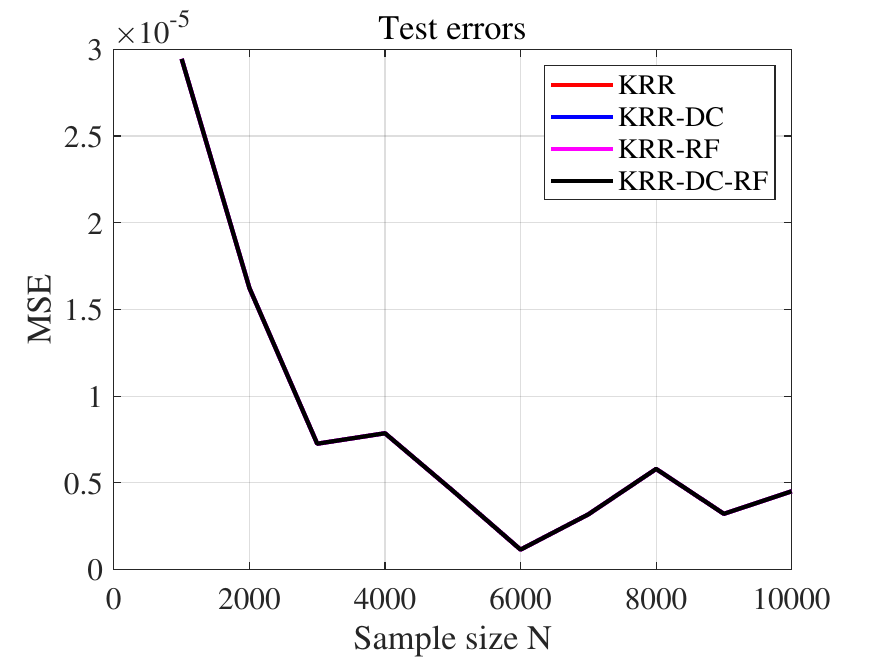}
    }   
    \subfigure{
        \centering
        \includegraphics[width=0.3\columnwidth]{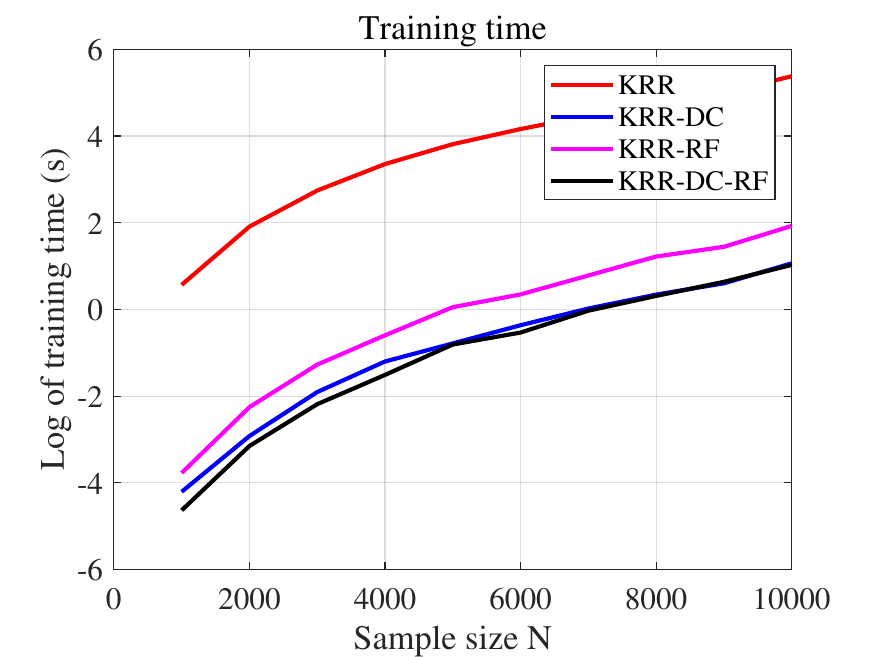}
    }
    \caption{The learning rates of expected error and \texttt{DKRR-RF} (left), the MSE (middle) and training time (right) of KRR, KRR-DC, KRR-RF, \texttt{DKRR-RF} on the easy problem with $r = 1$ and $\gamma = 0$.}
    \label{fig.easy_problem}
\end{figure*}

\begin{figure*}[t]
    \centering
    \subfigure{
        \centering
        \includegraphics[width=0.3\columnwidth]{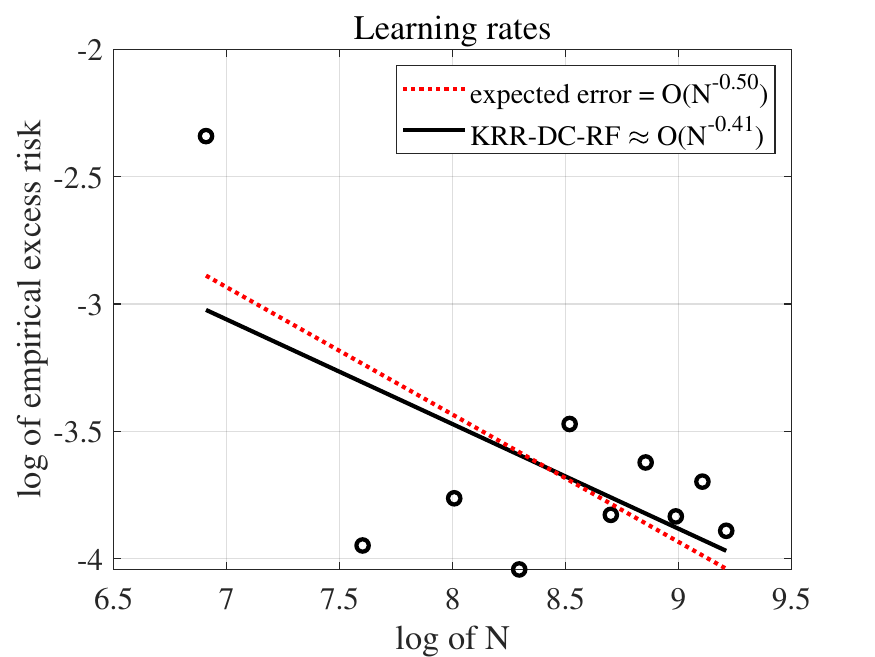}
    }
    \subfigure{
        \centering
        \includegraphics[width=0.3\columnwidth]{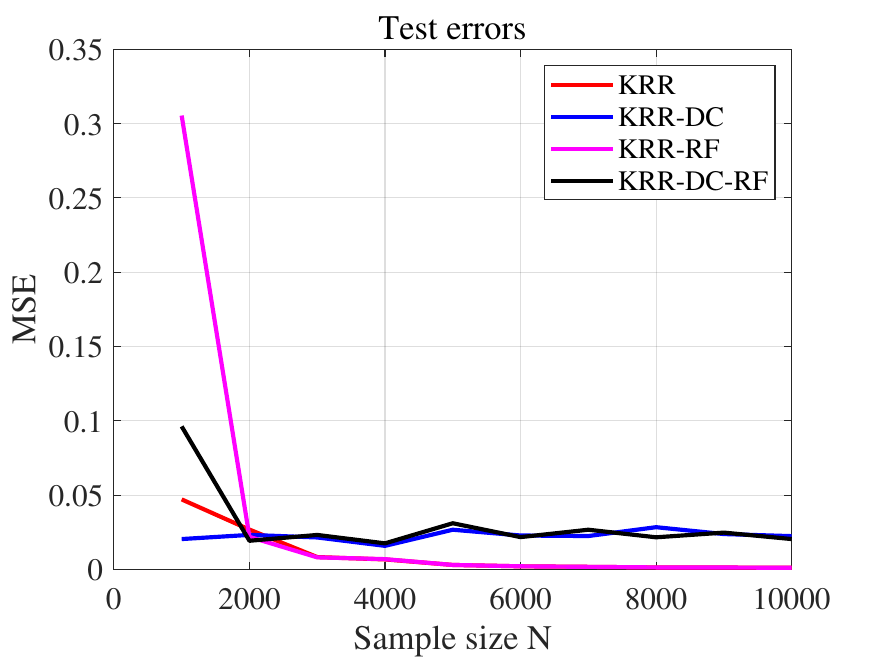}
    }
    \subfigure{
        \centering
        \includegraphics[width=0.3\columnwidth]{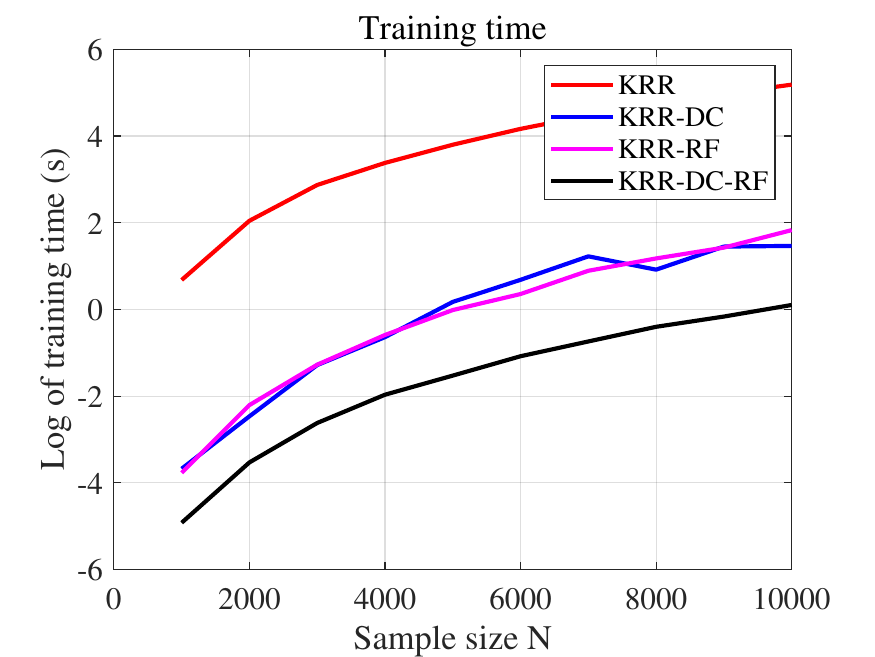}
    }
    \caption{The learning rates of expected error and \texttt{DKRR-RF} (left), the MSE (middle) and training time (right) of KRR, KRR-DC, KRR-RF, \texttt{DKRR-RF} on the general problem with $r = 0.5$ and $\gamma = 1$.}
    \label{fig.general_problem}
\end{figure*}

\begin{figure*}[t]
    \centering
    \subfigure{
        \centering
        \includegraphics[width=0.3\columnwidth]{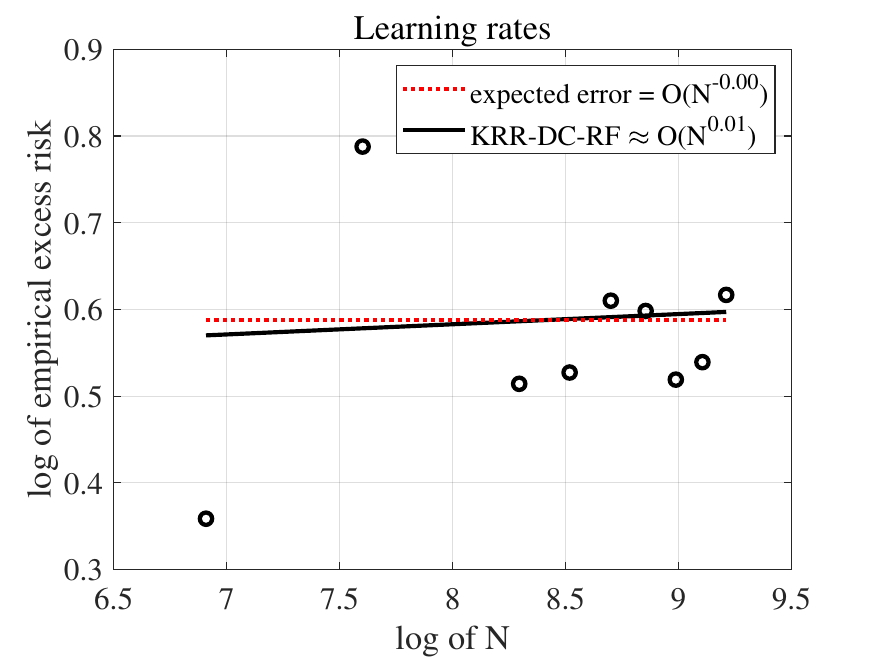}
    }
    \subfigure{
        \centering
        \includegraphics[width=0.3\columnwidth]{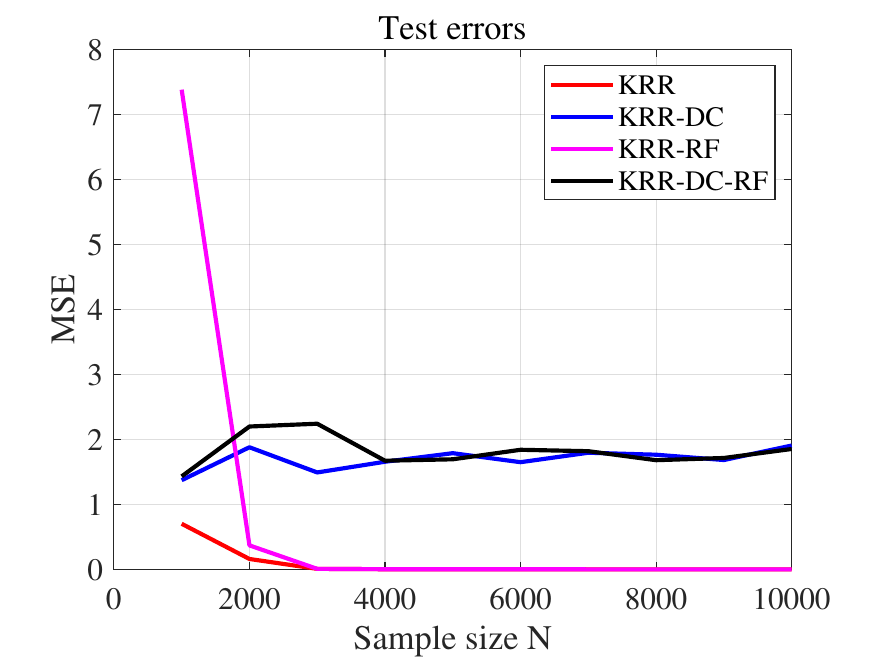}
    }
    \subfigure{
        \centering
        \includegraphics[width=0.3\columnwidth]{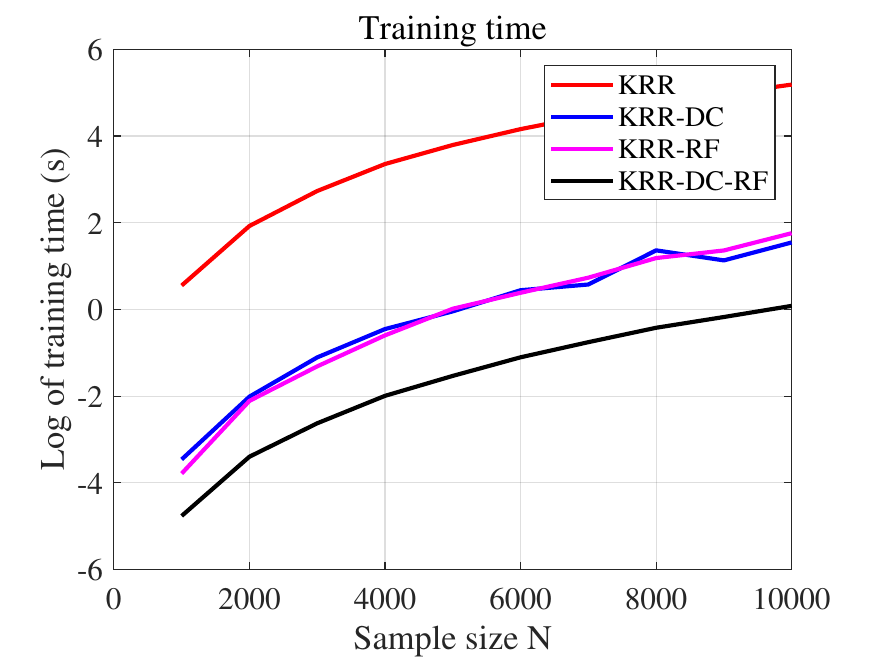}
    }
    \caption{The learning rates of expected error and \texttt{DKRR-RF} (left), the MSE (middle) and training time (right) of KRR, KRR-DC, KRR-RF, \texttt{DKRR-RF} on the general problem with $r = 0$ and $\gamma = 1$.}
    \label{fig.difficult_problem}
\end{figure*}

\subsubsection{Easy Problem}
Easy problem with the learning rate $\mathcal{O}(N^{-1})$ is given by setting $(r=1, \gamma=0)$, where the target function is $f_*(\xx) = \Lambda_{Inf}(\xx, 0) = 1 + 2\cos(2\pi\xx)$ and leads to a smooth curve in Figure \ref{fig.different_difficulties} (b).
Figure \ref{fig.different_difficulties} (b) illustrates that the problem is easy, and a smaller number of training samples ($N=100$) is enough to fit the target curve perfectly.

The left of Figure \ref{fig.easy_problem} shows the empirical learning rate of \texttt{DKRR-RF} $\mathcal{O}(N^{-0.99})$ is very close to the theoretical rate $\mathcal{O}(N^{-1})$ for the benign case $(r=1,\gamma=0)$. 
From the middle of Figure \ref{fig.easy_problem}, we find that the empirical MSE of KRR, KRR-DC, KRR-RF, and \texttt{DKRR-RF} are the same and extremely small, where the target problem is easy and even the approximate methods achieve the same optimal learning performance as the exact KRR.

\subsubsection{General Problem}
General problem with the learning rate $\mathcal{O}(N^{-1/2})$ is given by setting $(r=1/2, \gamma=1)$, which is seemed as the worst one for the attainable case $r \in [1/2, 1]$ and well-studied in \cite{rudi2017generalization,liu2021effective}.
The target function is $f_*(\xx) = \Lambda_{1}(\xx, 0).$
Figure \ref{fig.different_difficulties} (c) shows that the curve becomes sharp and thus the problem is of medium difficulty. 
A few samples $(N=100)$ bring noises, and it needs more samples $(N=1000)$ to achieve perfect fitting.

The left of Figure \ref{fig.general_problem} demonstrates the empirical error of \texttt{DKRR-RF} converges at $\mathcal{O}(N^{-0.41})$ near the expected rate $\mathcal{O}(N^{-0.5})$.
The comparison of MSE in the middle of Figure \ref{fig.general_problem} shows the errors of \texttt{DKRR-RF} are mainly due to more partitions rather than random features.
The empirical performance for the general problem in Figure \ref{fig.general_problem} (b) is much worse than the easy problem in Figure \ref{fig.easy_problem}.
The gap of test errors between distributed methods and centralized methods is negligible.
The right of Figure \ref{fig.general_problem} shows the training time of \texttt{DKRR-RF} is higher than KRR-DC when more random features are used.

\subsubsection{Diffiult Problem}
The difficult problem with the learning rate $\mathcal{O}(1)$ is given by setting $(r=0, \gamma=1)$, which is almost unable to be learned.
According to Figure \ref{fig.different_difficulties} (d), we find that $(r=0, \gamma=1)$ provides the difficult problem where the curve steepens rapidly near $0$ or $1$.
A large number $(N=1000)$ of training samples are still unable to fit the curve perfectly.

From the left of Figure \ref{fig.difficult_problem}, we can see that the learning rate is near $\mathcal{O}(1)$, such that the target function is hard to learn.
The middle and right of Figure \ref{fig.difficult_problem} illustrates that errors and training time are similar for KRR, KRR-RF, KRR-DC, and DKRR-RF.
Compared with test errors of the easy problem in Figure \ref{fig.easy_problem} (b) and the general problem in Figure \ref{fig.general_problem} (b), MSE of the difficult problem in Figure \ref{fig.difficult_problem} (b) is much higher and the performance gab between distributed learning and centralized learning is significant.
Therefore, distributed learning approaches are not suitable for difficult problems when $r$ approaches zero, which coincides with the theoretical findings in Theorem \ref{thm.refined-results}.

\begin{figure*}[t]
    \centering
    \subfigure[Test accuracy]{
        \centering
        \includegraphics[width=0.3\columnwidth]{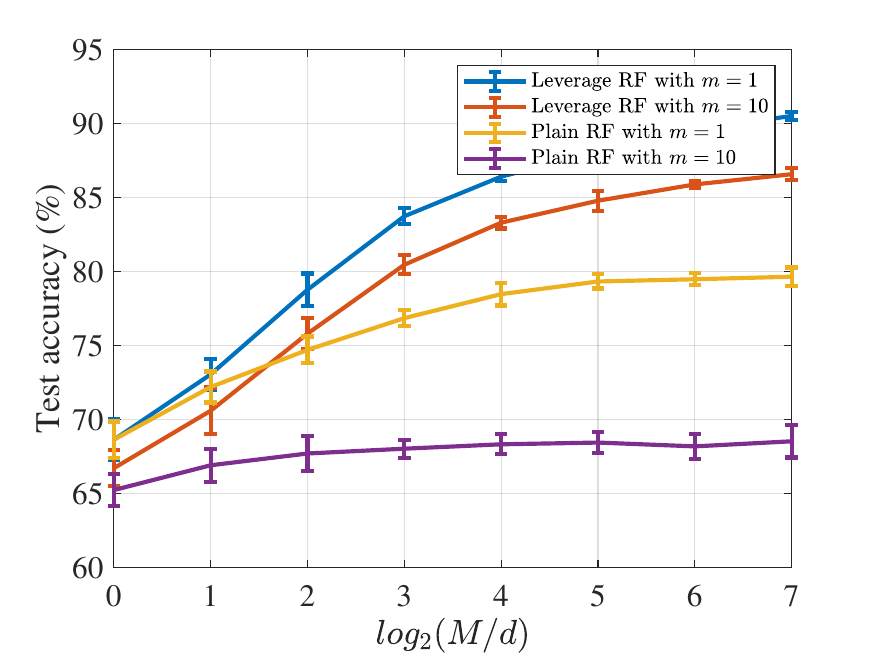}
    }
    \subfigure[Generating time]{
        \centering
        \includegraphics[width=0.3\columnwidth]{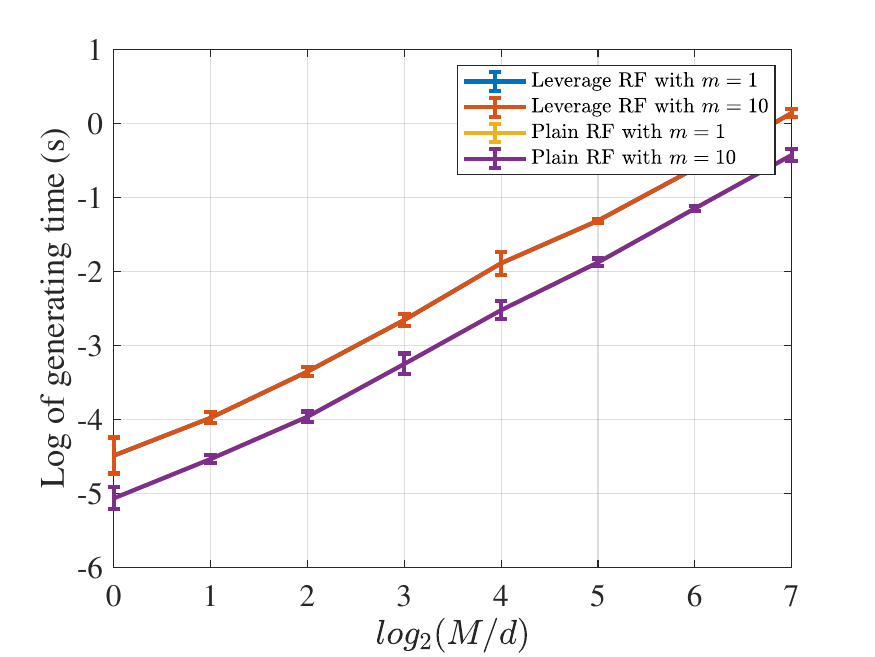}
    }
    \subfigure[Training time]{
        \centering
        \includegraphics[width=0.3\columnwidth]{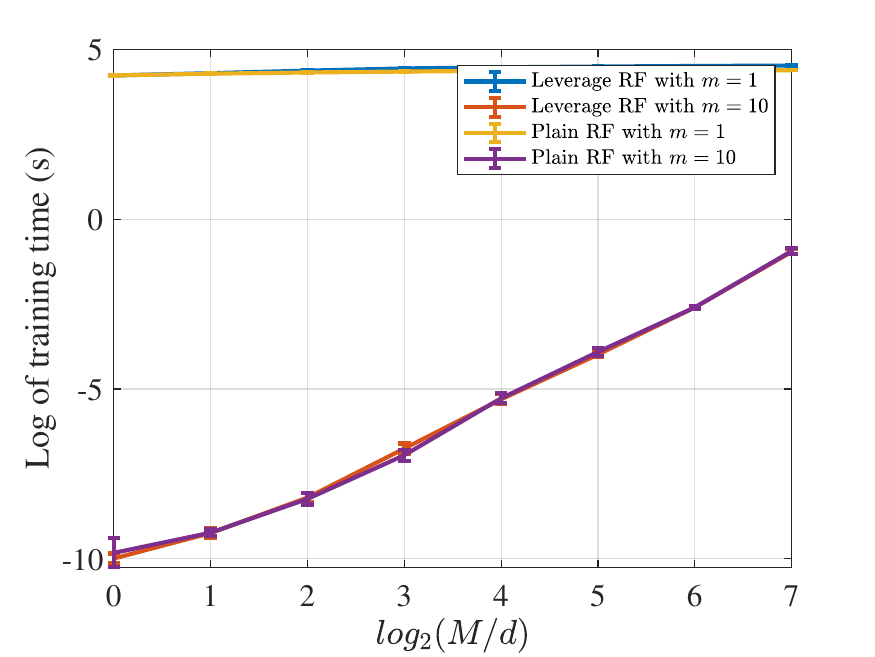}
    }
    \caption{Comparison of random feature generating algorithms on test accuracy (a), time cost for generating random features (b), and time cost for training \texttt{DKRR-RF} (c) versus the number of random features $M$ on the \textit{EGG} dataset.}
    \label{fig.leverage_rf}
\end{figure*}

\begin{figure*}[t]
    \centering
    \subfigure[Test accuracy]{
        \centering
        \includegraphics[width=0.3\columnwidth]{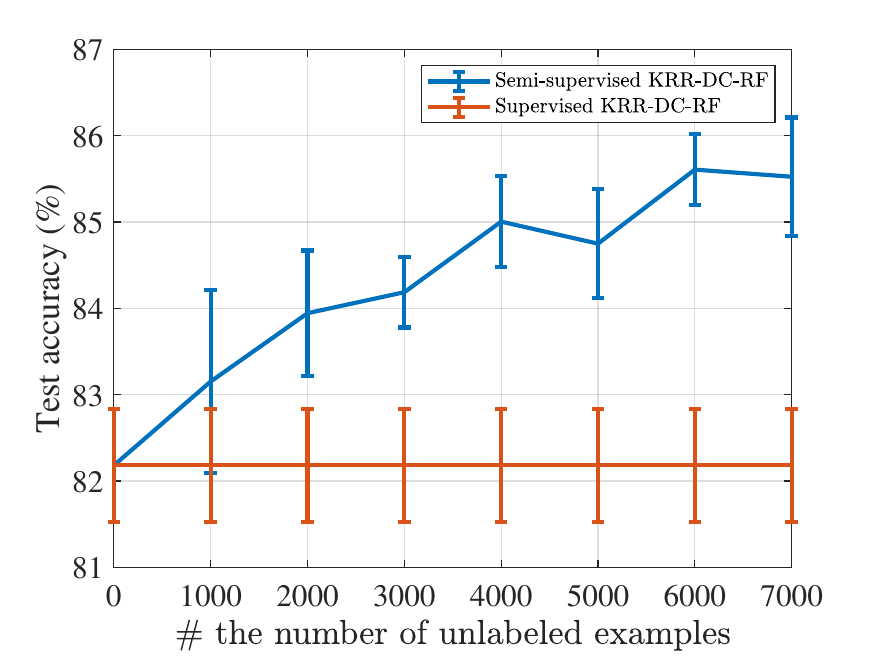}
    }
    \subfigure[Generating time]{
        \centering
        \includegraphics[width=0.3\columnwidth]{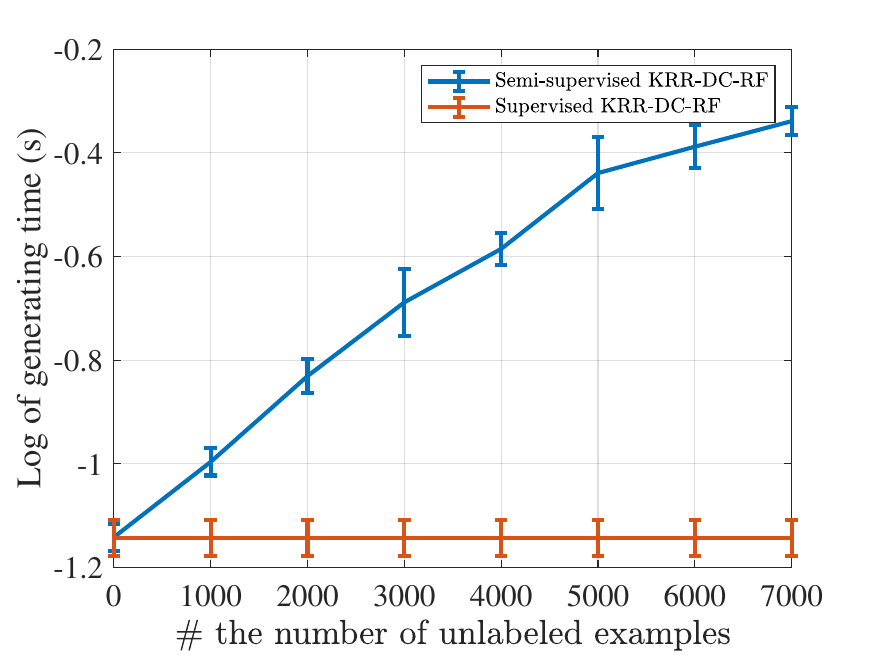}
    }
    \subfigure[Training time]{
        \centering
        \includegraphics[width=0.3\columnwidth]{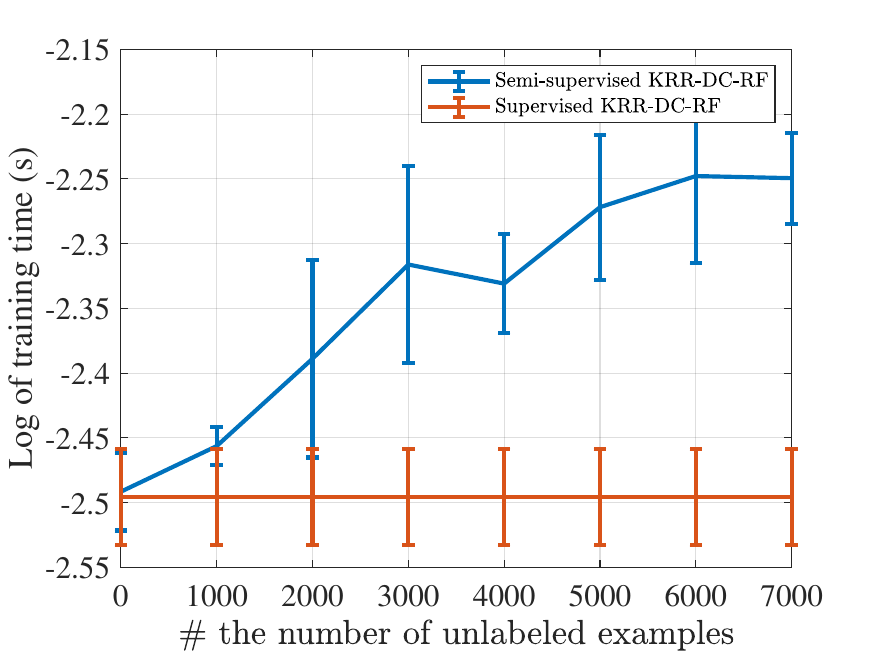}
    }
    \caption{Comparison of semi-supervised/supervised algorithms on test accuracy (a), time cost for generating random features (b), and time cost for training \texttt{DKRR-RF} (c) versus the number of unlabeled examples on the \textit{EGG}$^*$ dataset.}
    \label{fig.ssl_unlabeled}
\end{figure*}

\begin{figure*}[t]
    \centering
    \subfigure[Test accuracy]{
        \centering
        \includegraphics[width=0.3\columnwidth]{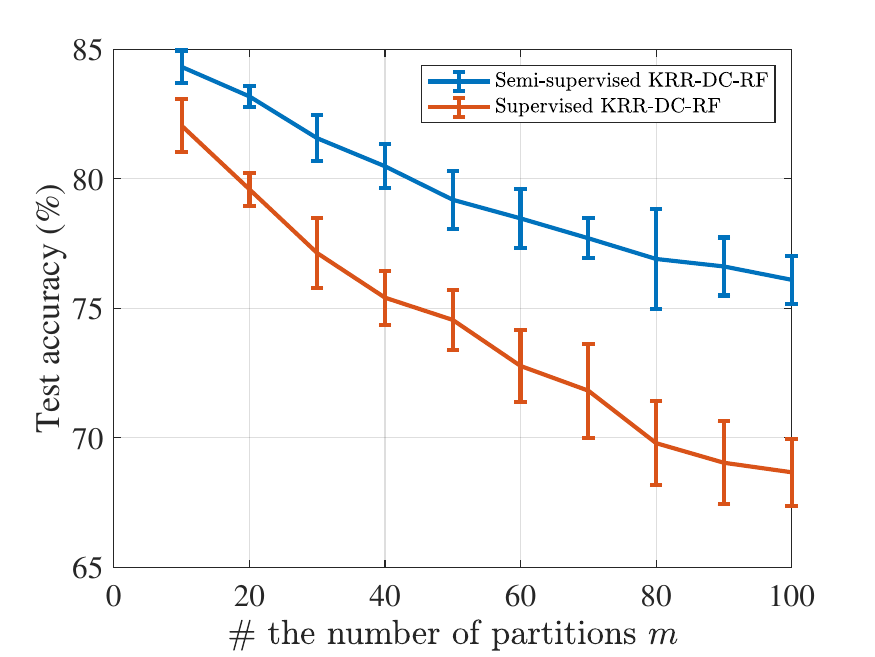}
    }
    \subfigure[Generating time]{
        \centering
        \includegraphics[width=0.3\columnwidth]{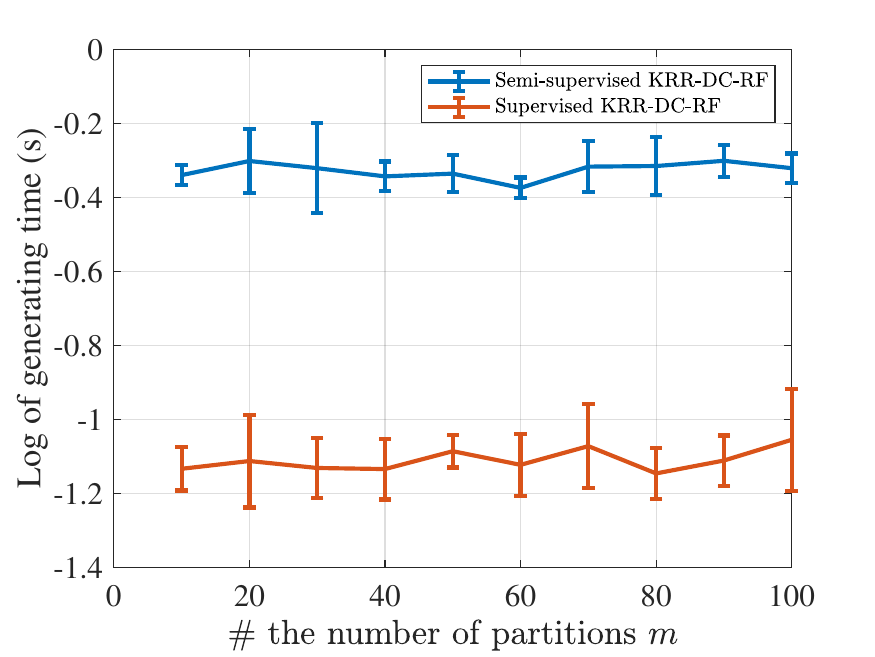}
    }
    \subfigure[Training time]{
        \centering
        \includegraphics[width=0.3\columnwidth]{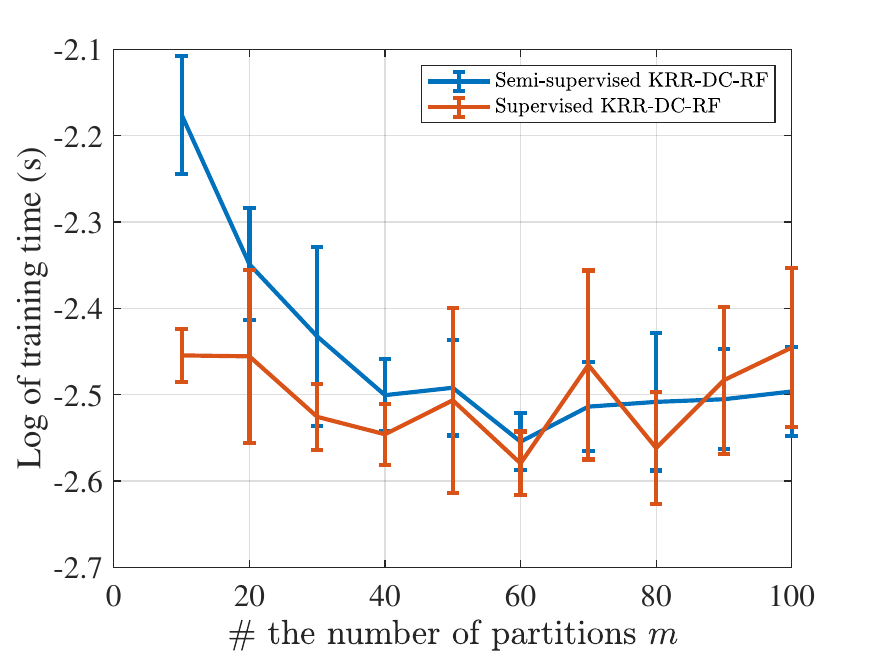}
    }
    \caption{Comparison of semi-supervised/supervised algorithms on test accuracy (a), time cost for generating random features (b), and time cost for training \texttt{DKRR-RF} (c) versus the number of unlabeled examples on the \textit{EGG}$^*$ dataset.}
    \label{fig.ssl_m}
\end{figure*}

\subsection{Influence of data-dependent random features (for Theorem \ref{thm.data-dependent})}

Inspired by the leverage weighted random Fourier features \cite{avron2017faster,li2019towards,liu2020random}, we proposed the leverage weighted random features (not just for shift-invariant kernels).
Based on \eqref{eq.random-features}, the data-dependent random features are defined as
\begin{align}
    \label{eq.leverage_rf}
    \varphi(\xx) = \frac{1}{\sqrt{M}} \left[\frac{p(\omega_1)}{q(\omega_1)}\psi(\xx, \omega_1), \cdots, \frac{p(\omega_M)}{q(\omega_M)}\psi(\xx, \omega_M)\right]^\top,
\end{align}
where $q(\ww_i) = \tilde{l}_{\lambda}(\ww_i) / D_{\mathbf{K}}^\lambda$.
Using the ideal matrix $\yy\yy^\top$ to replace the kernel matrix, we obtain the following leverage score function
\begin{align*}
    l_{\lambda}(\ww_i) = p(\ww_i) \mathbf{z}_{\ww_i}({\boldsymbol X})^\top \left(\frac{1}{n^2\lambda} \Big(\yy\yy^\top + n \mathbf{I}\Big)\right) \mathbf{z}_{\ww_i}({\boldsymbol X}),
\end{align*}
where $\mathbf{z}_{\ww_i}({\boldsymbol X}) = 1/\sqrt{L}\big[\psi(\xx_1, \ww_i), \cdots, \psi(\xx_n, \ww_i)\big]^\top$ and $L \geq M$.
Removing the data-independent terms, there holds
\begin{align*}
    \tilde{l}_{\lambda}(\ww_i) = p(\ww_i) \mathbf{z}_{\ww_i}({\boldsymbol X})^\top \yy\yy^\top \mathbf{z}_{\ww_i}({\boldsymbol X}) 
    = p(\ww_i) [\yy^\top \mathbf{z}_{\ww_i}({\boldsymbol X})]^2 
\end{align*}
and 
\begin{align*}
    D_{\mathbf{K}}^\lambda = \int_{\mathbb{R}^d} \tilde{l}_{\lambda}(\ww) d\ww = \text{Tr} \left[\yy\yy^\top\mathbf{K}\right] \approx \sum_{i=1}^L [\yy^\top \mathbf{z}_{\ww_i}({\boldsymbol X})]^2.
\end{align*}
The time complexity of generating data-dependent random features is $\mathcal{O}(NM^2)$ on a global machine, while it can be further reduced by computing in local machines and as a part of data preprocessing.
Then, we re-sample $M$ features from $\{\ww\}_{i=1}^L$ using the multinomial distribution given by $p(\ww_i) / q(\ww_i) = [\yy^\top \mathbf{z}_{\ww_i}({\boldsymbol X})]^2 /\sum_{i=1}^M [\yy^\top \mathbf{z}_{\ww_i}({\boldsymbol X})]^2$.
Then, using \eqref{eq.leverage_rf}, we compute the data-dependent random features.

\begin{table}[t]
    \centering
    \renewcommand\arraystretch{1.5}
    \begin{tabular}{llllll}
        \toprule
        datasets & $d$ & $\#$ training & $\#$ testing & $\sigma$ & $\lambda$\\ \hline
        EEG & $14$ & $7,490$ & $7,490$ & $1$ & $2^{-4}$\\ \hline
        EEG$^*$ & $14$ & $3,000(7,000^*)$ & $4,980$ & $1$ & $2^{-4}$\\ \hline
        covtype & $54$ & $250,000$ & $50,000$ & $2$  & $2^{-6}$\\ \hline
        SUSY & $18$ & $250,000$ & $50,000$ & $2^{2.5}$ &$2$\\ \hline
        HIGGS & $28$ & $250,000$ & $50,000$ & $2^3$ &$2^{-6}$\\ 
        \bottomrule
    \end{tabular}
    \caption{Datasets statistics. EEG$^*$ indicates $3000$ examples are used as labeled examples, $7000$ examples are used as unlabeled ones and the other $4980$ examples as the test data.}
    \label{tab.dataset}
\end{table}

To validate Theorem \ref{thm.data-dependent}, we compare the empirical performance of the following methods:
\begin{itemize}[leftmargin=*]
    \item Leverage RF with $m=1$: the proposed approximate leverage weighted random features \eqref{eq.leverage_rf} without distributed learning, similar to \cite{liu2020random}.
    \item Leverage RF with $m=10$: the proposed approximate leverage weighted random features \eqref{eq.leverage_rf} with $10$ partitions, a.k.a. \texttt{DKRR-RF} in Theorem \ref{thm.data-dependent}.
    \item Plain RF with $m=1$: the exact random features with Monte Carlo sampling \eqref{eq.random-features}, which is KRR-RF given in \cite{rudi2017generalization}.
    \item Plain RF with $m=10$: the exact random features with Monte Carlo sampling \eqref{eq.random-features} with $10$ partitions, namely \texttt{DKRR-RF} defined in Theorem \ref{thm.refined-results}.
\end{itemize}
In terms of different values of $M$, we perform different random features generating algorithms on the \textit{EGG} dataset to evaluate the test accuracies, time costs for generating random features, and time costs for $10$ trials.
Figure \ref{fig.leverage_rf} reports the mean and one standard deviation of test accuracies, time costs for generating random features, and time costs for training versus different settings of $M=d \times 2^{\{0, \cdots, 7\}}$.
We use Gaussian kernel $K(\xx, \xx') = \exp(-{\|\xx - \xx'\|^2}/{2\sigma^2})$ in experiments and the corresponding probability density function $p(\ww) = \mathcal{N}(0, 1/\sigma^2)$.
The kernel parameter $\sigma$ and the regularity parameter $\lambda$ are tuned via $5$-folds cross-validation over grids of $2^{\{-5, -4.5, \cdots, 5\}}$ and $2^{\{-7, -6, \cdots, 3\}}$.
The statistical information of the dataset and hyperparameter settings are reported in Table \ref{tab.dataset}.
From Figure \ref{fig.leverage_rf}, we find that: 
\begin{itemize}
    \item [1)] Centralized learning with data-dependent random features (blue) achieves the best accuracy but also leads to the highest computational costs for training the model on a centralized machine, while the generating times for centralized learning and distributed learning are the same.
    \item [2)] With a slight increase in random features generating time in Figure \ref{fig.leverage_rf} (b), two kinds of data-dependent approaches (blue and red ones) brings significant improvements on the classification accuracy as the number of random features increases in Figure \ref{fig.leverage_rf} (a), which validates the effectiveness of data-dependent random features.
    \item [3)] The use of data-dependent features generating approaches does not sacrifice too much computational efficiency as shown in Figure \ref{fig.leverage_rf} (b).
    Meanwhile, in Figure \ref{fig.leverage_rf} (c), distributed learning dramatically improves training efficiency.
    \item [4)]
    Consuming a little bit more generating time and similar training time, data-dependent \texttt{DKRR-RF} (red) markedly surpasses data-independent \texttt{DKRR-RF} (purple), which reveals the superiority of Theorem \ref{thm.data-dependent} than Theorem \ref{thm.refined-results}.
\end{itemize}

Overall, data-dependent \texttt{DKRR-RF} achieves a good tradeoff on accuracy and efficiency, which coincides with the theoretical findings in Theorem \ref{thm.data-dependent}.

\subsection{Influence of Unlabeled Data (For Theorem \ref{thm.unlabeled-explicit})}
To validate Theorem \ref{thm.unlabeled-explicit}, we split the EGG dataset into three parts: $3000$ examples as the labeled training data, $7000$ ones as the unlabeled training data, and $4980$ examples as the test data, which is illustrated in Table \ref{tab.dataset}.
There are two compared methods:
\begin{itemize}[leftmargin=*]
    \item Semi-supervised \texttt{DKRR-RF} (defined in Theorem \ref{thm.unlabeled-explicit}) is constructed as \eqref{f.global-sDKRR-RF}, where the unlabeled examples are marked as zeros.
    \item Supervised \texttt{DKRR-RF} (defined in Theorem \ref{thm.data-dependent}) only uses the labeled samples.
\end{itemize}

In the following experiments, we fix the labeled sample size $N=3000$, the number of random features $M = 1,000$ and the number of partitions $m = 10$.
In Figure \ref{fig.ssl_unlabeled}, we perform the compared methods across $10$ trials to plot the mean and one standard deviation of test accuracies, time costs for random feature generating, and times costs for training under varying unlabeled samples size $N^* - N \in \{0, 1000, \cdots, 7000\}$.
Figure \ref{fig.ssl_unlabeled} illustrates that:
\begin{itemize}
    \item[1)] The use of additional unlabeled samples improves the empirical performance of DKRR-RF. In other words, we can increase the number of partitions without losing accuracy by using additional unlabeled samples.
    It is consistent with the theoretical findings in Theorem \ref{thm.unlabeled-explicit} that additional unlabeled examples can relax the restriction on $m$.
    \item[2)] The increase of the number of unlabeled examples aggravates the computational burden. 
    It is worthy of balancing the accuracy gains and the computational costs in terms of different sizes of unlabeled data. 
    Generating time is larger than training time, thus generating time dominates in the time-consuming.
\end{itemize}

In Figure \ref{fig.ssl_m}, we fixed unlabeled sample size as $N^*-N = 7000$ and perform semi-supervised/supervised methods on  the different number of partitions $m \in \{10, 20, \cdots, 100\}$.
Figure \ref{fig.ssl_m} reports the mean and one standard deviation of test accuracies and training times under different partitions, which shows:
\begin{itemize}
    \item[1)] Both two test accuracies decrease as the number of partitions $m$ increases, and the accuracy gap between semi-supervised \texttt{DKRR-RF} and supervised \texttt{DKRR-RF} becomes larger and larger.
    \item[2)] Both the training time drops as the number of partitions increases, but there are no more significant computational gains when the number of partitions is greater than $50$.
    Both the generating times of semi-supervised \texttt{DKRR-RF} and supervised \texttt{DKRR-RF} is almost stationary and plays a dominant role.
    Semi-supervised \texttt{DKRR-RF} costs much more time for generating data-dependent random features than supervised  \texttt{DKRR-RF}.
    \item[3)] Semi-supervised \texttt{DKRR-RF} (Theorem \ref{thm.unlabeled-explicit}) always provides better empirical performance than supervised \texttt{DKRR-RF} (Theorem \ref{thm.data-dependent}), but also the training times of them are similar when $m \geq 50$.
\end{itemize}
Therefore, semi-supervised \texttt{DKRR-RF} achieves a good balance between the test accuracy and the training time at $m=50$.

\subsection{Large-scale Real Data}
\begin{figure}[t]
    \centering
    \subfigure{
        \centering
        \includegraphics[width=0.48\columnwidth]{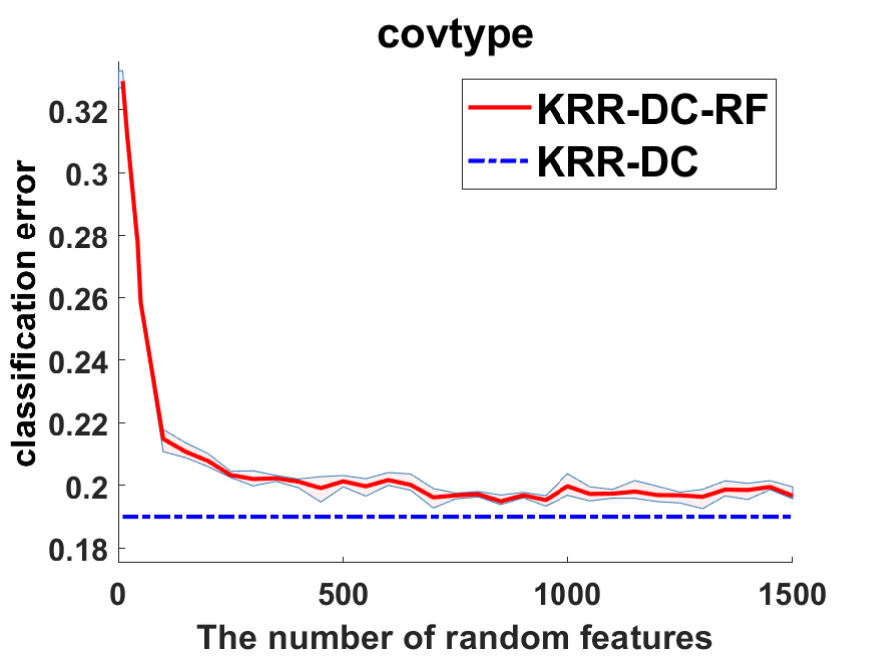}}
    \subfigure{
        \centering
        \includegraphics[width=0.48\columnwidth]{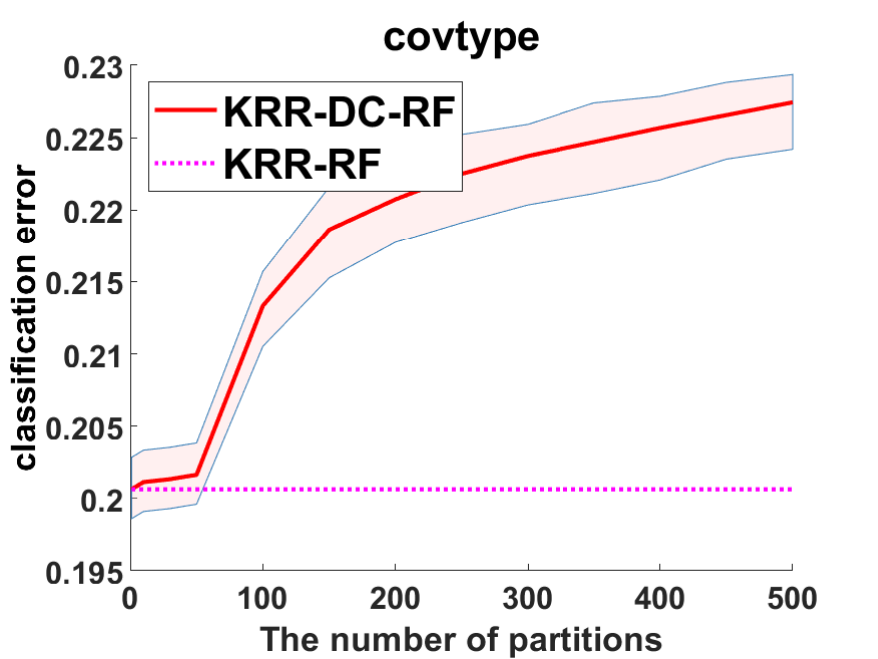}}
    \caption{Classification error of covtype dataset in terms of different number of random features (left) and different number of partitions (right).}
    \label{figure.exp1.covtype} 
\end{figure}

As listed in Table \ref{tab.dataset}, we study the empirical performance of \texttt{DKRR-RF} algorithm on three large-scale binary classification datasets, including covtype \footnote{https://archive.ics.uci.edu/ml/datasets/covertype} and SUSY \footnote{https://archive.ics.uci.edu/ml/datasets/susy} and HIGGS \footnote{https://archive.ics.uci.edu/ml/datasets/higgs}.
For the sake of comparison, we random sampled $N = 2.5\times10^5$ data points as the training data and $5 \times 10^4$ data points as the test data.
We use random Fourier features \cite{rahimi2007random} to approximate Gaussian kernel $K(\xx, \xx') = \exp(-{\|\xx - \xx'\|^2}/{2\sigma^2})$.
Random Fourier features are in the form $\psi(\xx, \omega) = \cos(\omega^T\xx + b)$, where $\omega$ is drawn from the corresponding Gaussian distribution and $b$ is drawn from uniform distribution $[0, 2\pi]$.
In the following experiments, we tune parameters $\sigma$ and $\lambda$ via $5$-folds cross-validation over grids of $2^{\{-5, -4.5, \cdots, 5\}}$ and $2^{\{-7, -6, \cdots, 3\}}$ respectively for each dataset, and report average errors over 10 repetitions.

The difficulties of those tasks are unknown ($r$ and $\gamma$ are unknown), such that for each dataset, we evaluate the classification errors in two cases:
\begin{itemize}
    \item[1)] For fixed $m=20$ and different $M \in [10, 1500]$, we compare the empirical performance of \texttt{DKRR-RF} and KRR-DC \cite{guo2017learning}.
    \item[2)] For fixed $M=500$ and different $m \in [1, 500]$, we compare the empirical performance of \texttt{DKRR-RF} and KRR-RF \cite{rudi2017generalization}.
\end{itemize}

To explore how the number of random features affects the classification accuracy, we fix the number of partitions as $20$ and vary the number of RFs.
As shown in left plots of Figures \ref{figure.exp1.covtype}, \ref{figure.exp1.SUSY}, \ref{figure.exp1.HIGGS}, when the number of RFs is small, classification errors of \texttt{DKRR-RF} decrease dramatically as the number of RFs increases.
However, when the number of features is more than certain thresholds, classification errors of \texttt{DKRR-RF} converge at some rate near the classification error of KRR-DC.
The certain threshold is near $300$ for covtype, $100$ for SUSY, and $600$ for HIGGS.
According to Theorem \ref{thm.unlabeled-explicit}, the number of random features is around $M \gtrsim N^\frac{1}{2r+\gamma}$, thus smaller thresholds represent smaller $1/ (2r+\gamma)$ and lead to higher computational efficiency.

\begin{figure}[t]
    \centering
    \subfigure{
        \centering
        \includegraphics[width=0.48\columnwidth]{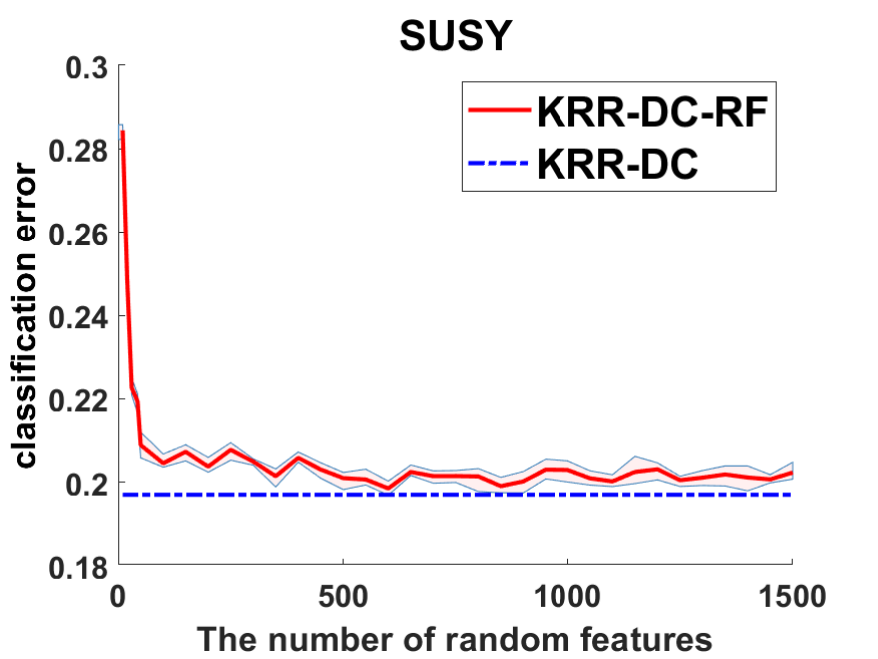}}
    \subfigure{
        \centering
        \includegraphics[width=0.48\columnwidth]{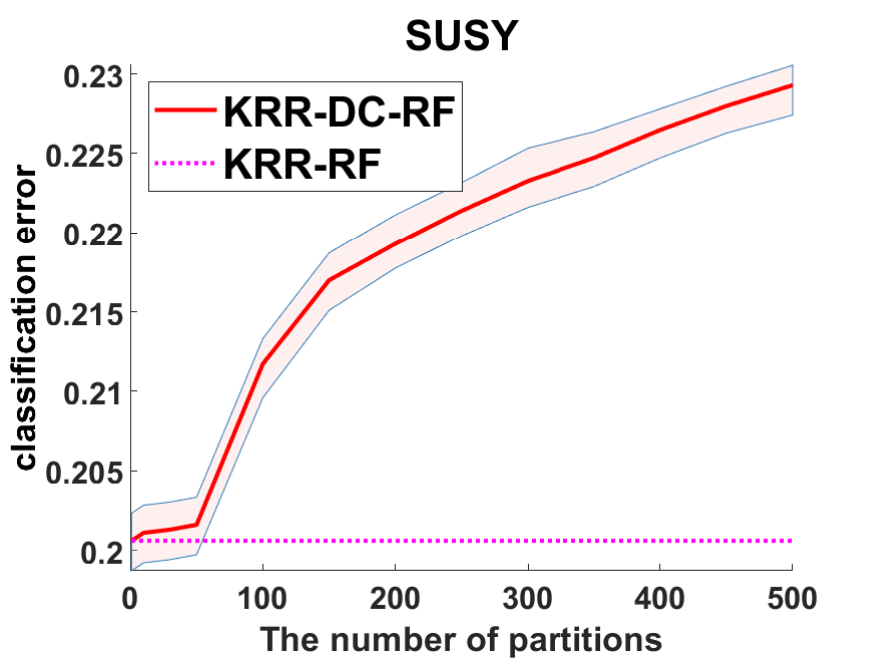}}
    \caption{Classification error of SUSY dataset in terms of different number of random features (left) and different number of partitions (right).}
    \label{figure.exp1.SUSY} 
\end{figure}

\begin{figure}[t]
    \centering
    \subfigure{
        \centering
        \includegraphics[width=0.48\columnwidth]{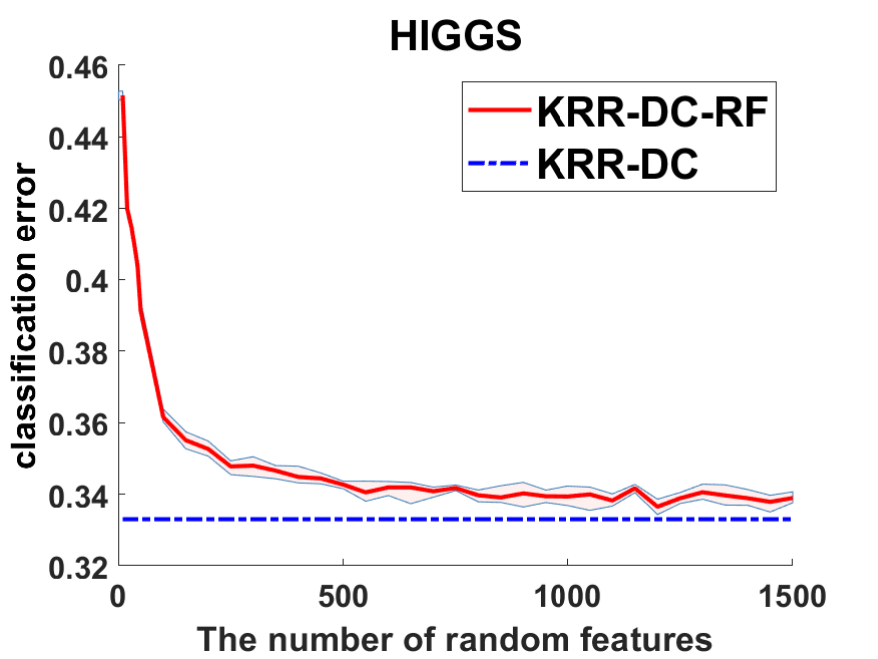}}
    \subfigure{
        \centering
        \includegraphics[width=0.48\columnwidth]{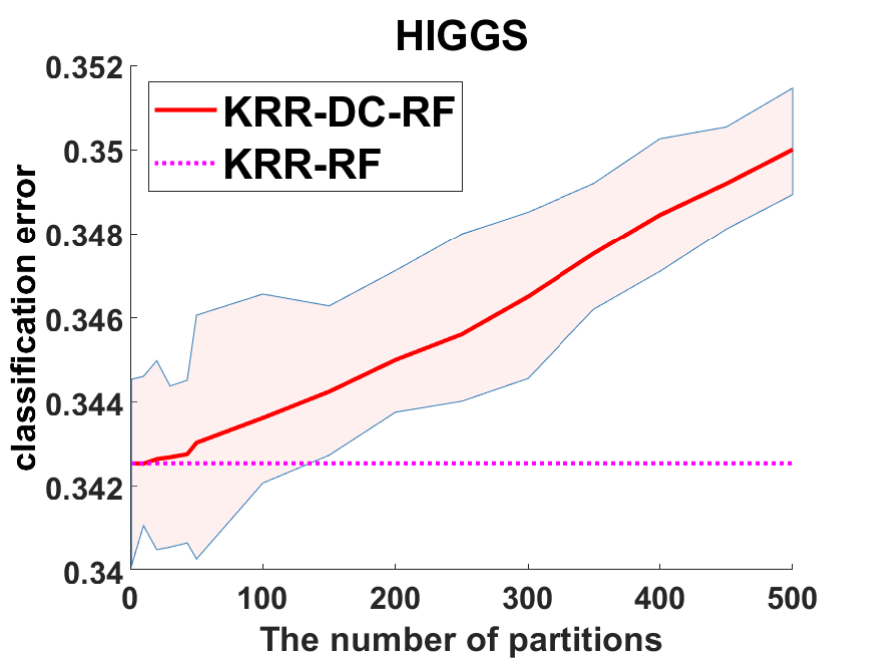}}
    \caption{Classification error of HIGGS dataset in terms of different number of random features (left) and different number of partitions (right).}
    \label{figure.exp1.HIGGS} 
\end{figure}

To study the influence of partitions on the accuracy, we fix the number of RFs as $M=500$ and increase the number of partitions $m$.
As demonstrated in the right plots of Figures \ref{figure.exp1.covtype}, \ref{figure.exp1.SUSY}, \ref{figure.exp1.HIGGS}, when the numbers of partitions are less than certain thresholds, \texttt{DKRR-RF} provides preferable classification accuracy that is close to the accuracy of KRR-RF. After that, errors increase quickly when the number of partitions increases.
The threshold is near $m=50$ for covtype, $m=60$ for SUSY and  $m=40$ for HIGGS.
According to Theorem \ref{thm.unlabeled-explicit}, the number of partitions is near $\mathcal{O}(NN^\frac{\gamma-1}{2r+\gamma})$, thus larger the number of partitions $m$ leads to larger $\frac{\gamma-1}{2r+\gamma}$ and less computational costs.

Indeed, when $M$ and $m$ are settled as the corresponding thresholds for each data, \texttt{DKRR-RF} achieves the optimal tradeoff of empirical performance and computational efficiency, for example $(M=300, m=50)$ for covtype, $(M=100, m=60)$ for SUSY, and $(M=600, m=40)$ for HIGGS.
We think that the optimal learning rates can still be achieved when the number of partitions is smaller than corresponding thresholds and the number of random features is larger than corresponding thresholds.
Nevertheless, more partitions or fewer random features break the optimal statistical properties of \texttt{DKRR-RF}, and thus the performance drops very fast.
Comparing the thresholds for $M$ and $m$ on those tasks, we find the following relationships on trainability: SUSY $>$ covtype $>$ HIGGS, which means the dataset SUSY is easier to obtain a good accuracy-efficiency tradeoff for DKRR-RF.

%% file: 4-conclusion.tex
\section{Conclusion}
This paper explores the generalization performance of kernel ridge regression with two commonly used efficient large-scale techniques: divide-and-conquer and random features.
We first present a general result with the optimal learning rates under standard assumptions. We then refine the theoretical results with more partitions and applicability in the non-attainable case.
Further, we reduce the number of random features by generating features in a data-dependent manner.
Finally, we present the theoretical results that substantially relax the constraint on the number of partitions with extra unlabeled data, which apply to both the attainable case and non-attainable case.
The proposed optimal theoretical guarantees are state-of-the-art in the theoretical analysis for KRR approaches.
With extensive experiments on both simulated and real-world data, we validate our theoretical findings with experimental results.

This paper can be extended in several ways:
(a) the combination with gradient algorithms such as multi-pass SGD \cite{lin2018optimaldcsgd,lin2020optimal} and preconditioned conjugate gradient \cite{avron2017faster} to further reduce the time complexity.
(b) using asynchronous distributed methods or a few of communications \cite{lin2020distributed,liu2021effective} instead of one-shot approach to alleviate the saturation phenomenon when $r \geq 1$.

%% file: 5-acknowledge.tex
\section*{acknowledgment} 
This work was supported in part by the Excellent Talents Program of Institute of Information Engineering, CAS, the Special Research Assistant Project of CAS, the Beijing Outstanding Young Scientist Program (No. BJJWZYJH012019100020098), and National Natural Science Foundation of China (No. 62076234, No. 62106257).

%% file: 3-proofs.tex
\newpage
\onecolumn
\appendix
\section{Proofs}

We denote $\nor{\cdot}$ the operatorial norm, specifically the norm $\|\cdot\|$ to represent the $\Ltwo$ norm $\|\cdot\|_\rho$ in the estimate of error terms.  
Moreover, we denote with $Q_\la$ the operator $Q + \la I$, where $Q$ is a bounded self-adjoint linear operator, $\la \in \R$ and $I$ the identity operator, so for example $\tCnl := (\tCn + \la I)$, $C_{M\la} := (C_M + \la I)$, $L_{M, \lambda} := (L_M + \lambda I)$ and $L_\lambda := (L + \lambda I)$, where operators $\tCn, C_M, L_M, L$ are defined in Definition \ref{def.ops-kernel} and Definition \ref{def.ops-rf}.
The estimates of error bounds are based on local estimators \eqref{f.local-y} and \eqref{f.local-frho} rather than global ones, such that they are associated with the number of local samples $n$, where $n = N/m$.

\subsection{Definitions of Linear Operators}Since KRR has closed-form solutions, we represent the intermediate estimators $\widehat{f}_{D,\lambda}^M, \widetilde{f}_{D,\lambda}^M, {f}_{\lambda}^M, {f}_{\lambda}$ in error decomposition by the redirection operators and their adjoint operators.
	In this part, we first provide useful linear operators associated with kernel $K$ (Definition \ref{def.ops-kernel}) and with random features $\phi_M$ (Definition \ref{def.ops-rf}), respectively.
	To bound the \textit{excess risk} $\mathcal{E}(\widehat{f}_{D,\lambda}^M) - \mathcal{E}(f_\rho)$, we present the closed-form solutions of the estimators used in Lemma \ref{lem.full-decomposition} based on those operators, which can be estimated by the difference between integral operator $L$ and random features based covariance operator $\widehat{C}_M$.
	
	To clearly state the relationships among estimators in Lemma \ref{lem.full-decomposition}, we introduce linear operators (both expected and empirical) associated with the RKHS $\mathcal{H}$ induced by the kernel $K$ and the feature space $\mathbb{R}^M$ induced by the random features $\phi_M$.
	\begin{definition}[Operators with kernel $K$]\label{def.ops-kernel}
		For any $g \in \Ltwo$, $\beta \in \mathcal{H}$ and $\alpha \in \CC^n$, we have
		\begin{itemize}
			\item $\SH: \HH \to \Ltwo, \quad (\SH \beta)(\cdot) = \langle\beta,\phi(\cdot)\rangle$.
			\item $\HSH: \HH \to \CC^n, \quad \HSH \beta = \frac{1}{\sqrt{n}} \big(\langle\beta,\phi(\xx_i)\rangle\big)_{i=1}^n \in \CC^n$.
			\item $\SH^*: \Ltwo \to \HH,  \quad \SH^*g = \int_X \phi(\xx) g(\xx)d\rhox(\xx)$.
			\item $\HSH^*: \CC^n \to \HH, \quad \HSH^* \alpha = \frac{1}{\sqrt{n}}\sum_{i=1}^n \phi(\xx_i) \alpha_i \in \HH$.
			\item $L: \Ltwo \to \Ltwo, \, L=\SH\SH^*, \, \text{such that} ~~ (L g)(\cdot) = \int_\X K(\cdot,\xx)g(\xx)d\rhox(\xx)$.
			\item $\mathbf{K}: \CC^n \to \CC^n, \quad \mathbf{K}=\HSH\HSH^*, \quad \text{such that} ~~ \mathbf{K}=\frac{1}{n}\big(K(\xx_i, \xx_j)\big)_{i, j=1}^n$.
			\item $C: \HH \to \HH, \quad C=\SH^*\SH, \quad \text{such that} ~~ C \beta = \int_X \langle \beta, \phi(\xx) \rangle \phi(\xx) d\rhox(\xx)$.
			\item $\widehat{C}: \HH \to \HH, \quad \widehat{C}=\HSH^*\HSH, \quad \text{such that} ~~ \widehat{C} \beta = \frac{1}{n}\sum_{i=1}^n \langle \beta, \phi(\xx_i) \rangle \phi(\xx_i)$.
		\end{itemize}
	\end{definition}
	Here, we denote $\SH$ the inclusion operator and $\HSH$ the sampling operator, while $\SH^*, \HSH^*$ are their adjoint operators.
	Note that $C:\mathcal{H} \to \mathcal{H}$ is the covariance operator given by $\SH^*\SH$, and the integral operator $L:\Ltwo \to \Ltwo$ given by $\SH\SH^*$.
	The kernel matrix $\mathbf{K}$ and the covariance matrix $\widehat{C}$ are the empirical counterparts of the integral operator $L$ and the covariance operator $C$, respectively.
	Using Singular Value Decomposition shows that $L$ and $C$ have the same eigenvalues, and the corresponding eigenvectors are closely related \cite{rosasco2010learning}.
	A similar relationship holds for the kernel matrix $\mathbf{K}$ and the covariance matrix $\widehat{C}$.
	Those kernels-related operators are widely used in the proof of optimal learning theory for standard KRR.
	Using Assumption \ref{asm.rf}, the integral operator $L$ and the covariance operator $C$ are positive trace class operators (and hence compact) and bounded by $\|L\|=\|C\|\leq \kappa^2.$
	For any function $f \in \mathcal{H}$, the estimator $f \in \Ltwo$ is obtained by $f(\cdot)=\langle f, \phi(\cdot)\rangle_\mathcal{H}=\SH f$.
	Thus, the RKHS norm can be related to the $\Ltwo$-norm by $C^{1/2}$ \cite{bauer2007regularization}:
	\begin{align*}
		\|f\|_\rho=\|\SH f\|_\rho= \|C^{1/2}f\|_\mathcal{H}, \quad \forall f \in \mathcal{H}.
	\end{align*}

	\begin{definition}[Operators with random features]\label{def.ops-rf}
		For any $g \in \Ltwo$, $\beta \in \CC^M$, $\alpha \in \CC^n$ and $K_M(\xx, \xx')=\langle\phi_M(\xx), \phi_M(\xx')\rangle$, we have
		\begin{itemize}[leftmargin=*]
			\item $S_M: \CC^M \to \Ltwo, \quad (\SM \beta)(\cdot) = \langle\phi_M(\cdot), \beta\rangle$.
			\item $\tSn: \CC^M \to \CC^n, \quad \tSn \beta = \frac{1}{\sqrt{n}} \big(\langle \beta, \phi_M(\xx_i) \rangle\big)_{i=1}^n \in \CC^n$.
			\item $\tS^*: \Ltwo \to \CC^M,  \quad \tS^*g = \int_X \phi_M(\xx) g(\xx)d\rhox(\xx)$.
			\item $\tSn^*: \CC^n \to \CC^M, \quad \tSn^* \alpha = \frac{1}{\sqrt{n}}\sum_{i=1}^n \phi_M(\xx_i) \alpha_i \in \CC^M$.
			\item $L_M: \Ltwo \to \Ltwo, \, L_M=\SM\SM^*, \, \text{such that} ~ (L_M g)(\cdot) = \int_\X K_M(\cdot, \xx) g(\xx)d\rhox(\xx)$.
			\item $\mathbf{K}_M: \CC^n \to \CC^n, \, \mathbf{K}_M=\tSn\tSn^*, \, \text{such that} ~ \mathbf{K}_M=\frac{1}{n}\big( K_M(\xx_i, \xx_j)\big)_{i, j=1}^n$.
			\item $C_M: \CC^M \to \CC^M, \, C_M=\SM^*\SM, \, \text{such that} ~ C_M \beta = \int_X \langle \beta, \phi_M(\xx) \rangle \phi_M(\xx) d\rhox(\xx)$.
			\item $\widehat{C}_M: \CC^M \to \CC^M, \, \widehat{C}_M=\tSn^*\tSn, \, \text{such that} ~ \widehat{C}_M\beta = \frac{1}{n}\sum_{i=1}^n \langle \beta, \phi_M(\xx_i) \rangle \phi_M(\xx_i)$.
		\end{itemize}
	\end{definition}
	
	Similarly, we also define the redirection operators $\SM,\tSn$ and their adjoint operators $\SM^*,\tSn^*$.
	Using the random features $\phi_M$, random features based operators $\SM,\tSn,\SM^*$ and $\tSn^*$ are close to the kernel based operators $\SH,\HSH,\SH^*,\HSH^*$.
	$L_M$ and $C_M$ are the integral operator and the covariance operator defined by random features on $\Ltwo$ and $\mathbb{R}^M$, respectively.
	The kernel matrix $\mathbf{K}_M$ is give by the kernel $K_M(\xx, \xx')=\langle\phi_M(\xx), \phi_M(\xx')\rangle$ associated with random features $\phi_M: \mathcal{X} \to \mathbb{R}^M$.
	Random features use Monte Carlo sampling approximates the kernel with $M$ features $K(\xx, \xx') \approx \langle \phi_M(\xx), \phi_M(\xx')\rangle$, thus the operators $L, \mathbf{K}, C, \widehat{C}$ are the expectation counterparts of $L_M, \mathbf{K}_M, C_M, \widehat{C}_M$ in terms of kernel probability density $\pi$.
	
	In Figure \ref{fig.relationships_operators}, we discuss the relationships among operators given in Definition \ref{def.ops-kernel} and Definition \ref{def.ops-rf}, where $L, L_M$ are self-adjoint integral operators on $\Ltwo$ and $C, C_M$ are self-adjoint integral operators on RKHS $\HH$ and $\mathbb{R}^M$, respectively. 
	Operators $\widehat{C}, \widehat{C}_M, \mathbf{K}, \mathbf{K}_M$ are the corresponding empirical counterparts of $C, C_M, L, L_M$.
	To estimate the error terms in Lemma \ref{lem.full-decomposition}, we utilize those operators to represent the estimators $\widehat{f}_{D,\lambda}^M,\widetilde{f}_{D,\lambda}^M, f_{\lambda}^M, f_\lambda, f_\rho$ with closed-form solutions.
	As shown in Figure \ref{fig.relationships_operators}, we measure the \textit{excess risk} of \texttt{DKRR-RF} by estimating the difference between the covariance matrix $\widehat{C}_M$ and the integral operator $L$ following the approximation chain $\widehat{C}_M \to C_M \to L_M \to L$ .
	\begin{figure}
		\centering
		\includegraphics[width=0.8\linewidth]{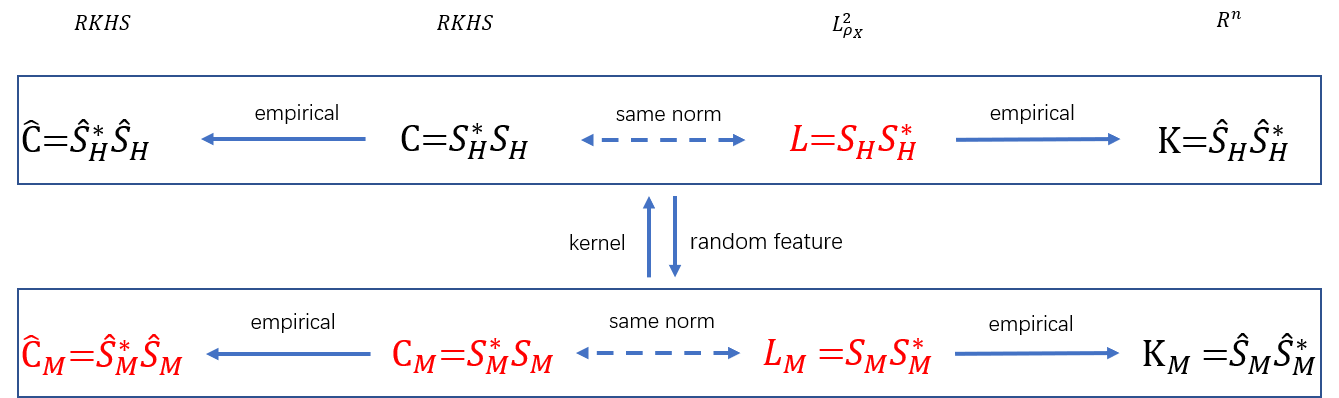}
		\caption{Relationships between linear operators defined by the kernel $K$ and random features $\phi_M$. Here, red marked operators indicate key approximation steps in the proof of the \textit{excess risk}: $\widehat{C}_M \to C_M \to L_M \to L$.}
		\label{fig.relationships_operators}
	\end{figure}
	\begin{remark}
		\label{rem.app1}
		Under Assumption \ref{asm.rf}, the integral operator $L$ is trace class \cite{caponnetto2007optimal} and $\tL, \tC, \tS, \tCn, \tSn$ are finite dimensional. Moreover we have that $\tL = \tS \tS^*$, $\tC = \tS^*\tS$ and $\tCn = \widehat{S}_M^*\tSn$. Finally $L, \tL, \tC, \tCn$ are self-adjoint and positive operators, with spectrum is $[0, \kappa^2]$.
	\end{remark}
	
	To represent the noise-free estimator \eqref{estimator.emp-noise-free}, we define a sampling operator $\bar{S}_M^*: \Ltwo \to \mathbb{R}^M$:
	$$\bar{S}_M^*g = \frac{1}{n}\sum_{i=1}^n \phi_M(\xx_i)g(\xx_i).$$

\subsection{Error decomposition}	
	To decompose the \textit{excess risk} of \texttt{DKRR-RF} $\mathcal{E}(\widehat{f}_{D,\lambda}^M) - \mathcal{E}(f_\rho)$ clearly, we provide some intermediate estimators $\widetilde{f}_{D_j,\lambda}^M(\xx) = \langle\widetilde{w},\phi_M(\xx)\rangle$, ${f}_{\lambda}^M(\xx)=\langle u,\phi_M(\xx)\rangle$, and ${f}_{\lambda}(\xx)=\langle v,\phi(\xx) \rangle$, where
	\begin{align}
		 &\widetilde{w} = \argmin_{w \in \CC^M} \left\{\frac{1}{n}\sum_{i=1}^{n} \big(\langle w, \phi_M(\xx_i)\rangle- f_\rho(\xx_i)\big)^2 + \lambda \|w\|^2\right\}, \label{estimator.emp-noise-free}   \\
		 & u = \argmin_{u \in \CC^M} \left\{\int_\mathcal{X} \big(\langle u, \phi_M(\xx)\rangle- f_\rho(\xx)\big)^2 d \rho_X(\xx)+ \lambda \|u\|^2\right\}, \label{estimator.expected-rf}                  \\
		 & v  = \argmin_{v \in \mathcal{H}} \left\{\int_\mathcal{X} \big(\langle v, \phi(\xx)\rangle_\mathcal{H} - f_\rho(\xx)\big)^2 d \rho_X(\xx)+ \lambda \|v\|_{\mathcal{H}}^2\right\}. \label{estimator.expected}
	\end{align}
	The local estimator $\widetilde{f}_{D_j,\lambda}^M$ in \eqref{estimator.emp-noise-free} is the noise-free version of $\widehat{f}_{D,\lambda}^M$, making up a global noise-free estimator $\widetilde{f}_{D,\lambda}^M = \frac{1}{m}\sum_{j=1}^m\widetilde{f}_{D_j,\lambda}^M$.
	The estimator ${f}_{\lambda}^M$ in \eqref{estimator.expected-rf} is a data-free version (expected version on $\rho$) of $\widehat{f}_{D,\lambda}^M$, which is still an approximation estimator by random features mapping $\phi_M:\mathcal{X} \to \CC^M$.
	The last one ${f}_{\lambda}$ in \eqref{estimator.expected} is the expected version of primal KRR with implicit feature mappings $\phi:\mathcal{X} \to \mathcal{H}$ associated with the kernel $K$ by $K(\xx, \xx') = \langle\phi(\xx), \phi(\xx')\rangle$.
	Using these estimators, we provide the following decomposition in the equality form of the \textit{excess risk} to further analyses the components of errors.

	\begin{lemma}
		\label{lem.estimator-definitions}
		Using operators defined in Definition \ref{def.ops-kernel} and Definition \ref{def.ops-rf}, the intermediate estimators $\widehat{f}_{D_j,\lambda}^M$ \eqref{estimator.emp-dc-rf}, $\widetilde{f}_{D_j,\lambda}^M$ \eqref{estimator.emp-noise-free}, $f_\lambda^M$ \eqref{estimator.expected-rf} and $f_\lambda$ \eqref{estimator.expected} admit the following closed-form solutions:
		\begin{align}
			\widehat{f}_{D_j,\lambda}^M   & = S_M (\widehat{C}_M + \lambda I)^{-1} \widehat{S}_M^* \yn,       \label{f.local-y} \\
			\widetilde{f}_{D_j,\lambda}^M & = S_M (\widehat{C}_M + \lambda I)^{-1} \bar{S}_M^* f_\rho. \label{f.local-frho}     \\
			f_{\lambda}^M                 & = L_M (L_M + \lambda I)^{-1} f_\rho, \label{f.expected-rf}                          \\
			f_{\lambda}                   & = L (L + \lambda I)^{-1} f_\rho.          \label{f.expected-erm}
		\end{align}
		Here, $\widehat{y}_n = \frac{1}{\sqrt{n}}[y_1, \cdots, y_n]^T$ represents the normalized labels of empirical samples.
	\end{lemma}

	\begin{proof}
		The objective of the local estimator is given in \eqref{eq.local-rf}.
		Using the representation theorem $\widehat{f}_{D_j,\lambda}^M(\xx) = \langle \widehat{\ww}_j, \phi_M(\xx) \rangle$, we let the derivative of the objective be zero and get the solution \eqref{estimator.emp-dc-rf} for the local estimator, which is
		\begin{align*}
			\widehat{f}_{D_j,\lambda}^M(\xx) = \langle \widehat{\ww}_j, \phi_M(\xx) \rangle, \qquad \widehat{\ww}_j=\big[\Phi_{M}^\top \Phi_{M} + \lambda I\big]^{-1}\Phi_{M}^\top\widehat{y},
		\end{align*}
	
		According to the definitions of operators in Definition \ref{def.ops-rf}, $\widehat{\ww}_j = (\tSn^*\tSn + \lambda I)^{-1} \tSn^* \widehat{y}_n$.
		Using the definitions of $S_M$ and $\widehat{C}_M$, we obtain 
		\begin{align*}
			\widehat{f}_{D_j,\lambda}^M   = S_M (\widehat{C}_M + \lambda I)^{-1} \widehat{S}_M^* \yn.
		\end{align*}
	
		According to the objective of $\widetilde{f}_{D_j,\lambda}^M$ \eqref{estimator.emp-noise-free}, we replace the noisy labels $\{y_1, y_2, \cdots, y_n\}$ with noise-free ones $\{\frho(\xx_1), \frho(\xx_2), \cdots, \frho(\xx_n)\}$.
		Meanwhile, using $\bar{S}_M^* f_\rho = \frac{1}{n} \sum_{i=1}^n \phi_M(\xx_i) \frho(\xx_i)$ instead of $\widehat{S}_M^* \yn = \frac{1}{n} \sum_{i=1}^n \phi_M(\xx_i) y_i$, it holds 
		\begin{align*}
			\widetilde{f}_{D_j,\lambda}^M = S_M (\widehat{C}_M + \lambda I)^{-1} \bar{S}_M^* f_\rho.
		\end{align*}
	
		Let the derivative of \eqref{estimator.expected-rf} be zero, and we can get $f^M_\lambda (\xx) = \langle u, \phi_M(\xx)\rangle$ with
		\begin{align*}
			u &= \left(\int_X \phi_M(\xx) \phi^\top_M(\xx) d \rho_X(\xx) + \lambda I \right)^{-1} \int_X \frho(\xx) \phi_M(\xx) d \rho_X(\xx) \\
			&= \left(C_M + \lambda I \right)^{-1} \tS^* \frho.
		\end{align*}
		Thus, it holds
		\begin{align*}
			f_{\lambda}^M &= S_M\left(C_M + \lambda I \right)^{-1} \tS^* \frho = S_M (S_M^* S_M + \lambda I)^{-1} S_M^* f_\rho \\
			&= S_M (S_M^* S_M + \lambda I)^{-1} S_M^* ( S_M S_M^* + \lambda I) (S_M S_M^* + \lambda I)^{-1} f_\rho \\
			&= S_M (S_M^* S_M + \lambda I)^{-1} (S_M^* S_M + \lambda I) S_M^* (S_M S_M^* + \lambda I)^{-1} f_\rho \\
			&= S_M S_M^* (S_M S_M^* + \lambda I)^{-1} f_\rho = L_M (L_M + \lambda I)^{-1} f_\rho.
		\end{align*}
	
		Similarly, using operators related the kernel $K$ in Definition \ref{def.ops-kernel}, we let the derivative of \eqref{estimator.expected} be zero, and obtain $f_\lambda (\xx) = \langle v, \phi_M(\xx)\rangle$ with
		\begin{align*}
			v &= \left(\int_X \phi(\xx) \phi^\top(\xx) d \rho_X(\xx) + \lambda I \right)^{-1} \int_X \frho(\xx) \phi(\xx) d \rho_X(\xx) \\
			&= \left(C + \lambda I \right)^{-1} \SH^* \frho.
		\end{align*}
		Thus, it holds
		\begin{align*}
			f_{\lambda} &= \SH\left(C + \lambda I \right)^{-1} \SH^* \frho = \SH (\SH^* \SH + \lambda I)^{-1} \SH^* f_\rho \\
			&= \SH (\SH^* \SH + \lambda I)^{-1} \SH^* (\SH \SH^* + \lambda I) (\SH \SH^* + \lambda I)^{-1} f_\rho \\
			&= \SH (\SH^* \SH + \lambda I)^{-1} (\SH^* \SH + \lambda I) \SH^* (\SH \SH^* + \lambda I)^{-1} f_\rho \\
			&= \SH \SH^* (\SH \SH^* + \lambda I)^{-1} f_\rho 
			= L (L + \lambda I)^{-1} f_\rho.
		\end{align*}
	\end{proof}
	
	We denote $\psi_\omega := \psi(\cdot, \omega)$ for any $\omega \in \Omega$. According to Assumption \ref{asm.rf} and the fact that $\rho$ is a finite measure, we have that $\psi_{\omega} \in \Ltwo$ almost surely.
	$\widehat{f}_{D_j,\lambda}^M$ is the linear combination of $\psi_{\omega_1}, \cdots, \psi_{\omega_M}$, such that $\widehat{f}_{D_j,\lambda}^M \in \Ltwo$ almost surely.
	Since $\widehat{f}_{D_j,\lambda}^M, f_\rho \in \Ltwo, \widehat{y}_n \in \mathbb{R}^n$ and using definitions of the operators, it holds that $\widetilde{f}_{D_j,\lambda}^M, f_{\lambda}^M, f_{\lambda}  \in \Ltwo$ and it is natural to estimate the differences of them in the $\Ltwo$-norm.
	
	\begin{remark}
		In some KRR theory literature, an estimator and its weight are defined as the same symbol. The excess risk measures the difference of the RKHS weight $f \in \mathcal{H}$ and the estimator $f_\rho \in \Ltwo$ (may not belong to the hypothesis space induced by the kernel $K$), which confuses the error decomposition of the \textit{excess risk}.
		In this paper, to cover the situation that $f_\rho$ is out of the hypothesis space, we measure the difference between $\|f - f_\rho\|^2_\rho, ~ \forall f \in \Ltwo$ rather than $\|f - \argmin_{f \in \mathcal{H}}\mathcal{E}(f)\|^2_\mathcal{H}, ~ \forall f \in \mathcal{H}$.
		Meanwhile, for sake of clarity, we define the estimators $\{\widehat{f}_{D,\lambda}^M, \widetilde{f}_{D,\lambda}^M, {f}_{\lambda}^M, f_\lambda \}\in \Ltwo$ and estimate the \textit{excess risk} in the $\Ltwo$-norm.
	\end{remark}

\begin{proposition}
	\label{pro.estimator_equivalence}
	For any $j \in [m]$, we use $(\mathbf{X}_j, \mathbf{y}_j)$ to denote the local samples and their labels on $D_j$, where $\overline{\mathbf{X}} =\{\mathbf{X}_1, \cdots, \mathbf{X}_m\}$ and $\bar{\mathbf{y}} = \{\mathbf{y}_1, \cdots, \mathbf{y}_m\}$ represent the samples and labels on all partitions. 
	For the estimators $\widehat{f}_{D_j,\lambda}^M$, and $\widetilde{f}_{D_j,\lambda}^M$, there holds 
	\begin{align}
		\mathbb{E}_{\mathbf{y}_j} [\widehat{f}_{D_j,\lambda}^M]   & = \widetilde{f}_{D_j,\lambda}^M, \label{eq.proof.expectation1}
	\end{align}
	Here, $\mathbb{E}_{\mathbf{y}_j}$ denotes the conditional expectation with respect to $\mathbf{y}_j$ given $\mathbf{X}_j$ on the $j$-th partition $D_j$.
\end{proposition}
\begin{proof}
	Using the operator based solutions of the local estimator of \texttt{DKRR-RF} $\widehat{f}_{D_j,\lambda}^M$ \eqref{f.local-y}, and the local noise-free estimator $\widetilde{f}_{D_j,\lambda}^M$ \eqref{f.local-frho}, we have 
	\begin{align*}
		\widehat{f}_{D_j,\lambda}^M &= S_M (\widehat{C}_M + \lambda I)^{-1} \widehat{S}_M^* \yn = S_M (\widehat{C}_M + \lambda I)^{-1} \frac{1}{n} \sum_{i=1}^n \phi_M(\xx_i) y_i,\\
		\widetilde{f}_{D_j,\lambda}^M &=S_M (\widehat{C}_M + \lambda I)^{-1} \bar{S}_M^* f_\rho = S_M (\widehat{C}_M + \lambda I)^{-1} \frac{1}{n} \sum_{i=1}^n \phi_M(\xx_i) f_\rho(\xx_i).
	\end{align*}
	Taking the expectation over $\widehat{f}_{D_j,\lambda}^M$ in terms of $\mathbf{y}_j$, it holds
	\begin{align*}
		\mathbb{E}_{\mathbf{y}_j} \widehat{f}_{D_j,\lambda}^M
		&= S_M (\widehat{C}_M + \lambda I)^{-1} \frac{1}{n} \sum_{i=1}^n \phi_M(\xx_i) \mathbb{E}_{\mathbf{y}_j} y_i\\
		&= S_M (\widehat{C}_M + \lambda I)^{-1} \frac{1}{n} \sum_{i=1}^n \phi_M(\xx_i) \int_\mathcal{Y} y d\rho(y|\xx_i) \\
		&= S_M (\widehat{C}_M + \lambda I)^{-1} \frac{1}{n} \sum_{i=1}^n \phi_M(\xx_i) f_\rho(\xx_i) 
		= \widetilde{f}_{D_j,\lambda}^M,
	\end{align*}
	that proves the identity \eqref{eq.proof.expectation1}.
\end{proof}

Taking the expectation over the conditional distribution $\rho(y|\xx)$, we first prove the equivalence between the local estimators.
We then establish the equivalence relationship between $\widehat{f}_{D_j,\lambda}^M$ and $\widetilde{f}_{D_j,\lambda}^M$.
Next, we derive relationships between global estimators and local ones to prove the error decomposition in Lemma \ref{lem.full-decomposition}. 
It is easy to bridge connection between the \textit{excess risk} and the discrepancy of two estimators: $f \in \Ltwo$ and the target regression $f_\rho$ \cite{smale2007learning} that
\begin{align}
	\label{eq.excess-risk}
	\mathcal{E}(f) - \mathcal{E}(f_\rho) = \|f - f_\rho\|_\rho^2.
\end{align}
Here, $\|f\|_\rho = \|f\|_{L^2_{\rho_X}}=\big(\int_\mathcal{X}|f(\xx)|^2d\rho_X\big)^{1/2}$, $\rho_X(\cdot)$ is the induced marginal measure on the input space $\mathcal{X}$.

\begin{lemma}
	\label{lem.full-decomposition}
	For any $j \in [m]$, let $\widehat{f}_{D,\lambda}^M, \widetilde{f}_{D,\lambda}^M, {f}_{\lambda}^M$ and ${f}_{\lambda}$ be defined as the above, we have
	\begin{align}
		& \mathbb{E}~\mathcal{E}(\widehat{f}_{D,\lambda}^M) - \mathcal{E}(f_\rho) \\
		\leq ~ & \frac{1}{m}\mathbb{E} \|\widehat{f}_{D_j,\lambda}^M - \widetilde{f}_{D_j,\lambda}^M\|_\rho^2 \quad  (\text{Sample Variance}) \label{error.variance} \\
		+~     & \frac{3}{m}\mathbb{E} \|\widetilde{f}_{D_j,\lambda}^M - {f}_{\lambda}^M\|_\rho^2 \quad (\text{Distributed Error}) \label{error.distributed} \\
		+~     & 3 \, \mathbb{E} \|\widetilde{f}_{D_j,\lambda}^M - {f}_{\lambda}^M\|_\rho^2 \quad (\text{Empirical Error}) \label{error.empirical} \\
		+ ~    & 3 \, \mathbb{E}\| {f}_{\lambda}^M
		- {f}_{\lambda}\|_\rho^2 \quad (\text{Random Features Error}) \label{error.rf} \\
		+ ~    & 3 \, \| {f}_{\lambda}
		- ~f_\rho\|_\rho^2 \quad (\text{Approximation Error}) \label{error.approximation}.
	\end{align}
\end{lemma}

\begin{proof}
	Using the noise-free estimator $\widetilde{f}_{D,\lambda}^M$ as the intermedium, we have
	\begin{align*}
		\|\widehat{f}_{D,\lambda}^M - f_\rho\|_\rho^2 = \|\widehat{f}_{D,\lambda}^M - \widetilde{f}_{D,\lambda}^M\|_\rho^2 + \|\widetilde{f}_{D,\lambda}^M - f_\rho\|_\rho^2 + 2 \langle \widehat{f}_{D,\lambda}^M - \widetilde{f}_{D,\lambda}^M,  \widetilde{f}_{D,\lambda}^M - f_\rho\rangle.
	\end{align*}
	Taking the conditional expectation with respect to $\bar{\mathbf{y}}$ given $\overline{\mathbf{X}}$ on both sides, using \eqref{eq.proof.expectation1} in Proposition \ref{pro.estimator_equivalence} which indicates
	\begin{align*}
		\mathbb{E}_{\bar{\mathbf{y}}} ~ [\widehat{f}_{D,\lambda}^M - \widetilde{f}_{D,\lambda}^M] = \frac{1}{m} \sum_{j=1}^m \mathbb{E}_{\mathbf{y}_j} [\widehat{f}_{D_j,\lambda}^M - \widetilde{f}_{D_j,\lambda}^M] = 0,
	\end{align*}
	we thus have
	\begin{align}
		\label{eq.proof.decomposition1}
		\mathbb{E}_{\bar{\mathbf{y}}} \|\widehat{f}_{D,\lambda}^M - f_\rho\|_\rho^2 = \mathbb{E}_{\bar{\mathbf{y}}} \|\widehat{f}_{D,\lambda}^M - \widetilde{f}_{D,\lambda}^M\|_\rho^2
		+ \|\widetilde{f}_{D,\lambda}^M - f_\rho\|_\rho^2.
	\end{align}

	Using the fact $(a+b+c)^2 \leq 3 a^2 + 3 b^2 + 3 c^2, ~\forall a, b, c >0$, we have
	\begin{align}
		\label{eq.proof.decomposition2}
		\|\widetilde{f}_{D,\lambda}^M - f_\rho\|_\rho^2 
		\leq 3\, \|\widetilde{f}_{D,\lambda}^M - f_\lambda^M\|_\rho^2 + 3\, \|f_\lambda^M - f_\lambda\|_\rho^2 + 3\, \|f_\lambda - f_\rho\|_\rho^2.
	\end{align}

	Following the proof of Proposition 5 in \cite{chang2017distributed}, we establish the relationship between global and local empirical error
	\begin{equation}
		\begin{aligned}
			\label{eq.proof.decomposition.empirical_error}
			\mathbb{E}\|\widetilde{f}_{D,\lambda}^M - f_\lambda^M\|_\rho^2 
			&\leq \sum_{j=1}^m \frac{|D_j|^2}{|D|^2} \mathbb{E} \left\| \widetilde{f}_{D_j,\lambda}^M - f_\lambda^M\right\|_\rho^2 + \sum_{j=1}^m \frac{|D_j|}{|D|} \left\|\mathbb{E} [\widetilde{f}_{D_j,\lambda}^M] - f_\lambda^M\right\|_\rho^2  \\
			&\leq \sum_{j=1}^m \left(\frac{|D_j|^2}{|D|^2} + \frac{|D_j|}{|D|}\right)\mathbb{E} \left\| \widetilde{f}_{D_j,\lambda}^M - f_\lambda^M\right\|_\rho^2.
		\end{aligned}
	\end{equation}

	Substituting \eqref{eq.proof.decomposition1}, \eqref{eq.proof.decomposition2} and \eqref{eq.proof.decomposition.empirical_error} to \eqref{eq.excess-risk}, we get the desired result
	\begin{align*}
		\mathbb{E} \mathcal{E}(\widehat{f}_{D,\lambda}^M) - \mathcal{E}(f_\rho) 
		\leq \frac{1}{m}\mathbb{E} \|\widehat{f}_{D_j,\lambda}^M - \widetilde{f}_{D_j,\lambda}^M\|_\rho^2
		+ 3\left(\frac{1}{m} + 1\right) \mathbb{E} \|\widetilde{f}_{D_j,\lambda}^M - f_\lambda^M\|_\rho^2 \\
		+ 3\,\mathbb{E} \|f_\lambda^M - f_\lambda\|_\rho^2 
		+ 3\,\|f_\lambda - f_\rho\|_\rho^2.
	\end{align*}
\end{proof}

Sample variance \eqref{error.variance} is brought by noise on labels $y$, which is output-dependent.
Distributed error \eqref{error.distributed} reflects errors from distributed learning.
Empirical error \eqref{error.empirical} represents the gap between expected learning and empirical learning.
Note that empirical error focuses on noise-free data, and thus it can be reduced by additional unlabeled data, resulting in Theorem \ref{thm.unlabeled-explicit}.
Independent on the sample, random features error \eqref{error.rf} is caused by the discrepancy between the kernel approximated by random features and the kernel, while approximation error \eqref{error.approximation} reflects the bias of the algorithm. Data-dependent features can reduce random features error \eqref{error.rf} that motivates Theorem \ref{thm.data-dependent}.

The global sample variance is reduced to $1/m$ of the local one, illustrating that distributed learning can reduce the sample error than any local estimator.
But also, the empirical error is output independent and can be reduced by using unlabeled data.
Thus, with the same optimal error convergence rate, we improve the number of partitions $m$ by introducing more unlabeled examples in Theorem \ref{thm.unlabeled-explicit}.
Sample variance relies on both samples and labels, while random features error and approximation error are independent of the data,
so additional unlabeled data do not influence other errors.

\subsection{Estimate Error Terms}
To analysis the excess risk, in this part, we estimate four error terms $\|\widehat{f}_{D_j,\lambda}^M - \widetilde{f}_{D_j,\lambda}^M\|_\rho^2, \|\widetilde{f}_{D_j,\lambda}^M - {f}_{\lambda}^M\|_\rho^2, \| {f}_{\lambda}^M - {f}_{\lambda}\|_\rho^2$
and $\| {f}_{\lambda} - \frho\|_\rho^2$.
The estimate of sample variance $\|\widehat{f}_{D_j,\lambda}^M - \widetilde{f}_{D_j,\lambda}^M\|_\rho^2$ and empirical error $\|\widetilde{f}_{D_j,\lambda}^M - {f}_{\lambda}^M\|_\rho^2$ are related to the key quantity $\|C_{M, \la}^{1/2}\tCnl^{-1/2}\|$.
To relax the restriction on the number of partitions, we provide a sharper upper bound for the critical quantity as a constant based on Bernstein's inequality.
The estimate of random features error is also associated with a critical quantity $\|L_{\lambda}^{-1/2}(L-L_M)L_{\lambda}^{-(1-r)}\|$, where we estimate this term separately.
The sample variance is related to the number local labeled sample size $n$, while the key quantity $\|C_{M, \la}^{1/2}\tCnl^{-1/2}\|$ is related to the local total sample size $n^*$.
Those two parts lead to two constraints on the number of partitions.
Random features error is related to the dimension of random features and independent of sample size.

\subsubsection{Estimates for Sample Variance $\|\widehat{f}_{D_j,\lambda}^M - \widetilde{f}_{D_j,\lambda}^M\|$}
\begin{lemma}
	\label{lem.variance-error}
	For the sample variance \eqref{error.variance} in the error decomposition, the following holds
	\begin{align*}
		&\|\widehat{f}_{D_j,\lambda}^M - \widetilde{f}_{D_j,\lambda}^M\| \\
		\leq &\|C_{M, \la}^{1/2}\tCnl^{-1/2}\|^2 \left[\|C_{M, \la}^{-1/2}(\widehat{S}_M^* \yn - S_M^* f_\rho)\|
		+ \|C_{M, \la}^{-1/2}(S_M^*f_\rho - \bar{S}_M^* f_\rho)\|\right].
	\end{align*}
	For any $\delta \in (0, 1/3]$, when $n \geq 16 (\mathcal{N}_\infty(\lambda) + 1) \log(2/\delta) $, there exists with the probability at least $1 - 3\,\delta$
	\begin{align*}
		\|\widehat{f}_{D_j,\lambda}^M - \widetilde{f}_{D_j,\lambda}^M\| \leq 8 \left(\frac{\kappa(B+\sigma)}{n\sqrt{\lambda}} + \sqrt{\frac{\sigma^2\mathcal{N}_M(\lambda)}{n}}\right)\log\frac{2}{\delta}.
	\end{align*}
\end{lemma}
\begin{proof}
	Let $\widehat{f}_{D_j,\lambda}^M$ and $\widetilde{f}_{D_j,\lambda}^M$ be defined as (\ref{f.local-y}) and (\ref{f.local-frho}), we have
	\begin{equation}
		\begin{aligned}
			\label{eq.variance.eq1}
			\|\widehat{f}_{D_j,\lambda}^M - \widetilde{f}_{D_j,\lambda}^M \|
			=    & \|S_M\tCnl^{-1}(\widehat{S}_M^* \yn - \bar{S}_M^* f_\rho)\|                                                       \\
			=    & \|\big(S_M\tCnl^{-1} C_{M, \la}^{1/2}\big)\big(C_{M, \la}^{-1/2}(\widehat{S}_M^* \yn - \bar{S}_M^* f_\rho)\big)\| \\
			\leq & \|S_M\tCnl^{-1} C_{M, \la}^{1/2}\|\|C_{M, \la}^{-1/2}(\widehat{S}_M^* \yn - \bar{S}_M^* f_\rho)\|.
		\end{aligned}
	\end{equation}
	The last step is due to Cauchy--Schwarz inequality.
	Note that
	\begin{align*}
		\|S_M \tCnl^{-1/2}\| = \|S_M C_{M, \la}^{-1/2} C_{M, \la}^{1/2} \tCnl^{-1/2}\| \leq
		\|S_M C_{M, \la}^{-1/2}\|\|C_{M, \la}^{1/2} \tCnl^{-1/2}\|,
	\end{align*}
	where $\|S_M C_{M, \la}^{-1/2}\| \leq \|C_{M, \la}^{-1/2} S_M^* S_M C_{M, \la}^{-1/2}\|^{1/2} \leq 1$.
	Thus, we have $\|S_M\tCnl^{-1/2}\| \leq \|C_{M, \la}^{1/2}\tCnl^{-1/2}\|$, and it holds that
	\begin{align}
		\label{eq.variance.eq2}
		\|S_M\tCnl^{-1} C_{M, \la}^{1/2}\| \leq \|C_{M, \la}^{1/2}\tCnl^{-1/2}\|^2.
	\end{align}

	The term $\|C_{M, \la}^{-1/2}(\widehat{S}_M^* \yn - \bar{S}_M^* f_\rho)\|$ can be rewritten as
	\begin{equation}
		\begin{aligned}
			\label{eq.variance.eq3}
			\|C_{M, \la}^{-1/2}(\widehat{S}_M^* \yn - \bar{S}_M^* f_\rho)\|
			=    & \|C_{M, \la}^{-1/2}(\widehat{S}_M^* \yn - S_M^*f_\rho + S_M^*f_\rho - \bar{S}_M^* f_\rho)\|                        \\
			\leq & \|C_{M, \la}^{-1/2}(\widehat{S}_M^* \yn - S_M^*f_\rho)\| + \|C_{M, \la}^{-1/2}(S_M^*f_\rho - \bar{S}_M^* f_\rho)\|.
		\end{aligned}
	\end{equation}

	Combining \eqref{eq.variance.eq1}, \eqref{eq.variance.eq2} and \eqref{eq.variance.eq3}, one can prove
	\begin{equation}
		\begin{aligned}
			\label{eq.sample_variance.proof.eq3}
			&\|\widehat{f}_{D_j,\lambda}^M - \widetilde{f}_{D_j,\lambda}^M\| \leq \\
			&\|C_{M, \la}^{1/2}\tCnl^{-1/2}\|^2 \left[\|C_{M, \la}^{-1/2}(\widehat{S}_M^* \yn - S_M^* f_\rho)\|
			+ \|C_{M, \la}^{-1/2}(S_M^*f_\rho - \bar{S}_M^* f_\rho)\|\right].
		\end{aligned}
	\end{equation}

	From Lemma \ref{lem.difference_between_C_C_M}, we know that with high probability $\|C_{M, \la}^{1/2}\tCnl^{-1/2}\|^2 \leq 2$ if $n \geq 16 (\mathcal{N}_\infty(\lambda) + 1) \log(2/\delta)$.
	Substituting Lemma \ref{lem.sample_variance.error2}, Lemma \ref{lem.sample_variance.error3} and Lemma \ref{lem.difference_between_C_C_M} to \eqref{eq.sample_variance.proof.eq3}, if $n \geq 16 (\mathcal{N}_\infty(\lambda) + 1) \log(2/\delta)$, it holds with the probability at least $1 - 3 \delta$
	\begin{align*}
		&\|\widehat{f}_{D_j,\lambda}^M - \widetilde{f}_{D_j,\lambda}^M\| \\
		\leq &2\left[2\left(\frac{\kappa B}{n\sqrt{\lambda}} + \sqrt{\frac{\sigma^2\mathcal{N}_M(\lambda)}{n}}\right) \log \frac{2}{\delta} + 2\left(\frac{\kappa\sigma}{n\sqrt{\lambda}} + \sqrt{\frac{\sigma^2\mathcal{N}_M(\lambda)}{n}} \right)\log\frac{2}{\delta}\right]\\
		\leq &8 \left(\frac{\kappa(B+\sigma)}{n\sqrt{\lambda}} + \sqrt{\frac{\sigma^2\mathcal{N}_M(\lambda)}{n}}\right)\log\frac{2}{\delta}.
	\end{align*}
\end{proof}

\begin{lemma}[Lemma 6 in \cite{rudi2017generalization}]
	\label{lem.sample_variance.error2}
	For $\delta \in (0, 1]$, under assumptions \ref{asm.rf}, \ref{asm.moment}, the following holds with the probability at least $1 - \delta$
	\begin{align*}
		\|C_{M, \la}^{-1/2}(\widehat{S}_M^* \yn - S_M^* f_\rho)\| \leq 2\left(\frac{\kappa B}{n\sqrt{\lambda}} + \sqrt{\frac{\sigma^2\mathcal{N}_M(\lambda)}{n}}\right) \log \frac{2}{\delta}.
	\end{align*}
\end{lemma}

Using Bernstein's inequality (Proposition \ref{pro.bernstein_inequality}), we prove the following lemma.
\begin{lemma}
	\label{lem.sample_variance.error3}
	For $\delta \in (0, 1],$ under Assumptions \ref{asm.rf}, \ref{asm.moment} with the probability at least $1 - \delta$, we have
	\begin{align*}
		\|C_{M, \la}^{-1/2}(S_M^*f_\rho - \bar{S}_M^* f_\rho)\| \leq 2\left(\frac{\kappa\sigma}{n\sqrt{\lambda}} + \sqrt{\frac{\sigma^2\mathcal{N}_M(\lambda)}{n}} \right)\log\frac{2}{\delta}.
	\end{align*}
\end{lemma}
\begin{proof}
	Let $\xi_i = C_{M, \la}^{-1/2} \phi_M(\xx_i) f_\rho(\xx_i)$ on $\mathcal{X}$ in the Hilbert space $\mathcal{H}_M$. We see that
	\begin{align*}
		\frac{1}{n}\sum_{i=1}^n \xi_i & =\frac{1}{n}\sum_{i=1}^n C_{M, \la}^{-1/2} \phi_M(\xx_i) f_\rho(\xx_i) = C_{M, \la}^{-1/2} \bar{S}_M^* f_\rho, \\
		\mathbb{E} \xi                & = \int_X C_{M, \la}^{-1/2} \phi_M(\xx) f_\rho(\xx) d\rho_X(\xx) = C_{M, \la}^{-1/2}S_M^*f_\rho
	\end{align*}
	Thus, the error term to bound can be stated as
	\begin{align}
		\label{eq.sample_variance.error3.bernstein}
		\|C_{M, \la}^{-1/2}(S_M^*f_\rho - \bar{S}_M^* f_\rho)\| = \left\|\frac{1}{n}\sum_{i=1}^n \xi_i - \mathbb{E} \xi\right\|.
	\end{align}
	The rhs of the above identity can be bounded by Bernstein's inequality (Proposition \ref{pro.bernstein_inequality}), thus we need to estimate $\|\xi\|$ and $\mathbb{E}(\|\xi_i - \mathbb{E}(\xi_i) \|^2)$ first.

	Note that $\int_\mathcal{Y} y^2 d\rho(y|\xx) \leq \sigma^2$ under Assumption \ref{asm.moment} when setting $p=2$ implies that the regression function is bounded almost surely \cite{lin2020optimal}
	\begin{align*}
		|f_\rho(\xx)| \leq \sigma.
	\end{align*}
	With the inequality
	\begin{align*}
		\|C_{M, \la}^{-1/2} \phi_M(\xx_i)\|^2 \leq \frac{1}{\lambda} \sup_{\xx \in \mathcal{X}} \|\phi_M(\xx)\|^2 \leq \frac{\kappa^2}{\lambda},
	\end{align*}
	we thus have
	\begin{align}
		\label{eq.sample_variance.error3.xi}
		\|\xi_i\| \leq \|C_{M, \la}^{-1/2} \phi_M(\xx_i)\| \|f_\rho\| \leq \frac{\kappa \sigma}{\sqrt{\lambda}}.
	\end{align}

	Note that
	\begin{equation}		
		\begin{aligned}
			\label{eq.sample_variance.error3.xi_square}
			\mathbb{E} (\|\xi_i - \mathbb{E}(\xi_i) \|^2)
			& \leq 2\int_\mathcal{X} \|C_{M, \la}^{-1/2} \phi_M(\xx)\|^2 |f_\rho(\xx)|^2 d\rho_X(\xx)\\
			& \leq 2\sigma^2 \int_\mathcal{X} \|C_{M, \la}^{-1/2} \phi_M(\xx)\|^2 d\rho_X(\xx)
			\leq 2\sigma^2 \mathcal{N}_M(\lambda).
		\end{aligned}
	\end{equation}

	Substituting \eqref{eq.sample_variance.error3.xi}  and \eqref{eq.sample_variance.error3.xi_square} to \eqref{eq.sample_variance.error3.bernstein}, by Bernstein's inequality (Proposition \ref{pro.bernstein_inequality}), one can prove that with the probability at least $1 - \delta$ 
	\begin{align*}
		\|C_{M, \la}^{-1/2}(S_M^*f_\rho - \bar{S}_M^* f_\rho)\| 
		&\leq \frac{2\kappa\sigma\log(2/\delta)}{n\sqrt{\lambda}} + 2\sqrt{\frac{\sigma^2\mathcal{N}_M(\lambda)\log(2/\delta)}{n}}\\
		&\leq 2\left(\frac{\kappa\sigma}{n\sqrt{\lambda}} + \sqrt{\frac{\sigma^2\mathcal{N}_M(\lambda)}{n}} \right)\log\frac{2}{\delta}.
	\end{align*}
\end{proof}

\subsubsection{Estimates for Empirical Error $\|\widetilde{f}_{D_j,\lambda}^M - {f}_{\lambda}^M\|$}
\begin{lemma}
	\label{lem.empirical-error}
	For the empirical error \eqref{error.empirical} in error decomposition,
	the following holds
	\begin{align*}
		\|\widetilde{f}_{D_j,\lambda}^M - {f}_{\lambda}^M\|
		\leq \left[\|C_{M,\lambda}^{1/2} \widehat{C}_{M,\lambda}^{-1/2}\| + \|C_{M,\lambda}^{1/2} \widehat{C}_{M,\lambda}^{-1/2}\|^2 \right]\|{f}_{\lambda}^M - f_\rho\|.
	\end{align*}
	Under Assumptions \ref{asm.rf} and \ref{asm.compatibility}, for $\delta \in (0, 1/3]$ and $\la > 0$, when the number of local examples satisfies $n \geq 16 (\mathcal{N}_\infty(\lambda) + 1) \log(2/\delta)$
	and the dimension of random features satisfies
	\begin{align*}
		M &\geq 16~ \mathcal{N}_\infty (\lambda)  \log(2/\delta) \quad \forall r \in (0, 1/2),\\
		M &\geq 16\kappa^2(\mathcal{N}_\infty(\lambda) + 1)\log(2/\delta) \left[1 \vee \frac{\mathcal{N}(\la)}{(\mathcal{N}_\infty(\la) + 1)\la}\right]^{2r-1} \quad \forall r \in [1/2, 1],
	\end{align*}
	there exists with the probability at least $1 - 3\delta$
	\begin{align*}
		\|\widetilde{f}_{D_j,\lambda}^M - {f}_{\lambda}^M\| \leq 9 R \lambda^{r}.
	\end{align*}
\end{lemma}
\begin{proof}
	Using the definition of $f^M_\lambda$, we have
	\begin{align*}
		f^M_\lambda &= L_M L_{M, \lambda}^{-1} f_\rho = S_M S_M^* (S_M S_M^* + \lambda I)^{-1} f_\rho \\
		&= S_M (S_M^*S_M  + \lambda I)^{-1} S_M^* f_\rho
		= S_M C_{M, \lambda}^{-1} S_M^* f_\rho.
	\end{align*}

	Under definitions in (\ref{f.local-frho}) and (\ref{f.expected-rf}),
	using the above identity and $A^{-1} - B^{-1} = A^{-1}(B-A)B^{-1}$ for positive operators $A, B$,
	we have
	\begin{align*}
		\|\widetilde{f}_{D_j,\lambda}^M - {f}_{\lambda}^M\|
		 & =\|S_M\widehat{C}_{M,\lambda}^{-1} \bar{S}_M^* f_\rho - S_M C_{M, \lambda}^{-1} S_M^* f_\rho\|                                                                                      \\
		 & =\|S_M\widehat{C}_{M,\lambda}^{-1} (\bar{S}_M^* - S_M^*)f_\rho + S_M (\widehat{C}_{M,\lambda}^{-1} - C_{M, \lambda}^{-1} ) S_M^* f_\rho\|                                           \\
		 & =\|S_M\widehat{C}_{M,\lambda}^{-1} (\bar{S}_M^* - S_M^*)f_\rho + S_M \widehat{C}_{M,\lambda}^{-1}(C_M - \widehat{C}_{M})C_{M, \lambda}^{-1} S_M^* f_\rho\|                          \\
		 & =\|S_M\widehat{C}_{M,\lambda}^{-1} (\bar{S}_M^* - S_M^*)f_\rho + S_M \widehat{C}_{M,\lambda}^{-1}(S_M^* S_M - \bar{S}_M^* S_M)C_{M, \lambda}^{-1} S_M^* f_\rho\|                    \\
		 & =\|S_M\widehat{C}_{M,\lambda}^{-1} (\bar{S}_M^* - S_M^*)f_\rho + S_M \widehat{C}_{M,\lambda}^{-1}(S_M^* - \bar{S}_M^*) f^M_\lambda\|                                                \\
		 & =\|S_M\widehat{C}_{M,\lambda}^{-1} \bar{S}_M^* (f_\rho - f^M_\lambda) + S_M \widehat{C}_{M,\lambda}^{-1}S_M^*(f^M_\lambda - f_\rho)\|                                               \\
		 & =\|S_M C_{M,\lambda}^{-1/2} C_{M,\lambda}^{1/2} \widehat{C}_{M,\lambda}^{-1/2} \widehat{C}_{M,\lambda}^{-1/2} \bar{S}_M^* (f_\rho - f^M_\lambda)                                    \\
		 & + S_M  C_{M,\lambda}^{-1/2} C_{M,\lambda}^{1/2} \widehat{C}_{M,\lambda}^{-1/2} \widehat{C}_{M,\lambda}^{-1/2}C_{M,\lambda}^{1/2} C_{M,\lambda}^{-1/2}S_M^*(f^M_\lambda - f_\rho)\|.
	\end{align*}
	To obtain the key term $\|\widehat{C}_{M,\lambda}^{-1/2} \widehat{C}_{M,\lambda}^{-1/2}\|$, we introduce additional terms in the last step of the above identity.
	Note that, the following inequalities hold $\|S_M C_{M,\lambda}^{-1/2} \|= \|C_{M,\lambda}^{-1/2} C_M C_{M,\lambda}^{-1/2}\|^{1/2} \leq 1$, $\|\widehat{C}_{M,\lambda}^{-1/2} \bar{S}_M^*\| = \|\widehat{C}_{M,\lambda}^{-1/2} \widehat{C}_M \widehat{C}_{M,\lambda}^{-1/2}\|^{1/2} \leq 1$, and $\|C_{M,\lambda}^{-1/2}S_M^*\| = \|C_{M,\lambda}^{-1/2}C_MC_{M,\lambda}^{-1/2}\|^{1/2} \leq 1$.
	Thus, one can obtain that
	\begin{align*}
		\|\widetilde{f}_{D_j,\lambda}^M - {f}_{\lambda}^M\| \leq \left[\|C_{M,\lambda}^{1/2} \widehat{C}_{M,\lambda}^{-1/2}\| + \|C_{M,\lambda}^{1/2} \widehat{C}_{M,\lambda}^{-1/2}\|^2 \right] \|f_\lambda^M - f_\rho\|.
	\end{align*}

	When $n \geq 16 (\mathcal{N}_\infty(\lambda) + 1) \log(2/\delta) $, by Lemma \ref{lem.difference_between_C_C_M}, it holds with the probability $1 - \delta$
	\begin{align*}
		\|\widetilde{f}_{D_j,\lambda}^M - {f}_{\lambda}^M\| \leq (\sqrt{2} + 2)\|f_\lambda^M - f_\rho\|.
	\end{align*}

	Using Lemma \ref{lem.rf-error} and Lemma \ref{lem.approximation-error}, we have  with probability at least $1 - 3 \delta$
	\begin{align*}
		\|\widetilde{f}_{D_j,\lambda}^M - {f}_{\lambda}^M\|
		&\leq (\sqrt{2} + 2)\left(\|f_\lambda^M - f_\lambda\| + \|f_\lambda - f_\rho\|\right) \\
		&\leq (\sqrt{2} + 2) (\sqrt{2} + 1) R \lambda^r
		\leq 9 \lambda^r.
	\end{align*}
\end{proof}

\subsubsection{Estimates for Random Features Error $\|f_\lambda^M - f_\lambda\|$}
The next lemma bounds the distance between the Tikhonov solution with RF and the Tikhonov solution without RF,
reflecting the approximation ability of random features.
\begin{lemma}
	\label{lem.rf-error}
	Under Assumptions \ref{asm.rf} and \ref{asm.compatibility}, for $\delta \in (0, 1/2], \la > 0$, 
	when 
	\begin{align*}
		M &\geq 16 (\mathcal{N}_\infty(\lambda) + 1) \log(2/\delta) \quad \forall r \in (0, 1/2), \quad \text{and}\\
		M &\geq 16\kappa^2(\mathcal{N}_\infty(\lambda) + 1)\log(2/\delta) \left[1 \vee \frac{\mathcal{N}(\la)}{(\mathcal{N}_\infty(\la) + 1)\la}\right]^{2r-1}
		\quad \forall r \in [1/2, 1],
	\end{align*}
	the following holds with a probability at least $1 - 2\delta$
	\begin{align*}
		\|f_\lambda^M - f_\lambda\| \leq \sqrt{2} R \lambda^r.
	\end{align*}
	
\end{lemma}
\begin{proof}
	According to the operator representations of $f_\lambda^M$ \eqref{f.expected-rf} and $f_\lambda$ \eqref{f.expected-erm}
	\begin{align*}
		f_\lambda - f_\lambda^M
		= (LL_\lambda^{-1} - L_M L_{M,\lambda}^{-1})f_\rho.
	\end{align*}
	Using the identity $A(A + \lambda I)^{-1} = I - \lambda(A + \lambda I)^{-1}$ and $A^{-1} - B^{-1} = A^{-1}(B - A)B^{-1}$, we have
	\begin{align*}
		f_\lambda - f_\lambda^M = \lambda(L_{M,\lambda}^{-1} - L_{\lambda}^{-1})f_\rho = \lambda L_{M,\lambda}^{-1} (L - L_M) L_{\lambda}^{-1} f_\rho.
	\end{align*}

	Applying Assumption \ref{asm.capacity}, there exists $g \in \Ltwo$ and $f_\rho = L^r g$, so we have
	\begin{align*}
		f_\lambda - f_\lambda^M = \sqrt{\lambda} (\sqrt{\lambda} L_{M, \lambda}^{-1/2}) (L_{M, \lambda}^{-1/2} L_{\lambda}^{1/2})[L_{\lambda}^{-1/2}(L-L_M)L_{\lambda}^{-(1-r)}] (L_\lambda^{-r} L^r) g.
	\end{align*}

	Note that $\|\sqrt{\lambda} L_{M, \lambda}^{-1/2}\| \leq 1$, $\|L_\lambda^{-r} L^r\| \leq 1$, $\|g\| \leq R$ and $\|L_{M, \lambda}^{-1/2}L_{\lambda}^{1/2}\| \leq \sqrt{2}$ if $M \geq 16 (\mathcal{N}_\infty(\lambda) + 1)  \log(2/\delta)$ due to Lemma \ref{lem.difference_between_L_L_M}, we thus have with probability at least $1 - \delta$
	\begin{align}		
		\label{eq.rf-errors.bound1}
		\|f_\lambda^M - f_\lambda\| \leq R \sqrt{2\lambda} \|L_{\lambda}^{-1/2}(L-L_M)L_{\lambda}^{-(1-r)}\|.
	\end{align}

	Then, we estimate the bound in two cases $r \in (0, 1/2)$ and $r \in [1/2, 1]$.
	\begin{itemize}
		\item When $r \in (0, 1/2)$, there exists
		      \begin{align*}
			      \|f_\lambda^M - f_\lambda\|
			       & \leq R\sqrt{2\lambda} \|L_{\lambda}^{-1/2}(L-L_M)L_{\lambda}^{-1/2}\|\|L_{\lambda}^{-(1/2-r)}\|                     \\
			       & \leq \sqrt{2} R \lambda^r \|L_{\lambda}^{-1/2}(L-L_M)L_{\lambda}^{-1/2}\|\|\lambda^{1/2-r} L_{\lambda}^{-(1/2-r)}\| \\
			       & \leq \sqrt{2} R \lambda^r \|L_{\lambda}^{-1/2}(L-L_M)L_{\lambda}^{-1/2}\|.
		      \end{align*}
			  The last step is due to $\|\lambda^{1/2-r} L_{\lambda}^{-(1/2-r)}\| \leq 1$ for any $ 0 < r < 1/2$.

			  Note that $\|L_{\lambda}^{-1/2}(L-L_M)L_{\lambda}^{-1/2}\| \leq 1/2$ using Lemma \ref{lem.difference_between_L_L_M}, thus for $r \in (0, 1/2)$, it holds with probability at least $1 - \delta$
		      \begin{align}
				\label{eq.rf-errors.difficult_problems}
			      \|f_\lambda^M - f_\lambda\| \leq R \lambda^r.
		      \end{align}
		\item When $r \in [1/2, 1]$, there exists
		      \begin{align*}
			      \|L_{\lambda}^{-1/2}(L-L_M)L_{\lambda}^{-(1-r)}\|
			      = \|L_{\lambda}^{-1/2}(L-L_M)L_{\lambda}^{-\varsigma/2}\|,
		      \end{align*}
		      with $\varsigma = 2 - 2r$ and $0 \leq \varsigma \leq 1$.

		      Using Proposition \ref{prop.varsigma} with $X=L_{\lambda}^{-1/2}(L-L_M)$ and $A=L_{\lambda}^{-1/2}$, one can obtain that
		      \begin{align*}
			      \|L_{\lambda}^{-1/2}(L-L_M)L_{\lambda}^{-(1-r)}\| \leq \|L_{\lambda}^{-1/2}(L-L_M)\|^{1-\varsigma}\|L_{\lambda}^{-1/2}(L-L_M)L_{\lambda}^{-1/2}\|^\varsigma.
		      \end{align*}
		      Thus, applying the above inequality to \eqref{eq.rf-errors.bound1}, we have
		      \begin{align}
			      \label{eq.rf-errors.bound2}
			      \|f_\lambda^M - f_\lambda\| \leq R \sqrt{2\lambda} \underbrace{ \|L_{\lambda}^{-1/2}(L-L_M)\|^{2r-1}\|L_{\lambda}^{-1/2}(L-L_M)L_{\lambda}^{-1/2}\|^{2-2r}}_{\text{mixed term}}
			  \end{align}
			  
			To obtain $
				\|f_\lambda^M - f_\lambda\| \leq R \lambda^r
			$
			we need the mixed term be bounded by
			\begin{align}
				\label{eq.rf-errors.mixed_term}
				 \|L_{\lambda}^{-1/2}(L-L_M)\|^{2r-1}\|L_{\lambda}^{-1/2}(L-L_M)L_{\lambda}^{-1/2}\|^{2-2r} \lesssim \lambda^{r - 1/2}.
			\end{align}

			From Lemma \ref{lem.L_difference}, with the condition $M \geq 16 \kappa^2(\mathcal{N}_\infty(\lambda) + 1)  \log(2/\delta)$, it holds
			\begin{align}
				\label{eq.rf-errors.proof.eq2}
				\|L_{\lambda}^{-1/2}(L-L_M)L_{\lambda}^{-1/2}\| 
				&\leq \frac{a(b-1)}{M} + \sqrt{\frac{ab}{M}} \leq \sqrt{\frac{2ab}{M}},
			\end{align}
			where $a=2\kappa^2\log(2/\delta)$ and $b=\mathcal{N}_\infty(\lambda) + 1$.
			Similarly, Lemma \ref{lem.L_difference_half} can be stated as
			\begin{align}
				\label{eq.rf-errors.proof.eq3}
				\|L_{\lambda}^{-1/2}(L-L_M)\| \leq \frac{a\sqrt{b - 1}}{\kappa M} + \sqrt{\frac{ac}{M}},
			\end{align}
			where $a=2\kappa^2\log(2/\delta)$, $b=\mathcal{N}_\infty(\lambda) + 1$ and $c = \mathcal{N}(\lambda)$.

			Note that, according to Minkowski's inequality, we have
			\begin{align}
				\label{eq.rf-errors.proof.jensen}
				\left(\frac{a\sqrt{b - 1}}{\kappa M} + \sqrt{\frac{ac}{M}}\right)^{2r-1}
				\leq \left(\frac{a\sqrt{b - 1}}{\kappa M}\right)^{2r-1} + \left(\sqrt{\frac{ac}{M}}\right)^{2r-1}.
			\end{align}

			Therefore, substituting \eqref{eq.rf-errors.proof.eq2} \eqref{eq.rf-errors.proof.eq3}, \eqref{eq.rf-errors.proof.jensen} to \eqref{eq.rf-errors.mixed_term}, there holds 
			\begin{align*}
				& \|L_{\lambda}^{-1/2}(L-L_M)\|^{2r-1}\|L_{\lambda}^{-1/2}(L-L_M)L_{\lambda}^{-1/2}\|^{2-2r}\\
				\leq & \left(\frac{a\sqrt{b - 1}}{\kappa M} + \sqrt{\frac{ac}{M}}\right)^{2r-1} \left(\sqrt{\frac{2ab}{M}}\right)^{2-2r},\\
				\leq & \left[\left(\frac{a\sqrt{b - 1}}{\kappa M}\right)^{2r-1} + \left(\sqrt{\frac{ac}{M}}\right)^{2r-1}\right] \left(\frac{2ab}{M}\right)^{1-r},\\
				= & 2^{1-r}\left[\left(\frac{a\sqrt{b - 1}}{\kappa M}\right)^{2r-1} + \left(\sqrt{\frac{ac}{M}}\right)^{2r-1}\right] \left(\frac{ab}{M}\right)^{1-r},\\
				\leq & \sqrt{2}\left[\left(\frac{a\sqrt{b - 1}}{\kappa M}\right)^{2r-1} + \left(\sqrt{\frac{ac}{M}}\right)^{2r-1}\right] \left(\frac{ab}{M}\right)^{1-r},\\
				\leq & \sqrt{2}\left[\frac{a^r\sqrt{b}}{\kappa^{2r-1}M^r} + \frac{\sqrt{a}b^{1-r}c^{r-1/2}}{\sqrt{M}}\right].
			\end{align*}
			To make the mixed term bounded by $\lambda^{r-1/2}$, we consider the following condition
			\begin{align*}
				M \geq 8 a b^{2-2r} \left(\frac{c}{\la}\right)^{2r-1}
			\end{align*}
			and obtain the bound of mixed term
			\begin{align*}
				& \|L_{\lambda}^{-1/2}(L-L_M)\|^{2r-1}\|L_{\lambda}^{-1/2}(L-L_M)L_{\lambda}^{-1/2}\|^{2-2r}\\
				\leq & ~ \sqrt{2}\left[\frac{a^r\sqrt{b}}{\kappa^{2r-1}M^r} + \frac{\sqrt{a} b^{1-r}c^{r-\frac{1}{2}}}{\sqrt{M}}\right] \\
				\leq & ~  \sqrt{2}\left[\frac{b^{\frac{1}{2}(1-2r)^2} c^{r-2r^2} \la^{2r^2-r}}{2^{3r} \kappa^{2r-1}}	 + \frac{\la^{r-1/2}}{2\sqrt{2}}\right] \\
				\leq & ~  \sqrt{2}\left[\frac{{(2\kappa^2\la^{-1})}^{\frac{1}{2}(1-2r)^2} 2^{2r^2-r} \la^{2r^2-r}}{2^{3r} \kappa^{2r-1}}	 + \frac{\la^{r-1/2}}{2\sqrt{2}}\right] \\
				\leq & ~  \sqrt{2}\left[\frac{\kappa^{4r^2-6r+2} \la^{r-1/2}}{2^{6r - 4r^2 - 1/2}}	 + \frac{\la^{r-1/2}}{2\sqrt{2}}\right] \\
				\leq & ~ \la^{r-1/2}.
			\end{align*}
			
	The third step is due to $b=\mathcal{N}_\infty(\lambda)  + 1 \leq 2 \kappa^2\lambda^{-1}$ and
	\begin{align*}
		c = \text{Tr}(L L_{\la}^{-1}) = \sum_{i \geq 1} \frac{\la_i(L)}{\la_i(L) + \la} \geq \frac{\la_1(L)}{\la_1(L) + \la} = \frac{\|L\|}{\|L\| + \la} \geq \frac{1}{2},
	\end{align*}
	due to $0 \leq \la \leq \|L\|$ to guarantee bounded effective dimension $\mathcal{N}_M(\la)$ in Proposition 10 \cite{rudi2017generalization}.
	The last step is due to $\kappa^{4r^2-6r+2} \leq 1$ since $4r^2-6r+2 \leq 0$ and $2^{6r - 4r^2 - 1/2} \geq 2\sqrt{2}$ since $6r - 4r^2 - 1/2 \geq 1.5$.

	Thus, with the condition $M \geq 16\kappa^2(\mathcal{N}_\infty(\lambda) + 1)\log(2/\delta)$ and $ M	\geq 16 \kappa^2 (\mathcal{N}_\infty(\lambda) + 1)^{2 - 2r} (\mathcal{N}(\lambda)/\la)^{2r-1}\ \log(2/\delta),$ we have with probability at least $1 - 2 \delta$  
	\begin{equation}
		\begin{aligned}
			\label{eq.rf-errors.easy_problems}
			&\|f_\lambda^M - f_\lambda\|\\
			\leq & R \sqrt{2\lambda} \|L_{\lambda}^{-1/2}(L-L_M)\|^{2r-1}\|L_{\lambda}^{-1/2}(L-L_M)L_{\lambda}^{-1/2}\|^{2-2r}\\
			\leq & \sqrt{2} R  \lambda^r.
		\end{aligned}
	\end{equation}
	\end{itemize}

	Combing the results in \eqref{eq.rf-errors.difficult_problems} and \eqref{eq.rf-errors.easy_problems}, we prove the lemma.
\end{proof}

\subsubsection{Estimates for Approximation Error $\| {f}_{\lambda} - f_\rho\|$}
The last term we need to estimate is approximation error $\| {f}_{\lambda} - f_\rho\|,$
whose proof is standard \cite{smale2007learning,caponnetto2007optimal,rudi2017generalization}.
\begin{lemma}
	\label{lem.approximation-error}
	Under Assumption \ref{asm.rf} and \ref{asm.regularity}, the following holds for any $\la > 0$ and $r > 0$,
	$$\| {f}_{\lambda} - f_\rho\| \leq R \la^{r}.$$
\end{lemma}
\begin{proof}
	Under Assumption \ref{asm.regularity}, there exists $g \in \Ltwo$ such that $f_\rho = L^r g$ with $\|g\| \leq R$.
	The identity $A(A + \la I)^{-1} = I - \la (A + \la I)^{-1}$ is valid for $\la > 0$ and $A$ the bounded self-adjoint positive operator and by the definition of $f_\lambda$ (\ref{f.expected-erm}), we have
	\begin{align*}
		\| {f}_{\lambda} - f_\rho\|
		=    & \nor{\L \Ll^{-1} \frho - \frho} = \nor{(\L \Ll^{-1} - I) \frho} = \nor{\lambda L_\lambda^{-1} \frho} \\
		=    & \|\la^r  (\la^{1-r}\Ll^{-(1-r)})  (\Ll^{-r}\L^r)  g\|                                                \\
		\leq & \|\la^r\|\|\la^{1-r}\Ll^{-(1-r)}\|\|\Ll^{-r}\L^r\|\| g\|
	\end{align*}
	Note that $\nor{\la^{1-r}\Ll^{-(1-r)}} \leq 1$ and $\nor{\Ll^{-r}\L^r} \leq 1$, while $R := \nor{g}_\Ltwo$ according to	Assumption \ref{asm.regularity}.
	The proof is completed.
\end{proof}

\subsection{Proofs of Main Results}
\begin{theorem}[\textit{General excess risk} bound]
	\label{thm.excess-risk-bound-proof}
	Let $\delta \in (0, 1/5]$ and $\widehat{f}_{D,\lambda}^M$ be defined by \eqref{estimator.emp-global-rf}.
	Under Assumptions \ref{asm.rf}, \ref{asm.moment}, \ref{asm.capacity}, \ref{asm.regularity} and \ref{asm.compatibility},
	when $\lambda = N^{-\frac{1}{2r + \gamma}}$, 
	the number of local processors satisfies
	\begin{align*}		
		m \leq \frac{1}{32 F \log(2/\delta)} N^{\frac{2r + \gamma - 1}{2r + \gamma}}
	\end{align*}
	and the dimension of random features satisfies
	\begin{align*}
		&\forall ~ r \in (0, 1/2) \qquad  M \geq 32 \kappa^2 \log(2/\delta) N^{\frac{\alpha}{2r + \gamma}}, \quad \text{and}\\
		&\forall ~ r \in [1/2, 1] \qquad M \geq 32 \kappa^2 \tilde{F} \tilde{Q}^2 \log (2/\delta) N^{\frac{(2r - 1)(1+\gamma-\alpha) + \alpha}{2r + \gamma}},
	\end{align*}
	then the following holds with a probability at least $1 - 5 \delta$,
	\begin{align*}
		\mathcal{E}(\widehat{f}_{D,\lambda}^M) - \mathcal{E}(\frho)
		\leq & c_2 N^{-\frac{2r}{2r + \gamma}},
	\end{align*}
	where $\tilde{F}=\max(F, \kappa^2)$, $\tilde{Q}^2 = \max(Q^2, 1)$ and $c_2$ is a constant independent on $m, n, N^*$ that
	$$c_2=128 \left(\frac{(B+\sigma)^2}{16F\log(2/\delta)} + 1.5\sigma^2 Q^2 \right) \log^2\frac{2}{\delta} + 495 R^2.$$
\end{theorem}
\begin{proof}
	From Lemma \ref{lem.full-decomposition}, there holds the upper bound for excess risk
	\begin{equation*}
		\begin{aligned}
			\mathbb{E}~\mathcal{E}(\widehat{f}_{D,\lambda}^M) - \mathcal{E}(f_\rho)
			\leq~ &\frac{1}{m}\mathbb{E} \|\widehat{f}_{D_j,\lambda}^M - \widetilde{f}_{D_j,\lambda}^M\|_\rho^2
			+~ 3\left(\frac{1}{m} + 1\right)\mathbb{E} \|\widetilde{f}_{D_j,\lambda}^M - {f}_{\lambda}^M\|_\rho^2 \\
			+~ &3 \, \mathbb{E}\| {f}_{\lambda}^M
			- {f}_{\lambda}\|_\rho^2 
			+~ 3 \, \| {f}_{\lambda}
			-~ f_\rho\|_\rho^2.
		\end{aligned}
	\end{equation*}

	In the following, we use Lemma \ref{lem.variance-error}, Lemma \ref{lem.empirical-error} and Lemma \ref{lem.rf-error} to bound error terms. 
	Therefore, we need to take into account the conditions in those lemmas. 
	There are constraints on the number of local examples $n$ and the dimension of random features $M$:
	\begin{align*}
		n &\geq 16 (\mathcal{N}_\infty(\lambda) + 1) \log(2/\delta), \quad \text{and}\\
		M &\geq 16 (\mathcal{N}_\infty(\lambda) + 1) \log(2/\delta) \quad \forall r \in (0, 1/2),\\
		M &\geq 16\kappa^2(\mathcal{N}_\infty(\lambda) + 1)\log(2/\delta) \left[1 \vee \frac{\mathcal{N}(\la)}{(\mathcal{N}_\infty(\la) + 1)\la}\right]^{2r-1}
		\quad \forall r \in [1/2, 1].
	\end{align*}
	Here, we merge the constraints on $M$ because it is difficult to acknowledge which range the regularity $r$ belongs to. 
	Meanwhile, $n$ is dependent on the number of partitions $m$, where $n=N/m$.
	Due the constraint on the number of samples $n \geq 32 \mathcal{N}_\infty(\lambda) \log(2/\delta) \geq 16 (\mathcal{N}_\infty(\lambda) + 1) \log(2/\delta)$ and  $\lambda = N^{-\frac{1}{2r + \gamma}}$, we use Assumption \ref{asm.compatibility} to obtain the restrict on the number of partitions
	\begin{align}
		\label{eq.proof.constraint_on_m}
		m = N/n \leq \frac{N \lambda^\alpha}{32 F \log(2/\delta)} = \frac{N^{\frac{2r + \gamma - \alpha}{2r + \gamma}}}{32 F\log(2/\delta)}  = \mathcal{O} \left(N^{\frac{2r + \gamma - \alpha}{2r + \gamma}}\right).
	\end{align}

	\begin{itemize}
		\item When $r \in (0, 1/2)$, using Assumption \ref{asm.compatibility} that $\mathcal{N}_\infty(\lambda) \leq F \lambda^{-\alpha}$, to ensure $M \geq 16 (\mathcal{N}_\infty(\lambda) + 1) \log(2/\delta)$, we sholud have
		\begin{align*}
			M \geq 32 F \lambda^{-\alpha} \log(2/\delta) \geq 16 (\mathcal{N}_\infty(\lambda) + 1) \log(2/\delta).
		\end{align*}
		Thus, it holds
		\begin{align}
			\label{eq.proof.constraint_on_M.small_r}
			M \geq 32 F \log(2/\delta) N^{\frac{\alpha}{2r + \gamma}} = \Omega\left(N^{\frac{\alpha}{2r + \gamma}}\right).
		\end{align}
		\item When $r \in [1/2, 1]$, using Assumption \ref{asm.capacity} and Assumption \ref{asm.compatibility}, we should have
		\begin{align*}
			&16 \kappa^2 (\mathcal{N}_\infty(\lambda) + 1)^{2 - 2r} \mathcal{N}(\lambda)^{2r-1}\lambda^{1 - 2r} \log(2/\delta)\\
			\leq & 16 \kappa^2 (F\lambda^{-\alpha} + 1)^{2 - 2r} (Q^2 \lambda^{-\gamma})^{2r-1} \lambda^{1-2r} \log (2/\delta) \\
			\leq & 16 \kappa^2 2^{2-2r} \tilde{F}^{2-2r} \lambda^{(2r - 2)\alpha} Q^{4r - 2} \lambda^{(1-2r)\gamma} \lambda^{1-2r} \log (2/\delta)\\
			\leq & 32 \kappa^2 \tilde{F} \tilde{Q}^2 \log (2/\delta) \lambda^{(1-2r)(1 + \gamma - \alpha) - \alpha},
		\end{align*}
		where $\tilde{F} = \max(F, \kappa^2) \geq 1$ and $\tilde{Q}^2 = \max(Q^2, 1) \geq 1$.
	
		To ensure $M \geq 16\kappa^2\log(2/\delta) \big[(\mathcal{N}_\infty(\lambda) + 1)\vee \lambda^{1-2r}\mathcal{N}(\lambda)^{2r-1}(\mathcal{N}_\infty(\lambda) + 1)^{2-2r}\big]$, using the above inequality it holds for $r \in [1/2, 1]$
		\begin{align}
			\label{eq.proof.constraint_on_M.big_r}
			M &\geq 32 \kappa^2 \tilde{F} \tilde{Q}^2 \log (2/\delta) \left[N^{\frac{(2r - 1)(1+\gamma-\alpha) + \alpha}{2r + \gamma}} \vee N^{\frac{\alpha}{2r + \gamma}}\right]\\
			&= \Omega\left(N^{\frac{(2r - 1)(1+\gamma-\alpha) + \alpha}{2r + \gamma}} + N^{\frac{\alpha}{2r + \gamma}}\right),\\
			&= \Omega\left(N^{\frac{(2r - 1)(1+\gamma-\alpha) + \alpha}{2r + \gamma}}\right).
		\end{align}
		due to the fact $32 \kappa^2 \log(2/\delta) \leq 32 \kappa^2 \tilde{F} \tilde{Q}^2 \log (2/\delta)$.
	\end{itemize}

	By Lemma \ref{lem.variance-error}, $\mathcal{N}(\lambda) \leq Q^2 \lambda^{-\gamma}$ and $\lambda = N^{-\frac{1}{2r + \gamma}}$, it holds for the global sample variance
	\begin{align*}
		&\frac{1}{m}\mathbb{E} \|\widehat{f}_{D_j,\lambda}^M - \widetilde{f}_{D_j,\lambda}^M\|_\rho^2\\
		\leq &\frac{64}{m} \left(\frac{\kappa(B+\sigma)}{n\sqrt{\lambda}} + \sqrt{\frac{\sigma^2\mathcal{N}_M(\lambda)}{n}}\right)^2\log^2\frac{2}{\delta}\\
		= &64 \left(\frac{\sqrt{m}\kappa(B+\sigma)}{N\sqrt{\lambda}} + \sqrt{\frac{\sigma^2\mathcal{N}_M(\lambda)}{N}}\right)^2\log^2\frac{2}{\delta}\\
		\leq &128 \left(\frac{m\kappa^2(B+\sigma)^2}{N^2\lambda} + \frac{\sigma^2\mathcal{N}_M(\lambda)}{N}\right)\log^2\frac{2}{\delta} \\
		\leq &128 \left(\kappa^2(B+\sigma)^2 m N^\frac{1-4r-2\gamma}{2r+\gamma} + \frac{2.55\sigma^2 Q^2 \lambda^{-\gamma}}{N}\right)\log^2\frac{2}{\delta}\\
		\leq &128 \left(\kappa^2(B+\sigma)^2 m N^\frac{1-4r-2\gamma}{2r+\gamma} + 2.55\sigma^2 Q^2 N^\frac{-2r}{2r+\gamma}\right)\log^2\frac{2}{\delta}.
	\end{align*}
	The last step is due the inequality $(a + b)^2 \leq 2 a^2 + 2 b^2.$	
	From Assumption \ref{asm.capacity}, we have $\mathcal{N}(\lambda) \leq Q^2 \lambda^{-\gamma}$. 
	Note that, we can obtain $\mathcal{N}_M(\lambda) \leq 2.55 \mathcal{N}(\lambda) \leq 2.55 Q^2 \lambda^{-\gamma}$ by Proposition 10 of \cite{rudi2017generalization} and $\la \leq \|L\|$.
	Using $m \leq \frac{1}{32 F\log(2/\delta)}N^{\frac{2r + \gamma - \alpha}{2r + \gamma}}$ and the worst case $\alpha=1$, it holds
	\begin{equation}
		\begin{aligned}
			\label{eq.proof.excess_risk.variance_bounded}
			&\frac{1}{m}\mathbb{E} \|\widehat{f}_{D_j,\lambda}^M - \widetilde{f}_{D_j,\lambda}^M\|_\rho^2 \\
			\leq &128 \left(\kappa^2(B+\sigma)^2 m N^\frac{1-4r-2\gamma}{2r+\gamma} + 2.55\sigma^2 Q^2 N^\frac{-2r}{2r+\gamma}\right)\log^2\frac{2}{\delta}\\
			\leq &128 \left(\frac{(B+\sigma)^2}{32 F\log(2/\delta)} N^{\frac{-2r - \gamma}{2r + \gamma}} + 2.55\sigma^2 Q^2 N^{\frac{-2r}{2r + \gamma}}\right)\log^2\frac{2}{\delta}.\\
			\leq &c_1 N^{\frac{-2r}{2r + \gamma}},
		\end{aligned}	
	\end{equation}
	where $c_1 = 128 \left(\frac{(B+\sigma)^2}{32F\log(2/\delta)} + 2.55\sigma^2 Q^2 \right) \log^2\frac{2}{\delta}.$ 

	According to Lemma \ref{lem.empirical-error}, there holds for the empirical error
	\begin{align}
		\label{eq.proof.excess_risk.empirical-error}
		\, \left(\frac{1}{m} + 1\right) \mathbb{E} \|\widetilde{f}_{D_j,\lambda}^M - {f}_{\lambda}^M\|_\rho^2 
		\leq 2 \, \mathbb{E} \|\widetilde{f}_{D_j,\lambda}^M - {f}_{\lambda}^M\|_\rho^2 
		\leq 162 \, R^2 N^{-\frac{2r}{2r + \gamma}}.
	\end{align}

	Using Lemma \ref{lem.rf-error}, for random features error, it holds
	\begin{align}
		\label{eq.proof.excess_risk.rf-error}
		\mathbb{E}\|f_\lambda^M - f_\lambda\|_\rho^2 \leq 2 R^2 \lambda^{2r} = 2 R^2 N^{-\frac{2r}{2r + \gamma}}
	\end{align}

	Using Lemma \ref{lem.approximation-error}, for approximation error, it holds
	\begin{align}
		\label{eq.proof.excess_risk.approximation-error}
		\| {f}_{\lambda} - f_\rho\|_\rho^2 \leq R^2 \lambda^{2r} = R^2 N^{-\frac{2r}{2r + \gamma}}
	\end{align}

	Substituting the above inequalities \eqref{eq.proof.excess_risk.variance_bounded} \eqref{eq.proof.excess_risk.empirical-error} \eqref{eq.proof.excess_risk.rf-error} \eqref{eq.proof.excess_risk.approximation-error} to Lemma \ref{lem.full-decomposition}, we then get the final result
	\begin{align*}
		\mathcal{E}(\widehat{f}_{D^*,\lambda}^M) - \mathcal{E}(\frho)
		\leq c_2 N^{-\frac{2r}{2r + \gamma}},
	\end{align*}
	where $c_2 = c_1 + 495 R^2$.
	Note that, the proof use inequalities with high probability $1 - \delta$, including Lemmas \ref{lem.sample_variance.error2}, \ref{lem.sample_variance.error3}, \ref{lem.difference_between_C_C_M}, \ref{lem.difference_between_L_L_M}, \ref{lem.L_difference_half}, and thus the final result holds with the probability at least $1 - 5\delta$.
\end{proof}

\begin{proof}[Proof of Theorem \ref{thm.general-results}]
	The results in Theorem \ref{thm.general-results} is a trivial extension of Theorem 2 in \cite{rudi2017generalization} and Corollary 1 in \cite{guo2017learning}.
	Only considering the attainable case $r \in [1/2, 1]$, this theorem can be proved by combining the proofs in \cite{guo2017learning} and \cite{rudi2017generalization}.

	Following the error decomposition and proof process in the proof of Theorem \ref{thm.data-dependent}, one can easily prove Theorem \ref{thm.general-results}.
	However, the main difference is how to bound the term $\|\tCnl^{-1/2}C_{M, \la}^{1/2}\|$ as a constant.
	Using Proposition 1 and the second-order decomposition of operator difference in \cite{guo2017learning}, one can obtain the following identities
	\begin{align*}
		BA^{-1} = (B-A)B^{-1}(B-A)A^{-1} + (B-A)B^{-1} + I.
	\end{align*}
	Applying $A = \tCnl$, $B = C_{M, \la}$, the facts $\|\tCnl^{-1}\|\leq 1/\la$ and $\|C_{M, \la}^{-1/2}\| \leq 1/\sqrt{\la}$, it holds
	\begin{align*}
		\|\tCnl^{-1/2}C_{M, \la}^{1/2}\|^2 = \|\tCnl^{-1}C_{M, \la}\|
		\leq &\|C_{M, \la}^{-1/2}(C_M - \widehat{C}_M)\|^2\la^{-1} \\
		+ &\|C_{M, \la}^{-1/2}(C_M - \widehat{C}_M)\|\la^{-1/2} + 1.
	\end{align*}
	With confidence at least $1 - \delta$, there holds for $\delta \in (0, 1)$ and can be found in \cite{caponnetto2007optimal,lin2017distributed}
	\begin{align*}
		\|C_{M, \la}^{-1/2}(C_M - \widehat{C}_M)\|\la^{-1/2} \leq 2\kappa \left\{\frac{\kappa}{n \la} + \sqrt{\frac{\mathcal{N}(\la)}{n \la}}\right\} \log \frac{2}{\delta}
	\end{align*}
	To guarantee the term $\|\tCnl^{-1/2}C_{M, \la}^{1/2}\|$ be a constant, it requires 
	\begin{align*}
		n \gtrsim \frac{1}{\lambda} \vee \frac{\mathcal{N}(\la)}{\lambda}
	\end{align*}
	to make sure that
	\begin{align*}
		\|C_{M, \la}^{-1/2}(C_M - \widehat{C}_M)\|\la^{-1/2} \lesssim \mathcal{O}(1).
	\end{align*}
	Using Assumption \ref{asm.capacity} and $\lambda = N^{-\frac{1}{2r+\gamma}}$, one can obtain the condition $n \gtrsim N^\frac{\gamma + 1}{2r + \gamma}$ same as in \cite{caponnetto2007optimal,lin2017distributed,guo2017learning}.
	However, in Lemma \ref{lem.C_M_difference} and Lemma \ref{lem.difference_between_C_C_M}, we directly apply a relaxed condition $n \gtrsim N^\frac{\alpha}{2r + \gamma}$ by Bernstein's inequality to guarantee the term $\|\tCnl^{-1/2}C_{M, \la}^{1/2}\|$ be a constant.
	
	To prove Theorem \ref{thm.general-results}, we just need use the condition $n \gtrsim N^\frac{\gamma + 1}{2r + \gamma}$ to replace the condition $n \gtrsim N^\frac{\alpha}{2r + \gamma}$ in the proof of Lemma \ref{lem.variance-error} and Lemma \ref{lem.empirical-error}.
	Then, following the proof of Theorem \ref{thm.excess-risk-bound-proof} for $r \in [1/2, 1]$, we prove the result with $\alpha=1$ due to $\mathcal{N}_\infty(\lambda) \leq \kappa^2 \lambda^{-1}.$
\end{proof}

\begin{proof}[Proof of Theorem \ref{thm.refined-results}]
	Consider the worst case of Assumption \ref{asm.compatibility}, it is equivalent to making no assumption on $\mathcal{N}_\infty(\lambda)$, and there always exists $\mathcal{N}_\infty(\lambda) \leq \kappa^2 \lambda^{-1}.$
	Applying Theorem \ref{thm.excess-risk-bound-proof} with $\tilde{F}=\kappa^2$ and $\alpha=1$, we prove the result.
\end{proof}

\begin{proof}[Proof of Theorem \ref{thm.data-dependent}]
Theorem \ref{thm.excess-risk-bound-proof} is the detailed version of Theorem \ref{thm.data-dependent}.
\end{proof}

\begin{theorem}[Improved Bounds with Additional Unlabeled Samples]
	\label{thm.unlabeled-proof}
	Let $\delta \in (0, 1]$ and $\widehat{f}_{D_j^*,\lambda}^M$ be defined by (\ref{f.global-sDKRR-RF}).
	Under Assumptions \ref{asm.rf}, \ref{asm.moment}, \ref{asm.capacity}, \ref{asm.regularity} and \ref{asm.compatibility},
	when $\lambda = N^{-\frac{1}{2r + \gamma}}$, 
	the total number of samples satisfies
	\begin{align*}
		N^* \geq 32 F \log(2/\delta) N^\frac{2r+2\gamma+\alpha-1}{2r+\gamma} \vee N,
	\end{align*}
	the number of local processors satisfies
	\begin{align*}		
		1 \leq m \leq N^\frac{2r+2\gamma-1}{2r+\gamma}
	\end{align*}
	and the dimension of random features satisfies
	\begin{align*}
		&\forall ~ r \in (0, 1/2) \qquad  M \geq 32 \kappa^2 \log(2/\delta) N^{\frac{\alpha}{2r + \gamma}}, \quad \text{and}\\
		&\forall ~ r \in [1/2, 1] \qquad M \geq 32 \kappa^2 \tilde{F} \tilde{Q}^2 \log (2/\delta) N^{\frac{(2r - 1)(1+\gamma-\alpha) + \alpha}{2r + \gamma}},
	\end{align*}
	then the following holds with a probability at least $1 - \delta$,
	\begin{align*}
		\mathcal{E}(\widehat{f}_{D^*,\lambda}^M) - \mathcal{E}(\frho)
		\leq & c_4 N^{-\frac{2r}{2r + \gamma}},
	\end{align*}
	where $\tilde{F}=\max(F, \kappa^2)$, $\tilde{Q}^2 = \max(Q^2, 1)$ and $c_2$ is a constant independent on $m, n, N*$ that
	$$c_4 = 128\left(\kappa^2(B+\sigma)^2 +2.55\sigma^2Q^2 \right) \log^2 \frac{2}{\delta} + 495 R^2.$$
\end{theorem}
\begin{proof}
	From Lemma \ref{lem.full-decomposition}, there holds the upper bound for excess risk
	\begin{equation*}
		\begin{aligned}
			\mathbb{E}~\mathcal{E}(\widehat{f}_{D,\lambda}^M) - \mathcal{E}(f_\rho)
			\leq~ &\frac{1}{m}\mathbb{E} \|\widehat{f}_{D_j,\lambda}^M - \widetilde{f}_{D_j,\lambda}^M\|_\rho^2
			+~ 3\left(\frac{1}{m} + 1\right)\mathbb{E} \|\widetilde{f}_{D_j,\lambda}^M - {f}_{\lambda}^M\|_\rho^2 \\
			+~ &3 \, \mathbb{E}\| {f}_{\lambda}^M
			- {f}_{\lambda}\|_\rho^2 
			+~ 3 \, \| {f}_{\lambda}
			-~ f_\rho\|_\rho^2.
		\end{aligned}
	\end{equation*}

	Using the above equality and Lemma \ref{lem.empirical-error}, we find that empirical error is data-dependent but output-independent.
	Meanwhile, the sample variance (Lemma \ref{lem.variance-error}) is dependent on the number of labeled samples $n = N/m$, while other terms (including $\|\tCnl^{-1/2}C_{M, \la}^{1/2}\|$) can be related to total sample size $n^* = N^*/m$.
	
	Based the sample variance, we first estimate the number of required labeled samples $n$.
	Using Lemma \ref{lem.variance-error} and \eqref{eq.proof.excess_risk.variance_bounded}, we have
	\begin{align*}
		&\frac{1}{m}\mathbb{E} \|\widehat{f}_{D^*_j,\lambda}^M - \widetilde{f}_{D^*_j,\lambda}^M\|_\rho^2 \\
		\leq &128 \left(\kappa^2(B+\sigma)^2 m N^\frac{1-4r-2\gamma}{2r+\gamma} + 2.55\sigma^2 Q^2 N^\frac{-2r}{2r+\gamma}\right)\log^2\frac{2}{\delta}.
	\end{align*}
	To guarantee the optimal learning rate, we need $m N^\frac{1-4r-2\gamma}{2r+\gamma} \leq \mathcal{O}\big(N^\frac{-2r}{2r+\gamma}\big)$, and thus
	\begin{align}
		\label{eq.unlabeled.proof.m1}
		m \leq \mathcal{O}\left(N^\frac{2r+2\gamma-1}{2r+\gamma}\right).
	\end{align}

	We then consider the additional unlabeled samples to reduce empirical error, where the local samples is label-free and the constraint is related to total sample size from Lemma \ref{lem.difference_between_C_C_M}:
	\begin{align*}
		n^* &\geq 16 (\mathcal{N}_\infty(\lambda)+ 1) \log(2/\delta).
	\end{align*}
	Let $\lambda = N^{-\frac{1}{2r+\gamma}}$, then the restriction on the dimension of random features $M$ is same to Theorem \ref{thm.excess-risk-bound-proof}.
	But the restriction on the number of partitions $m$ is changed to 
	\begin{align}
		\label{eq.unlabeled.proof.m2}
		m = \frac{N^*}{n^*} \leq \frac{N^* \lambda^{\alpha}}{32 F \log(2/\delta)}
		= \frac{1}{32 F \log(2/\delta)} N^* N^{-\frac{\alpha}{2r + \gamma}}
		= \mathcal{O}\left(N^* N^{-\frac{\alpha}{2r + \gamma}}\right).
	\end{align}
	From the constraint \eqref{eq.unlabeled.proof.m1} due to sample variance, we know that the number of partitions $m$ can not be bigger than $\mathcal{O}(N^\frac{2r+2\gamma-1}{2r+\gamma})$ and plays the leading role.
	Thus, combining \eqref{eq.unlabeled.proof.m1} and \eqref{eq.unlabeled.proof.m2}, one can obtain
	\begin{align*}
		N^* \gtrsim N N^\frac{\gamma+\alpha-1}{2r+\gamma} \vee N.
	\end{align*}
	
	We consider the following two conditions for $\alpha$
	\begin{itemize}
		\item The case $\alpha < 1 -\gamma$. It holds $2r+2\gamma-1 < 2r+\gamma-\alpha$, thus the constraint of the number of partition is $m \lesssim N^\frac{2r+2\gamma-1}{2r+\gamma}$.
		\item 
		The case $\alpha \geq 1 -\gamma$. It holds $\gamma+\alpha-1 \geq 0$ and we  make use of additional unlabeled examples $N^* \gtrsim N N^\frac{\gamma+\alpha-1}{2r+\gamma}$ to guarantee $m \lesssim N^\frac{2r+\gamma-\alpha}{2r+\gamma} \leq N^\frac{2r+2\gamma-1}{2r+\gamma}$.
	\end{itemize}
	Therefore, using unlabeled examples, the number of partitions always achieves $m \lesssim N^\frac{2r+2\gamma-1}{2r+\gamma}$.

	Considering the following constraints on the number of partitions $m$ and the dimension of random features $M$:
	\begin{align*}
		m &\lesssim N^\frac{2r+2\gamma-1}{2r+\gamma},\\
		M &\geq 16 (\mathcal{N}_\infty(\lambda) + 1) \log(2/\delta) \quad \forall r \in (0, 1/2),\\
		M &\geq 16\kappa^2(\mathcal{N}_\infty(\lambda) + 1)\log(2/\delta) \left[1 \vee \frac{\mathcal{N}(\la)}{(\mathcal{N}_\infty(\la) + 1)\la}\right]^{2r-1}
		\quad \forall r \in [1/2, 1].
	\end{align*}
	
	We first estimate the output-dependent error term: sample variance.
	Using $\lambda = N^{-\frac{1}{2r+\gamma}}$, $m \leq N^\frac{2r+2\gamma-1}{2r+\gamma}$ and \eqref{eq.proof.excess_risk.variance_bounded}, the global sample variance is bounded by 
	\begin{equation}
		\begin{aligned}
			\label{eq.unlabel.proof.global_sample_variance}
			&\frac{1}{m}\mathbb{E} \|\widehat{f}_{D_j,\lambda}^M - \widetilde{f}_{D_j,\lambda}^M\|_\rho^2 \\
			\leq & 128\left(\kappa^2(B+\sigma)^2 +2.55\sigma^2Q^2 \right) N^{\frac{-2r}{2r+\gamma}} \log^2 \frac{2}{\delta} \\
			\leq & c_3 N^{\frac{-2r}{2r+\gamma}},
		\end{aligned}
	\end{equation}
	where $c_3 = 128\left(\kappa^2(B+\sigma)^2 +2.55\sigma^2Q^2 \right) \log^2 \frac{2}{\delta}$.

	We then bound the label-free terms in Lemma \ref{lem.full-decomposition} with $\lambda = N^{-\frac{1}{2r+\gamma}}$.
	Using Lemma \ref{lem.empirical-error}, Lemma \ref{lem.rf-error} and Lemma \ref{lem.approximation-error}, it holds
	\begin{equation}
		\begin{aligned}
			\label{eq.unlabel.proof.global_label_free}
			\left(\frac{1}{m}+1\right) \mathbb{E} \|\widetilde{f}_{D_j,\lambda}^M - {f}_{\lambda}^M\|_\rho^2
			+ \mathbb{E}\|f_\lambda^M - f_\lambda\|_\rho^2 + \mathbb{E}\| f_\lambda - f_\rho\|_\rho^2 \leq 165 R^2 N^{-\frac{2r}{2r+\gamma}}.
		\end{aligned}
	\end{equation}

	Combining the above inequalities \eqref{eq.unlabel.proof.global_sample_variance} and \eqref{eq.unlabel.proof.global_label_free} to Lemma \ref{lem.full-decomposition}, one can prove the desired result.
\end{proof}

\begin{proof}[Proof of Theorem \ref{thm.unlabeled-explicit}]
	Theorem \ref{thm.unlabeled-proof} is the detailed version of Theorem \ref{thm.unlabeled-explicit}.
	\end{proof}

	\begin{corollary}
		\label{corollary.rf_bound}
		Under the same assumptions of Theorem \ref{thm.data-dependent}, if $r \in (0, 1]$, $\gamma \in [0, 1]$ and $\lambda = N^{-\frac{1}{2r+\gamma}}$, then $m = 1$
		and the number of random features $M$ satisfying
		\begin{align*}
			M &\gtrsim N^{\frac{\alpha}{2r + \gamma}} \qquad\qquad\qquad \text{when} ~~ 0<r<1/2 \qquad \text{and} \\
			M &\gtrsim N^{\frac{(2r - 1)(1+\gamma-\alpha) + \alpha}{2r + \gamma}} \qquad \text{when} ~~ 1/2 \leq r \leq 1,
		\end{align*}
		are sufficient to guarantee, with a high probability, that
		\begin{align*}
			\mathbb{E} ~ \mathcal{E}(\widehat{f}_{D^*,\lambda}^M) - \mathcal{E}(f_\rho)
			= \mathcal{O}\Big(N^{-\frac{2r}{2r+\gamma}}\Big).
		\end{align*}
	\end{corollary}
	The above error bound is a special case of Theorem \ref{thm.data-dependent} with only using one partition $m=1$, namely KRR-RF.
	Compared to theoretical results in \cite{rudi2017generalization} which only take effect in the attainable case $r \in [1/2, 1]$, Corollary \ref{corollary.rf_bound} pertain to both the attainable and non-attainable cases $r \in (0, 1]$, covering all difficult problems.
	Meanwhile, the requirements on the number of random features are reasonable and lead to higher computational efficiency.

	\begin{corollary}
		\label{corollary.krr-dc}
		Under the same assumptions of Theorem \ref{thm.unlabeled-explicit}, if $r \in (0, 1], \gamma \in [0, 1], 2r + 2\gamma \geq 1$ and $\lambda = N^{-\frac{1}{2r+\gamma}},$
		then the total number of samples corresponding to
		\begin{align*}
			N^* \gtrsim NN^\frac{\gamma+\alpha-1}{2r+\gamma} \vee N,
		\end{align*}
		and the number of local processors satisfying
		\begin{align*}      
			1 \lesssim m \lesssim N^\frac{2r+2\gamma-1}{2r+\gamma}
		\end{align*}
		are sufficient to guarantee, with a high probability, that
		\begin{align*}
			\mathbb{E} ~ \mathcal{E}(\widehat{f}_{D^*,\lambda}^M) - \mathcal{E}(f_\rho)
			= \mathcal{O}\Big(N^{-\frac{2r}{2r+\gamma}}\Big).
		\end{align*}
	\end{corollary}
	
	The above Corollary is a special case of \texttt{DKRR-RF} with the induced kernel rather than random features, i.e., KRR-DC.
	The existing theoretical results on KRR-DC are still restricted with $m \lesssim N^\frac{2r+\gamma-1}{2r+\gamma}$, while we improve the condition to $m \lesssim N^\frac{2r+2\gamma-1}{2r+\gamma}$ for the first time, which admits higher computational complexities and covers more complicated problems in the non-attainable cases.
	Using the condition $m \lesssim N^\frac{2r+2\gamma-1}{2r+\gamma}$, it is worthy of devising more efficient distributed KRR methods together with Nystr\"om subsampling, random projections, stochastic optimization, and other techniques in the future.

\subsection{Probabilistic Inequalities}
	
	\begin{proposition}[Lemma 2 in \cite{smale2007learning}]
		\label{pro.bernstein_inequality}
		Let $\mathcal{L}$ be a separable Hilbert space and $\{\xi_1, \cdots, \xi_n\}$ be a sequence of i.i.d random variables in $\mathcal{L}$.
		Assume the bound be $\|\xi_i\| \leq \widetilde{M} \leq \infty$ and the variance be $\tilde{\sigma}^2=\mathbb{E}(\|\xi_i - \mathbb{E}(\xi_i)\|^2)$ for any $i \in [n]$.
		For any $\delta \in (0, 1)$, with confidence $1 - \delta$,
		\begin{align}
			\label{eq.bernstein_inequality}
			\left\|\frac{1}{n} \sum_{i=1}^n \xi_i - \mathbb{E}(\xi_i)\right\| \leq \frac{2\widetilde{M}\log(2/\delta)}{n} + \sqrt{\frac{2 \tilde{\sigma}^2 \log(2/\delta)}{n}}.
		\end{align}
	\end{proposition}

	The above Bernstein's inequality is the key to analyzing the relationship between the empirical random vector and its expected counterpart, which is used to prove Lemma \ref{lem.C_M_difference} and Lemma \ref{lem.sample_variance.error2}.
	The above Bernstein's inequality for random vectors was provided in \cite{smale2007learning,rudi2017generalization}
	and later was extended to the random operator case in Lemma 24 of \cite{lin2020optimal}.
	
	\begin{proposition}[Lemma E.2 of \cite{blanchard2010optimal}]
		\label{prop.blanchard}
		For any self-adjoint and positive semi-definite operators $A$ and $B$, if there exists $0 < \eta < 1$ such that the following inequality holds
		\begin{align*}
			\|(A+\lambda I)^{-1/2}(B - A)(A+\lambda I)^{-1/2}\| \leq 1 - \eta,
		\end{align*}
		then
		\begin{align*}
			\|(A + \lambda I)^{1/2}(B + \lambda I)^{-1/2}\| \leq \frac{1}{\sqrt{\eta}}.
		\end{align*}
	\end{proposition}
	
	The above inequality \cite{blanchard2010optimal} was used to establish the connection between $\|(A+\lambda I)^{-1/2}(B - A)(A+\lambda I)^{-1/2}\|$ and $\|(A + \lambda I)^{1/2}(B + \lambda I)^{-1/2}\|$.
	In this paper, those two terms $\|C_{M, \la}^{-1/2} (C_M - \widehat{C}_M) C_{M, \la}^{-1/2} \|$ and $\|C_{M, \la}^{1/2}\tCnl^{-1/2}\|$ often exist on the left parts of the estimates of error terms, where we make use of Proposition \ref{prop.blanchard} to guarantee both of two terms of lhs as constants.

	\begin{proposition}[Proposition 9 in \cite{rudi2017generalization}]
		\label{prop.varsigma}
		Let $\mathcal{H}, \mathcal{K}$ be two separable Hilbert spaces and $X, A$ be bounded linear operators, with $X: \mathcal{H} \to \mathcal{K}$ and $A: \mathcal{H} \to \mathcal{H}$ be positive semidefinite.
		The following holds
		\begin{align*}
			\|X A^{\varsigma}\| = \|X\|^{1-\varsigma}\|XA\|^{\varsigma}, \qquad \forall \varsigma \in [0, 1].
		\end{align*}
	\end{proposition}

\begin{lemma}
	\label{lem.C_M_difference}
	Given $\phi_M(\xx)=M^{-1/2}\big[\psi(\xx, \omega_1), \cdots, \psi(\xx, \omega_M)\big]^\top$, let $i.i.d$ random vectors $\bigl[\phi_M(\xx_1), \cdots, \phi_M(\xx_n)\bigr]$ with $n \geq 1$ be on a separable Hilbert space $\mathcal{H}_M$ such that $C_M = \mathbb{E}_{\rho_X}[\phi_M(\xx) \otimes \phi_M(\xx)]$ and $\widehat{C}_M = \frac{1}{n}\sum_{i=1}^n \phi_M(\xx_i)\otimes\phi_M(\xx_i)$ are trace class.
	Then for any $\delta \in (0, 1)$ with the probability at least $1-\delta$, the following  holds
	\begin{align*}
		&\left\|(C_M + \lambda I)^{-1/2}(C_M - \widehat{C}_M)(C_M+\lambda I)^{-1/2}\right\| \\
		\leq &\frac{2 \mathcal{N}_\infty(\lambda)\log(2/\delta)}{n}
		+ \sqrt{\frac{2(\mathcal{N}_\infty(\lambda) + 1)\log(2/\delta)}{n}}.
	\end{align*}
\end{lemma}
\begin{proof}
	Let $C_{M, \lambda}^{-1/2} = (C_M + \lambda I)^{-1/2}$ and
	$$\xi = C_{M, \lambda}^{-1/2} \phi_M(\xx) \otimes C_{M, \lambda}^{-1/2} \phi_M(\xx),$$
	thus we have
	\begin{align*}
		\mathbb{E}(\xi)               &
		=C_{M, \lambda}^{-1/2}\mathbb{E}[\phi_M(\xx) \otimes \phi_M(\xx)]C_{M, \lambda}^{-1/2}
		=C_{M, \lambda}^{-1/2}C_MC_{M, \lambda}^{-1/2},                                                                                                                                                      \\
		\frac{1}{n}\sum_{i=1}^n \xi_i & = \frac{1}{n}\sum_{i=1}^nC_{M, \lambda}^{-1/2} [\phi_M(\xx_i) \otimes \phi_M(\xx_i)] C_{M, \lambda}^{-1/2} =C_{M, \lambda}^{-1/2}\widehat{C}_MC_{M, \lambda}^{-1/2}.
	\end{align*}
	The left of the desired inequality becomes
	\begin{align*}
		\left\|C_{M, \lambda}^{-1/2}(C_M - \widehat{C}_M)C_{M, \lambda}^{-1/2}\right\| = \left\|\mathbb{E} (\xi) - \frac{1}{n}\sum_{i=1}^n \xi_i\right\|.
	\end{align*}
	Note that
	\begin{align*}
		\|C_{M, \lambda}^{-1/2}\phi_M(\xx)\|^2
		\leq \frac{1}{M} \sum_{i=1}^M \|L_{M, \lambda}^{-1/2}\psi_{\omega_i}(\xx)\|^2
		\leq \sup_{\omega \in \Omega}\|L_{M, \lambda}^{-1/2}\psi_\omega(\xx)\|^2 = \mathcal{N}_\infty(\lambda).
	\end{align*}
	To use Bernstein's inequality (Proposition \ref{pro.bernstein_inequality}), we need to bound $\|\xi\|$ and $\mathbb{E}\|\xi\|^2$ as follows
	\begin{align*}
		& \|\xi\|= \|\langle C_{M, \lambda}^{-1/2} \phi_M(\xx), C_{M, \lambda}^{-1/2} \phi_M(\xx) \rangle\| =\|C_{M, \lambda}^{-1/2}\phi_M(\xx)\|^2 \leq \mathcal{N}_\infty(\lambda).\\
		&\mathbb{E}\|\xi - \mathbb{E}(\xi)\|^2 \\
		& = \left\|\mathbb{E} \left[\big\langle C_{M, \lambda}^{-1/2} \phi_M(\xx), C_{M, \lambda}^{-1/2} \phi_M(\xx)\big\rangle C_{M, \lambda}^{-1/2} \phi_M(\xx) \otimes C_{M, \lambda}^{-1/2} \phi_M(\xx)\right]-  C_{M, \lambda}^{-2}C_M^2\right\|              \\
		& \leq \mathcal{N}_\infty(\lambda)\left\|\mathbb{E} \left[C_{M, \lambda}^{-1/2} \phi_M(\xx) \otimes C_{M, \lambda}^{-1/2} \phi_M(\xx)\right]\right\| + \left\|C_{M, \lambda}^{-2}C_M^2\right\| \\
		& \leq \mathcal{N}_\infty(\lambda) \|C_{M, \lambda}^{-1} C_M\| + 1
		\leq \mathcal{N}_\infty(\lambda) + 1.
	\end{align*}
	Substituting the above two identities to Bernstein's inequality \eqref{eq.bernstein_inequality}, we prove the result.
\end{proof}

\begin{lemma}
	\label{lem.difference_between_C_C_M}
	When the number of the local samples $n \geq 16 (\mathcal{N}_\infty(\lambda) + 1) \log(2/\delta) $, then for any $\delta \in (0, 1)$, there exists with the confidence $1 - \delta$
	\begin{align*}
		\|C_{M, \la}^{-1/2} (C_M - \widehat{C}_M) C_{M, \la}^{-1/2} \| \leq \frac{1}{2} \quad \text{and} \quad
		\|C_{M, \la}^{1/2}\tCnl^{-1/2}\| \leq \sqrt{2}.
	\end{align*}
\end{lemma}
\begin{proof}
	From the Proposition \ref{lem.C_M_difference}, we set $n \geq 16 (\mathcal{N}_\infty (\lambda) + 1) \log(2/\delta) $ and obtain that
	\begin{align*}
		&\|C_{M, \la}^{-1/2} (\widehat{C}_M - C_M) C_{M, \la}^{-1/2}\| \\
		\leq &\frac{2\mathcal{N}_\infty(\lambda)\log(2/\delta)}{n}
		+ \sqrt{\frac{2(\mathcal{N}_\infty(\lambda) + 1)\log(2/\delta)}{n}}
		\leq \frac{1}{2}.
	\end{align*}
	From Proposition \ref{prop.blanchard} and the above inequality, there exists
	\begin{align*}
		\|C_{M, \la}^{1/2}\tCnl^{-1/2}\| \leq \left(1 - \frac{1}{2}\right)^{-\frac{1}{2}} = \sqrt{2}.
	\end{align*}
\end{proof}

\begin{lemma}
	\label{lem.L_difference}
	Let $\psi_{\omega_1}, \cdots, \psi_{\omega_M}$ with $M \geq 1$, be $i.i.d$ random vectors on a separable Hilbert space $\mathcal{H}_M$ such that $L = \mathbb{E}_\omega [\psi_\omega \otimes \psi_\omega]$ and $L_M = \frac{1}{M}\sum_{i=1}^M [\psi_{\omega_i} \otimes \psi_{\omega_i}]$ are trace class.
	Then for any $\delta \in (0, 1)$ with the probability at least $1-\delta$, the following  holds
	\begin{align*}
		&\left\|(L + \lambda I)^{-1/2}(L - L_M)(L + \lambda I)^{-1/2}\right\| \\
		\leq &\frac{2 \mathcal{N}_\infty(\lambda)\log(2/\delta)}{M}
		+ \sqrt{\frac{2(\mathcal{N}_\infty(\lambda) + 1)\log(2/\delta)}{M}}.
	\end{align*}
\end{lemma}
\begin{proof}
	Let $L_{\lambda}^{-1/2} = (L + \lambda I)^{-1/2}$ and
	$$\xi = \big[ L_{\lambda}^{-1/2} \psi_\omega \otimes L_{\lambda}^{-1/2} \psi_\omega\big],$$
	thus we have
	\begin{align*}
		\mathbb{E}(\xi)               &
		=L_{\lambda}^{-1/2}\mathbb{E}_\omega[\psi_\omega \otimes \psi_\omega]L_{\lambda}^{-1/2}
		=L_{\lambda}^{-1/2} L L_{\lambda}^{-1/2},        \\
		\frac{1}{M}\sum_{i=1}^M \xi_i & = \frac{1}{M}\sum_{i=1}^M  L_{\lambda}^{-1/2} [\psi_{\omega_i} \otimes \psi_{\omega_i}] L_{\lambda}^{-1/2} =L_{\lambda}^{-1/2} L_M L_{\lambda}^{-1/2}.
	\end{align*}
	The left of the desired inequality becomes
	\begin{align*}
		\left\|L_{\lambda}^{-1/2}(L - L_M)L_{\lambda}^{-1/2}\right\| = \left\|\mathbb{E} (\xi) - \frac{1}{M}\sum_{i=1}^M \xi_i\right\|.
	\end{align*}
	Note that
	\begin{align*}
		\|L_{\lambda}^{-1/2}\psi_\omega\|^2
		\leq \frac{1}{M} \sum_{i=1}^M \|L_{\lambda}^{-1/2}\psi_{\omega_i}(\xx)\|^2
		\leq  \sup_{\omega \in \Omega} \|L_{\lambda}^{-1/2}\psi_\omega(\xx)\|^2 = \mathcal{N}_\infty(\lambda).
	\end{align*}
	To use Bernstein's inequality (Proposition \ref{pro.bernstein_inequality}), we need to bound $\|\xi\|$ and $\mathbb{E}\|\xi\|^2$.
	Note that
	\begin{align*}
		\|\xi\| & = \| \langle L_{\lambda}^{-1/2} \psi_\omega, L_{\lambda}^{-1/2} \psi_\omega \rangle\| =\|L_{\lambda}^{-1/2}\psi_\omega\|^2 \leq \mathcal{N}_\infty(\lambda).                                                                 \\
		\mathbb{E}\|\xi - \mathbb{E}(\xi)\|^2
		        & = \left \|\big\langle L_{\lambda}^{-1/2} \psi_\omega, L_{\lambda}^{-1/2} \psi_\omega\big\rangle \mathbb{E} \left[L_{\lambda}^{-1/2} \psi_\omega \otimes L_{\lambda}^{-1/2} \psi_\omega\right] - L_{\lambda}^{-2} L^2 \right\|                                                                                                                                   \\
		        & \leq \mathcal{N}_\infty(\lambda) \left\|\mathbb{E} \left[L_{\lambda}^{-1/2} \psi_\omega \otimes L_{\lambda}^{-1/2} \psi_\omega\right]\right\|  + \left \| L_{\lambda}^{-2} L^2 \right\| \\
		        & \leq \mathcal{N}_\infty(\lambda) \|L_{\lambda}^{-1} L\| + 1
		\leq \mathcal{N}_\infty(\lambda) + 1.
	\end{align*}
	Substituting the above two identities to Bernstein's inequality \eqref{eq.bernstein_inequality}, we prove the result.
\end{proof}

\begin{lemma}
	\label{lem.difference_between_L_L_M}
	When the dimension of random features $M \geq 16 (\mathcal{N}_\infty(\lambda) + 1)  \log(2/\delta)$, then for any $\delta \in (0, 1)$, there exists with the confidence $1 - \delta$
	\begin{align*}
		\|L_\lambda^{-1/2} (L - L_M) L_\lambda^{-1/2}\| \leq \frac{1}{2} \quad \text{and} \quad
		\|L_{M, \la}^{-1/2}L_\lambda^{1/2}\| \leq \sqrt{2}.
	\end{align*}
\end{lemma}
\begin{proof}
	From the Proposition \ref{lem.L_difference}, we set $M \geq 16 (\mathcal{N}_\infty(\lambda) + 1)  \log(2/\delta) $ and obtain that
	\begin{align*}
		\|L_{\lambda}^{-1/2} (L_M - L) L_\lambda^{-1/2}\| \leq \frac{2\mathcal{N}_\infty(\lambda)\log(2/\delta)}{M}
		+ \sqrt{\frac{2(\mathcal{N}_\infty(\lambda) + 1)\log(2/\delta)}{M}}
		\leq \frac{1}{2}.
	\end{align*}
	From Proposition \ref{prop.blanchard} and the above inequality, there exists
	\begin{align*}
		\|L_{M, \la}^{-1/2}L_\lambda^{1/2}\| \leq \left(1 - \frac{1}{2}\right)^{-\frac{1}{2}} = \sqrt{2}.
	\end{align*}
\end{proof}

\begin{lemma}
	\label{lem.L_difference_half}
	Let $\psi_{\omega_1}, \cdots, \psi_{\omega_M}$ with $M \geq 1$, be $i.i.d$ random vectors on a separable Hilbert space $\mathcal{H}_M$ such that $L = \mathbb{E}_\omega [\psi_\omega \otimes \psi_\omega]$ and $L_M = \frac{1}{M}\sum_{i=1}^M  [\psi_{\omega_i} \otimes \psi_{\omega_i}]$ are trace class.
	Then for any $\delta \in (0, 1)$ with the probability at least $1-\delta$, the following  holds
	\begin{align*}
		\left\|(L + \lambda I)^{-1/2}(L - L_M)\right\| \leq \frac{2 \sqrt{\kappa^2\mathcal{N}_\infty(\lambda)}\log(2/\delta)}{M}
		+ \sqrt{\frac{2\kappa^2\mathcal{N}(\lambda)\log(2/\delta)}{M}}.
	\end{align*}
\end{lemma}
\begin{proof}
	Let $L_{\lambda}^{-1/2} = (L + \lambda I)^{-1/2}$ and
	$$\xi = \big[ L_{\lambda}^{-1/2} \psi_\omega \otimes \psi_\omega\big],$$
	thus we have
	\begin{align*}
		\mathbb{E}(\xi)               &
		=L_{\lambda}^{-1/2}\mathbb{E}_\omega[\psi_\omega \otimes \psi_\omega]
		=L_{\lambda}^{-1/2} L,                                                                                                                             \\
		\frac{1}{M}\sum_{i=1}^M \xi_i & = \frac{1}{M}\sum_{i=1}^M  L_{\lambda}^{-1/2} [\psi_{\omega_i} \otimes \psi_{\omega_i}]
		= L_{\lambda}^{-1/2} L_M.
	\end{align*}
	The left of the desired inequality becomes
	\begin{align*}
		\left\|L_{\lambda}^{-1/2}(L - L_M)\right\| = \left\|\mathbb{E} (\xi) - \frac{1}{M}\sum_{i=1}^M \xi_i\right\|.
	\end{align*}
	Note that
	\begin{align*}
		 & \|L_{\lambda}^{-1/2}\psi_\omega\|^2
		\leq \frac{1}{M} \sum_{i=1}^M \|L_{\lambda}^{-1/2}\psi_{\omega_i}(\xx)\|^2
		\leq  \sup_{\omega \in \Omega} \|L_{\lambda}^{-1/2}\psi_\omega(\xx)\|^2 = \mathcal{N}_\infty(\lambda). \\
		 & \mathbb{E}_\omega \|L_{\lambda}^{-1/2}\psi_\omega\|^2
		=\int \langle\psi_\omega, L_\lambda^{-1}\psi_\omega \rangle d\pi(\omega)
		= \tr \left(L_\lambda^{-1} \int \psi_\omega \otimes \psi_\omega d\pi(\omega)\right)=\mathcal{N}(\lambda).
	\end{align*}
	To use Bernstein's inequality (Proposition \ref{pro.bernstein_inequality}), we need to bound $\|\xi\|$ and $\mathbb{E}\|\xi\|^2$.
	Note that
	\begin{align*}
		\|\xi\| & = \| \langle L_{\lambda}^{-1/2} \psi_\omega, \psi_\omega \rangle\|
		\leq \|L_{\lambda}^{-1/2}\psi_\omega\|\|\psi_\omega\| \leq \sqrt{\mathcal{N}_\infty(\lambda)}\kappa.                                                                                                                       \\
		\mathbb{E}\|\xi - \mathbb{E}(\xi)\|^2
				& = \tr(\mathbb{E}\xi \otimes \xi - (\mathbb{E}\xi)^2)  \leq \tr(\mathbb{E}\xi \otimes \xi)    \\
				& =  \|\psi_\omega)\|^2 \mathbb{E} \tr(L_\lambda^{-1/2} \psi_\omega \otimes \psi_\omega L_\lambda^{-1/2})\\
				& \leq \kappa^2 \tr(L_\lambda^{-1/2} \mathbb{E}(\psi_\omega \otimes \psi_\omega) L_\lambda^{-1/2})\\
				& \leq \kappa^2 \mathcal{N}(\lambda).
	\end{align*}
	The last step is due to $\mathcal{N}(\lambda) \leq \mathcal{N}_\infty(\lambda).$
	Substituting the above two identities and $\kappa \geq 1$ to Bernstein's inequality \eqref{eq.bernstein_inequality}, we prove the result.
\end{proof}